%% file: Bonalli.Rudi.2022.tex
\documentclass[onefignum,onetabnum]{siamart190516}

\input{shared}

\ifpdf
\hypersetup{
  pdftitle={Non-Parametric Learning of Stochastic Differential Equations with Fast Rates of Convergence},
  pdfauthor={R. Bonalli and A. Rudi}
}
\fi

\begin{document}

\maketitle

\begin{abstract}
  We propose a novel non-parametric learning paradigm for the identification of drift and diffusion coefficients of \firstRev{multi-dimensional} non-linear stochastic differential equations, which relies upon discrete-time observations of the state. The key idea essentially consists of fitting a RKHS-based approximation of the corresponding Fokker-Planck equation to such observations, yielding theoretical estimates of \firstRev{non-asymptotic} learning rates which, unlike previous works, become increasingly tighter when the regularity of the unknown drift and diffusion coefficients becomes higher. Our method being kernel-based, offline pre-processing may be profitably leveraged to enable efficient numerical implementation, \firstRev{offering excellent balance between precision and computational complexity.}
\end{abstract}



\section{Introduction}

\label{sec:Intro}
\input{Intro}

\section{Notation and Preliminary Results}
\label{sec:Notation}
\input{Notation}

\section{Stochastic Differential and Fokker-Planck Equations}
\label{sec:ShortFokkerPlanck}
\input{ShortFokkerPlanck}

\section{Useful Results from RKHS Theory}
\label{sec:RKHS}
\input{RKHS}

\section{Learning Stochastic Differential Equations}
\label{sec:LearningDynamics}
\input{LearningDynamics}

\section{Computational Considerations}
\label{sec:Computational}
\input{Computational}

\section{Conclusion and Perspectives}
\label{sec:Conclusion}
\input{Conclusion}

\section{Details on Stochastic Differential and Fokker-Planck Equations}
\label{sec:FokkerPlanck}
\input{FokkerPlanck}

\appendix\section{Proofs of Sections \ref{sec:Notation} and \ref{sec:FokkerPlanck}}
\label{sec:AppendixFokkerPlanck}
\input{AppendixFokkerPlanck}

\newpage

\bibliographystyle{siamplain}
\bibliography{references}
\end{document}

%% file: shared.tex
\usepackage{lipsum}
\usepackage{amsfonts}
\usepackage{graphicx}
\usepackage{epstopdf}
\usepackage{algorithmic}
\ifpdf
  \DeclareGraphicsExtensions{.eps,.pdf,.png,.jpg}
\else
  \DeclareGraphicsExtensions{.eps}
\fi
\usepackage{amssymb}
\usepackage{dsfont}
\usepackage{enumitem}

\newenvironment{Itemize}{\begin{itemize}[itemsep=1pt,topsep=10pt,left=2.25ex]}{\end{itemize}}
\usepackage{bm}
\usepackage{xspace}

\newcommand{\R}{\mathbb{R}}
\newcommand{\Rn}{\mathbb{R}^n}
\newcommand{\Rnn}{\mathbb{R}^{n \times n}}
\newcommand{\SymP}{\textnormal{Sym}_+(n)}
\newcommand{\SymPP}{\textnormal{Sym}_{++}(n)}

\newcommand{\expo}{m}
\newcommand{\Cbasis}{\mathcal{H}_{\expo}}
\newcommand{\Cweak}{\mathcal{H}^+_{\expo}}
\newcommand{\F}{\mathcal{F}}
\newcommand{\G}{\mathcal{G}}
\newcommand{\DistSpace}{S}
\newcommand{\Dist}{\mathfrak{d}}
\newcommand{\Ft}{\mathcal{F}_t}

\newcommand{\Borel}{\mathcal{B}(\Rn)}
\newcommand{\BorelDist}{\mathcal{B}(\DistSpace)}
\newcommand{\ProbaSpace}{\mathcal{P}}
\newcommand{\Proba}{\mathbb{P}}
\newcommand{\ProbaX}{\mu_X}
\newcommand{\Kolmo}{\mathcal{L}^{a,b}}

\newcommand{\unknownDrift}{b_*}
\newcommand{\unknownDiffusion}{a_*}

\newcommand{\unknownDensity}{p}
\newcommand{\sinc}{\operatorname{sinc}}
\newcommand{\unknownKolmo}{\mathcal{L}}

\newcommand{\approxDensity}{\widehat{p}}
\newcommand{\solDiffusion}{\widehat{a}}
\newcommand{\solDrift}{\widehat{b}}

\newcommand{\solDensity}{p_{\solDiffusion,\solDrift}}
\newcommand{\OperatorSDE}{\mathcal{O}}

\newcommand{\solAppDiffusion}{\widehat{a}_Q}
\newcommand{\solAppDrift}{\widehat{b}_Q}

\newcommand{\proj}{\mathcal{P}_Q}

\newcommand{\projKolmoUnk}{\mathcal{L}^{\proj(\unknownDiffusion),\proj(\unknownDrift)}}

\newcommand{\Bil}{A^{\alpha,a,b}}
\newcommand{\BilDot}{A^{0,\dot{a},\dot{b}}}

\newcommand{\OperatorSDECoeff}{\OperatorSDE^{a,b}}
\newcommand{\densityCoeff}{p_{a,b}}

\newcommand{\LearningProblem}{\textnormal{LP}\xspace}
\newcommand{\LearningProblemFinite}{\textnormal{LP$_Q$}\xspace}

\newcommand{\XCoeff}{X^{a,b}}
\newcommand{\unknownX}{X}

\newcommand{\Metric}{E}

\newcommand{\scal}[2]{{\langle{{#1}},{{#2}}\rangle}}

\newcommand{\firstRev}[1]
{\textcolor{black}{#1}}

\newcommand{\secRev}[1]
{\textcolor{black}{#1}}

\newcommand{\thirdRev}[1]
{\textcolor{black}{#1}}

\newcommand{\C}
{\secRev{\mathcal{H}^{+,\alpha}_{\expo}}}
\newcommand{\CR}{\secRev{\mathcal{H}^{+,\alpha}_{m,R_*}}}
\newcommand{\CRhalf}{\secRev{\mathcal{H}^{+,\alpha / 2}_m}}

\newcommand{\CRh}{\secRev{\mathcal{H}^{+,Q}_{m,R_*}}}
\newcommand{\KolmoDot}{\mathcal{L}^{\dot{a},\dot{b}}}

\newsiamremark{remark}{Remark}
\newsiamremark{hypothesis}{Hypothesis}
\crefname{hypothesis}{Hypothesis}{Hypotheses}
\newsiamthm{claim}{Claim}
\newsiamthm{metaTheorem}{Meta-Theorem}
\newsiamthm{theoremDefinition}{Theorem-Definition}

\headers{Learning of Stochastic Differential Equations with Fast Rates}{R. Bonalli and A. Rudi}

\title{Non-Parametric Learning of Stochastic Differential Equations with Non-Asymptotic Fast Rates of Convergence
\thanks{Submitted for review on May 24, 2023. Communicated by Thomas Strohmer on February 10, 2025. This version of the article has been accepted for publication, after peer review but is not the Version of Record and does not reflect post-acceptance improvements, or any corrections. The Version of Record is available online at: \url{https://doi.org/10.1007/s10208-025-09705-x}.}}

\author{Riccardo Bonalli\thanks{Laboratoire des Signaux et Syst\`emes, Université Paris-Saclay, CNRS, CentraleSupélec, Bât. Bréguet, 3 Rue Joliot Curie, 91190 Gif-sur-Yvette, France. \textit{E-mail}: {\tt riccardo.bonalli@cnrs.fr}.}
\and Alessandro Rudi\thanks{SDA Bocconi, Bocconi University, Italy. \textit{E-mail}: {\tt alessandro.rudi@sdabocconi.it}.}}

\usepackage{amsopn}


%% file: Intro.tex
Consider non-linear \textit{stochastic differential equations} of type
\begin{equation} \label{eq:IntroSDE}
    \mathrm{d}X(t) = b(t,X(t)) \, \mathrm{d}t + \sqrt{a(t,X(t))} \, \mathrm{d}W_t , \quad t \in [0,T] .
\end{equation}
Here, $X(t) \in \Rn$ denotes the (stochastic) state of dimension $n \in \mathbb{N}$ of the system at time $t \in [0,T]$, with $T > 0$ a fixed time horizon, $(W_t)_{t \in [0,T]}$ is a $n$--dimensional Wiener process, whereas $b : [0,T]\times\Rn\to\Rn$ and $a : [0,T]\times\Rn\to\SymPP$ are regular enough \textit{drift} and \textit{diffusion} coefficients, with $\SymPP \subseteq \Rnn$ the subset of symmetric positive-definite matrices.
Equations such as \eqref{eq:IntroSDE} may be profitably leveraged to accurately model complex phenomena in a wide range of applications, such as aerospace, finance, and robotics to name a few \cite{Ridderhof2019,Yong1999,Bonalli2022,Bonalli2023,Velho2024}.

In practice, both coefficients $a$ and $b$ might be completely unknown, e.g., $a$ might model external perturbations due to many different physical phenomena which affect the motion of the system, as it occurs in aerospace and robotics. Therefore, appropriate \textit{stochastic system identification} procedures for both $a$ and $b$ must be devised.

\subsection{Related Work}

The problem of stochastic system identification has been investigated for several decades. The earliest methods mainly address either discrete-time models \cite{Lennart1998,Nelles2001} or continuous-time models that are however linear in the state variable \cite{Garnier2008}, which thus do not fit the identification setting introduced by \eqref{eq:IntroSDE}.

More recently, identification of non-linear stochastic differential equations such as \eqref{eq:IntroSDE} have seen an important surge of interest. \firstRev{More specifically, besides some likelihood-based \cite{Liptser1997} or Kalman filtering-based estimation methods \cite{Surace2018,Bhudisaksang2021,Sharrock2022}, which assume continuous-time observation of the stochastic state, the vast majority of the existing methods leverage more realistic discrete-time observations of $X$. Efficient paradigms range from 1. non-parametric estimation \cite{Dacunha1986,Genon2014,Comte2021,Lavenant2021} and Bayesian estimation \cite{Meulen2013,Nickl2017}, to 2. maximum likelihood and quasi-likelihood methods \cite{Florens1989,Yoshida2011,Masuda2013,Uchida2013}, generalized methods of moments \cite{Gallant1997,Gallant1999,Nielsen2000}, and online gradient descent-based methods \cite{Kushner2003,Nakakita2022}, among others. 
On the one hand, although Group 1. of these works provides learning rates which become increasingly tighter when the number of observations of the state grows, estimates of the error between the unknown drift and diffusion coefficients and the learned ones are obtained by assuming diffusion coefficients are either known or have specific parametric structures. 
On the other hand, although Group 2. of the aforementioned works rely on generally milder assumptions to be implemented, they require working with given families of finite-dimensional parametric drift and diffusion coefficients, 
which might hinder the identification process in the case the family of parametric coefficients is not rich enough.} 

\firstRev{Importantly, none of the above works has investigated how to leverage the regularity, i.e., high order continuous differentiability of drift and diffusion coefficients to improve learning rates. Although such analysis has been recently undertaken in, e.g., \cite{Abraham2019,Nickl2020,Aeckerle2022,Marie2023}, these studies often work with SDE that are either scalar or perturbed by constant diffusion coefficients, and they offer learning rates that are often only asymptotic. Therefore, designing methods for the identification of multi-dimensional non-linear SDE with non-constant diffusion coefficients under discrete-time observations of the state remains a challenging open question (note that multivariate settings are known to generally require very different techniques, see, e.g., \cite{Nickl2020,Aeckerle2022}). In particular, a key part of the challenge lies in developing non-asymptotic rates of convergence that become tighter not only when the number of observations of the state grows, but also when the regularity of drift and diffusion coefficients is higher.}

Finally, from a numerical standpoint, learning-based methods, such as scalable gradient methods \cite{Li2020} and infinitely deep Bayesian neural networks \cite{Xu2022}, have shown remarkable performance on complex stochastic differential equations. Unfortunately, it seems extremely challenging to endow these latter paradigms with theoretical guarantees of convergence, motivating investigation of numerically efficient identification methods for \eqref{eq:IntroSDE} which enjoy guarantees of accuracy under mild assumptions.

\subsection{Outline and Contributions}

We propose a non-parametric, \textit{Reproducing Kernel Hilbert Space} (RKHS)-based learning paradigm for the identification of drift and diffusion coefficients of the \firstRev{multi-dimensional non-linear SDE} \eqref{eq:IntroSDE}, which relies upon discrete-time observation of the state. In particular, motivated by classical likelihood-based methods \cite{Nielsen2000} we propose a two-step, discrete-time observation-based scheme which entails fitting the \textit{Fokker-Planck equation} related to \eqref{eq:IntroSDE}. 

Under assumptions of smoothness for the unknown drift and diffusion coefficients, we provide theoretical estimates of \firstRev{non-asymptotic} learning rates which become increasingly tighter when the number of observations of the state grows. In particular, given the nature of our data set, which essentially depends on observations of the state process, we consider the error between solutions to \eqref{eq:IntroSDE} generated with the learned drift and diffusion coefficients and the unknown trajectories of \eqref{eq:IntroSDE} is a ``good metric'' to test the accuracy of our identification method. Importantly, under this accuracy metric we can additionally prove our learning rates become tighter when the regularity (in a Sobolev sense) of the unknown drift and diffusion coefficients is higher. Finally, from a numerical standpoint, our method being kernel-based, offline pre-processing may be successfully leveraged to enable efficient implementations, \firstRev{offering excellent balance between precision and computational complexity (details are in Section \ref{sec:Computational}).}

Our method is composed of two steps which are informally summarized below.

\vspace{5pt}
\newpage

    \noindent A. \ \textbf{Learning the laws of the stochastic differential equation through independent discrete-time observation of the state.}

    \vspace{5pt}
    
    We assume there exist regular enough, \firstRev{i.e. essentially $C^{2 m + 1}$, $m \in \mathbb{N}$}, drift $b : [0,T]\times\Rn\to\Rn$ and diffusion $a : [0,T]\times\Rn\to\SymPP$, and a stochastic process $X : [0,T]\times\Omega\to\Rn$ in some filtered probability space $(\Omega,\G,\F,\Proba)$, such that $X$ solves \eqref{eq:IntroSDE} with drift $b$ and diffusion $a$. In addition, we may assume we can sample $X$ at $M \in \mathbb{N}$ times $0 = t_1 < \dots < t_M = T$, and more specifically that at every time $t_{\ell}$, $\ell=1,\dots,M$, we have access to $N \in \mathbb{N}$ samples $X_{\ell,1} \triangleq X(t_\ell,\omega_1) , \dots , X_{\ell,N} \triangleq X(t_\ell,\omega_N)$ of the solution $X$ which have been independently drawn from 
    $\Proba_{X(t_{\ell},\cdot)}$.

    In such setting, as first step we propose to approximate the unknown densities $\unknownDensity : [0,T]\times\Rn\to\R$ of the laws of $X$ through the (random) RKHS-based model:
    \begin{align*}
    \approxDensity(t,x) \triangleq \sum_{\ell=1}^M c_{\ell}(t) \widehat{g}_{\ell}(x), \quad \textnormal{where} \quad \widehat{g}_{\ell}(x) \triangleq \frac{1}{N} \sum_{j=1}^N \rho_R(x - X_{\ell,j}) ,
    \end{align*}
    for appropriate coefficients $c_{\ell} : [0,T]\to\R$ and radial mappings $\rho_R$, $R > 0$. Let $\mu_X$ denote the probability measure which is generated by the process $X$. For every tuple of precision parameters $0 < \varepsilon , \delta < 1$, if the unknown drift $b$ and diffusion $a$ are regular enough, by leveraging RKHS approximation theory we can show appropriate values for $M$, $N$, and $R$ (which depend on $\varepsilon$ \firstRev{and $m$} uniquely) may be selected so that the following learning rate holds with probability $\mu_X$ at least $1 - \delta$:
    \begin{equation} \label{eq:estimateDensityIntro}
        \int^T_0 \left\| \frac{\partial \approxDensity}{\partial t}(t,\cdot) - \frac{\partial \unknownDensity}{\partial t}(t,\cdot) \right\|^2_{L^2} + \| \approxDensity(t,\cdot) - \unknownDensity(t,\cdot) \|^2_{H^2} \; \mathrm{d}t = 
        \firstRev{O\left( \left( \log\left( \frac{1}{\delta \varepsilon} \right)^{\frac{1}{2}} \varepsilon \right)^2 \right)} ,
    \end{equation}
    where $\| \cdot \|_{L^2}$ and $\| \cdot \|_{H^2}$ denote the norms of the Hilbert spaces $L^2(\Rn,\R)$ and $H^2(\Rn,\R)$, respectively. Moreover, the higher the degree of smoothness of $b$ and $a$ is, the lower the values of $M$, $N$, and $R$ needed to achieve this precision become.
    
    The main benefit which comes with the model $\approxDensity$ consists of computing accurate finite-dimensional approximations of the laws of solutions to \eqref{eq:IntroSDE} without a priori resorting to conservative families of parametric densities. In particular, the model $\approxDensity$ being kernel-based, one may considerably reduce the computational effort by resorting to prior offline computations, which essentially boil down to simply inverting a $M \times M$ matrix. As a byproduct, \eqref{eq:estimateDensityIntro} provides a quantitative estimate of the approximation error which is key to derive theoretical guarantees for the accuracy of our identification method in the next step.

    \vspace{5pt}
    
    \noindent B. \ \textbf{Learning finite-dimensional models for the drift and diffusion coefficients by fitting approximated solutions to the Fokker-Planck equation.}

    \vspace{5pt}
    
    If the unknown drift $b$ and diffusion $a$ are regular enough, the unknown densities $p$ satisfy the following Fokker-Planck equation:
    \begin{equation} \label{eq:FPEIntro}
        \frac{\partial p}{\partial t}(t,y) = (\mathcal{L}^{a,b}_t)^* p(t,y) , \quad (t,y) \in [0,T]\times\Rn ,
    \end{equation}
    where $(\mathcal{L}^{a,b}_t)^*$ is the dual operator of the Kolmogorov generator
    $$
    \mathcal{L}^{a,b}_t \varphi(y) \triangleq \frac{1}{2} \sum^n_{i,j=1} a_{ij}(t,y) \frac{\partial^2 \varphi}{\partial y_i \partial y_j}(y) + \sum^n_{i=1} b_i(t,y) \frac{\partial \varphi}{\partial y_i}(y) , \quad \varphi \in C^2(\Rn,\R) .
    $$
    Given the results at the previous step, it is then natural to learn models of the drift $\solAppDrift : [0,T]\times\Rn\to\Rn$ and the diffusion $\solAppDiffusion : [0,T]\times\Rn\to\SymPP$ which ``best'' match the Fokker-Planck equation when evaluated at $\approxDensity$, that is as solutions to the following finite-dimensional convex optimization problem:
    \begin{equation} \label{eq:LPFinalIntro}
        \underset{(\solAppDiffusion,\solAppDrift) \in \mathcal{H}_Q}{\min} \ \int^T_0 \left\| \frac{\partial \approxDensity}{\partial t}(t,\cdot) - (\mathcal{L}^{\solAppDiffusion,\solAppDrift}_t)^* \approxDensity(t,\cdot) \right\|^2_{L^2} \; \mathrm{d}t + \lambda \| (\solAppDiffusion,\solAppDrift) \|^2_{\mathcal{H}} ,
    \end{equation}
    where $\mathcal{H}_Q$ is an appropriate finite-dimensional subspace of a RKHS $\mathcal{H}$ with norm $\| \cdot \|_{\mathcal{H}}$, in which the unknown drift and diffusion coefficients lie, whereas $\lambda > 0$ is a regularization weight to be appropriately selected. \firstRev{As detailed in Section~\ref{sec:Computational}, problem \eqref{eq:LPFinalIntro} admits a finite dimensional characterization that can be solved exactly, i.e., no approximation or integral discretization is needed, with reduced computational cost.}
    
    Thanks to estimate \eqref{eq:estimateDensityIntro}, one shows that, for appropriate choices of the subspace $\mathcal{H}_Q$ and the regularization weight $\lambda$, the solution $(\solAppDiffusion,\solAppDrift) \in \mathcal{H}_Q$ to problem \eqref{eq:LPFinalIntro} satisfies the following, with probability $\mu_X$ at least $1 - \delta$:
    \begin{align} \label{eq:estimateCoeffIntro}
        &\int^T_0 \left\| \frac{\partial \approxDensity}{\partial t}(t,\cdot) - (\mathcal{L}^{\solAppDiffusion,\solAppDrift}_t)^* \approxDensity(t,\cdot) \right\|^2_{L^2} \; \mathrm{d}t + \lambda \| (\solAppDiffusion,\solAppDrift) \|^2_{\mathcal{H}} = \nonumber \\
        &\hspace{50ex}\firstRev{= O\left( \left( \log\left( \frac{1}{\delta \varepsilon} \right)^{\frac{1}{2}} \varepsilon \right)^2 \right)} .
    \end{align}
    By combining estimate \eqref{eq:estimateCoeffIntro} with energy-type estimates for parabolic partial differential equations, we may infer theoretical estimates of learning rates for the identification of drift and diffusion coefficients of non-linear stochastic differential equations, which we informally summarize as follows.

    As we mentioned previously, since our data consists of observation of the state process, the error between the unknown densities and the densities stemming from the learned coefficients is a ``good metric'' with which the convergence of an identification algorithm for stochastic differential equations which leverages observations of the state process may be tested. To better formalize this metric, let $X^{\solAppDiffusion,\solAppDrift}$ and $p_{\solAppDiffusion,\solAppDrift}$ respectively denote the solutions to \eqref{eq:IntroSDE} and \eqref{eq:FPEIntro} with coefficients $(\solAppDiffusion,\solAppDrift) \in \mathcal{H}_H$. We thus define the following metric to test the accuracy:
    \begin{equation*} 
        \Metric(\solAppDiffusion,\solAppDrift) \triangleq \| p_{\solAppDiffusion,\solAppDrift} - \unknownDensity \|^2_{L^2} = \int^T_0 \| p_{\solAppDiffusion,\solAppDrift}(t,\cdot) - \unknownDensity(t,\cdot) \|^2_{L^2} \; \mathrm{d}t .
    \end{equation*}
    By adopting this metric, our main result on the accuracy of our learning method may be summarized in the following meta-theorem:
    
    \begin{metaTheorem}
    Assume the unknown drift $b$ and diffusion $a$ coefficients are regular enough. The following estimate holds true with probability $\mu_X$ at least $1 - \delta$:
    $$
    \Metric(\solAppDiffusion,\solAppDrift) = 
    \firstRev{O\left( \left( \log\left( \frac{1}{\delta \varepsilon} \right)^{\frac{1}{2}} \varepsilon \right)^2 \right)}
    $$
    In particular, for every regular enough function $f : [0,T]\times\Rn\to\R$:
    \begin{equation} \label{eq:ultimateEstimateIntro}
        \mathbb{E}\left[ \int^T_0 f(t,X(t)) \; \mathrm{d}t \right] = \mathbb{E}\left[ \int^T_0 f(t,X^{\solAppDiffusion,\solAppDrift}(t)) \; \mathrm{d}t \right] + 
        \firstRev{O\left( \log\left( \frac{1}{\delta \varepsilon} \right)^{\frac{1}{2}} \varepsilon \right)} .
    \end{equation}
    \end{metaTheorem}

Although the estimates provided in this meta-theorem are informal and need some clarification (see Section \ref{sec:LearningDynamics} for more rigorous statements), they show the aforementioned two-step identification method enjoys practical theoretical guarantees of accuracy, i.e., estimate \eqref{eq:ultimateEstimateIntro}: an observation/regulation metric $f$ computed at the unknown state solution to \eqref{eq:IntroSDE} with unknown drift $b$ and diffusion $a$ may be rather observed through the process solution to \eqref{eq:IntroSDE} with model drift $\solAppDrift$ and diffusion $\solAppDiffusion$ up to an error $O\left( \log\left( \frac{1}{\delta \varepsilon} \right)^{\frac{1}{2}} \varepsilon \right)$. Such result has important implications in observation and regulation of stochastic differential equations \cite{Lavenant2021}, and it may represent a good starting result to develop paradigms for the identification of \textit{controlled stochastic differential equations}, which are crucial for the control of complex autonomous systems.

\subsection{Paper Organization}

The paper is organized as follows. After gathering basic notation and preliminary results in Section \ref{sec:Notation}, in Section \ref{sec:ShortFokkerPlanck} we summarize both classical and less classical results about stochastic differential equations and corresponding relationships with the Fokker-Planck equation. To ease the reading of this section, we moved a more detailed description of the aforementioned results to Section \ref{sec:FokkerPlanck} and their technical proofs to Appendix \ref{sec:AppendixFokkerPlanck}. Similarly, in Section \ref{sec:RKHS} we expose the results in RKHS theory which we leverage in this work. Our main contributions are contained in Section \ref{sec:LearningDynamics}, which in particular details the methodologies we exposed at the previous steps 1) and 2) and corresponding learning rates. \firstRev{Precision and computational complexity of our approach are discussed in Section \ref{sec:Computational}.} Finally, in Section \ref{sec:Conclusion} we provide concluding remarks and some perspectives. 

%% file: Notation.tex
We fix the dimension $n \in \mathbb{N}$ of the state space and a time horizon $T > 0$ for the identification process. We denote by $\SymP$ and $\SymPP$ the subsets of $\Rnn$ of symmetric semi-positive-definite and symmetric positive-definite matrices, respectively.

We assume stochastic differential equations are defined on a given filtered probability space $(\Omega,\G\triangleq\F_T,\F\triangleq(\Ft)_{t \in [0,T]},\Proba)$, which is complete, and the noise is generated by a $\F$--adapted Wiener process $W : [0,T] \times \Omega \to \Rn$ (e.g., we may consider the canonical process in the space $\Omega = C([0,T],\Rn)$, equipped with the Wiener measure). Moreover, we introduce the complete metric space $(\DistSpace,\Dist)$, where
$$
\DistSpace \triangleq C([0,T],\Rn) , \quad \Dist(w_1,w_2) \triangleq \underset{t \in [0,T]}{\sup} \ \| w_1(t) - w_2(t) \| ,
$$
and equip it with the Borel sigma-algebra $\BorelDist$ induced by the metric $\Dist$.

For any $\ell \in \mathbb{N}$, $r , s > 0$, and any $A \subseteq \R^{\ell}$, we denote by $H^{r,s}([0,T]\times\Rn,A)$ the Sobolev (Hilbert) space of functions whose image is in $A$, and whose weak time derivatives are defined up to order $r$ and whose weak space derivatives are defined up to order $s$; in particular, we denote $H^r([0,T]\times\Rn,A) \triangleq H^{r,r}([0,T]\times\Rn,A)$. Finally, we introduce the following Hilbert space and corresponding positive convex cone, in which we will assume the unknown drift and diffusion coefficients lie (this latter requirement will be made more explicit shortly):
\begin{align*}
    \Cbasis \triangleq H^{d(m)}([0,T]\times\Rn,\Rnn) \times H^{d(m)}([0,T]\times\Rn,\Rn) , \\
    \Cweak \triangleq \Big\{ (a,b) :[0,T]\times\Rn \to \SymP\times\Rn : \ (a,b) \in \Cbasis \Big\} ,
\end{align*}
for every $m \in \mathbb{N}$, where $d(m) \triangleq 2 (m + 1) + \left\lfloor \frac{n}{2} \right\rfloor \in \mathbb{N}$ is the unique integer greater than $2 m + 1$ such that $\left\lfloor d(m) - \frac{n}{2} \right\rfloor = 2 m + 1$. The choice of taking the same exponent $d(m)$ for both the drift and the diffusion coefficients has been made without loss of generality for the sake of conciseness. The following regularity result will be crucial:

\begin{theorem} \label{Theo:CoeffProperties}
There exists a constant $C > 0$, which depends on $\expo$ uniquely, such that every $(a,b) \in \Cweak$ satisfies the following properties:
\begin{enumerate}
    \item Differentiability:
    $$
    a \in C^{2 \expo + 1}([0,T]\times\Rn,\Rnn) , \quad b \in C^{2 \expo + 1}([0,T]\times\Rn,\Rn) .
    $$
    
    \item Boundness of functions and their derivatives:
    $$
    \sum^{2 m + 1}_{i=0} \Big( \| D^i_{(t,y)} a \|_{L^{\infty}} + \| D^i_{(t,y)} b \|_{L^{\infty}} \Big) \le C \| (a,b) \|_{\Cbasis} ,
    $$
    where \firstRev{$D^i_{(t,y)}$} denotes the differential of order $i$ with respect to $(t,y)$. In particular, the following refined bound holds:
    $$
    \sum^2_{i=0} \Big( \| D^i_{(t,y)} a \|_{L^{\infty}} + \| D^i_{(t,y)} b \|_{L^{\infty}} \Big) \le C \| (a,b) \|_{H^{d(m) - 2 m + 1}} .
    $$
    %
    
\end{enumerate}
\end{theorem}

The proof of this result makes use of classical embedding arguments and it is reported in Appendix \ref{sec:AppendixFokkerPlanck} for the sake of completeness.

%% file: ShortFokkerPlanck.tex
In this section, we summarize both classical and less classical results about stochastic differential equations, and corresponding relationships with the Fokker-Planck equation. In particular, we propose a minimally detailed discussion for the sake of conciseness, reporting a more structured exposition in Section \ref{sec:FokkerPlanck} for the sake of completeness.

\subsection{The Fokker-Planck equation}


From now on, we fix $m \in \mathbb{N}$ and a constant $\alpha > 0$ (to be specified later). Since we will need to work with diffusion coefficients that are never trivial, \secRev{we replace $\Cweak$ with}
$$
\secRev{\C \triangleq (\alpha I , 0) + \Big\{ (a,b) \in \Cweak : \ a(t,x) \succcurlyeq 0 , \ (t,x) \in [0,T] \times \Rn \Big\} .} 
$$
At this step, we fix a non-negative density $p_0 \in L^2(\Rn,\R)$ which will serve as appropriate initial condition, and we denote by $\mu_0 \in \ProbaSpace(\Rn)$ the associate probability measure. Motivated by \cite{Stroock1997,Figalli2008}, we recall the following notions of stochastic differential equation and its solutions (though they differ from the ones in \cite{Stroock1997,Figalli2008}):

\begin{theoremDefinition} \label{def:ShortSDE}
There exists a measurable mapping $X : \Rn\times\Omega\to\DistSpace$ such that each process $X_x(t,\omega) \triangleq X(x,\omega)(t)$ is $\F$--progressively measurable for every $x \in \Rn$, and such that the following Stochastic Differential Equation (SDE)
$$
\textnormal{SDE}_x \quad \begin{cases}
\mathrm{d}X_x(t) = b(t,X_x(t)) \; \mathrm{d}t + \secRev{\sqrt{a(t,X_x(t))}} \; \mathrm{d}W_t , \\[5pt]
\Proba\big( X_x(0) = x \big) = 1 ,
\end{cases}
$$
holds in $(\Omega,\G,\F,\Proba)$ for $\mu_0$-almost every $x \in \Rn$. We then say that $X$ solves or is solution to SDE (with coefficients $(a,b) \in \C$). The solution $X$ to SDE is unique in the following sense: if a measurable mapping $Y : \Rn\times\Omega\to\DistSpace$ solves SDE, then it holds that $X(x,\cdot) = Y(x,\cdot)$ a.s., for $\mu_0$-almost every $x \in \Rn$.
\end{theoremDefinition}

For the well-posedness of Theorem-Definition \ref{def:ShortSDE}, see Section \ref{sec:FokkerPlanck}. Solutions to SDE share a close relationship with the solutions to the Fokker-Planck equation, which we introduce next. For every $(a,b) \in \C$, we denote the Kolmogorov generator by
$$
\Kolmo_t \varphi(y) \triangleq \frac{1}{2} \sum^n_{i,j=1} \secRev{a}_{ij}(t,y) \frac{\partial^2 \varphi}{\partial y_i \partial y_j}(y) + \sum^n_{i=1} b_i(t,y) \frac{\partial \varphi}{\partial y_i}(y) , \quad \varphi \in C^2(\Rn,\R) .
$$
Let $X : \Rn\times\Omega\to\DistSpace$ be solution of SDE with coefficients $(a,b) \in \C$. Thanks to specific regularity properties of $X$ (we refer to Section \ref{sec:FokkerPlanck} for these latter), we may define the curve $\mu : [0,T] \to \ProbaSpace(\Rn)$ of probability measures
\begin{equation} \label{eq:ShortmesureFokker}
    \mu_t(A) \triangleq \int_{\Rn} \int_{\Omega} \mathds{1}_{\{ X_x(t) \in A \}} \; \mathrm{d}\Proba \; \mu_0(\mathrm{d}x) , \quad A \in \Borel ,
\end{equation}
and note that $\mu$ is narrowly continuous, i.e., for every $\varphi \in C_b(\Rn,\R)$, the mapping
$$
t \in [0,T] \mapsto \int_{\Rn} \varphi(y) \mu_t(\mathrm{d}y) = \int_{\Rn} \int_{\Omega} \varphi(X_x(t)) \; \mathrm{d}\Proba \; \mu_0(\mathrm{d}x) \in \R
$$
is continuous. By combining this latter property with the results in \cite{Figalli2008}, we introduce the following notions of Fokker-Planck equation, its solutions, and additional relationship between these solutions and the solutions to SDE (see also Section \ref{sec:FokkerPlanck}):

\begin{theoremDefinition} \label{def:ShortFPE}
The curve $\mu : [0,T] \to \ProbaSpace(\Rn)$ defined through \eqref{eq:ShortmesureFokker} is the unique narrowly continuous curve satisfying the Fokker-Planck Equation
$$
\textnormal{FPE} \quad \begin{cases}
\displaystyle \frac{\mathrm{d}}{\mathrm{d}t} \int_{\Rn} \varphi(y) \mu_t(\mathrm{d}y) = \int_{\Rn} \Kolmo_t \varphi(y) \mu_t(\mathrm{d}y) , \quad \varphi \in C^{\infty}_c(\Rn,\R) , \\[5pt]
\mu_{t=0} = \mu_0 .
\end{cases}
$$
We then say that $\mu$ solves or is solution to FPE (with coefficients $(a,b) \in \C$). If $X$ denotes the solution to SDE, the following representation formula holds:
\begin{equation} \label{eq:RepFormula}
    \int_{\Rn} \varphi(y) \mu_t(\mathrm{d}y) = \int_{\Rn} \int_{\Omega} \varphi(X_x(t)) \; \mathrm{d}\Proba \; \mu_0(\mathrm{d}x) , \quad \textnormal{for} \ t \in [0,T] , \ \varphi \in C_c(\Rn,\R) .
\end{equation}
\end{theoremDefinition}

\subsection{Absolutely continuous solutions to the Fokker-Planck equation}

Thanks to the regularity of the coefficients $(a,b) \in \C$ which is offered through Theorem \ref{Theo:CoeffProperties}, the solution $\mu$ to FPE is absolutely continuous, i.e., it takes the form
$$
\mu_t(A) = \int_A p(t,y) \; \mathrm{d}y , \quad A \in \Borel ,
$$
for an appropriate $p : [0,T]\times\Rn\to\R$. To elucidate this property, we first introduce broader definitions of FPE and corresponding solutions, which encompass Theorem-Definition \ref{def:ShortFPE} as a sub-case (see Theorem \ref{theor:ShortexFPE} below) and will be useful in our analysis:

\begin{definition}
Let $f \in L^2([0,T]\times\Rn,\R)$ and $\bar p \in L^2(\Rn,\R)$. A (regular enough) function $p : [0,T]\times\Rn\to\R$ is said to solve or be solution to the non-homogeneous Fokker-Planck Equation (with coefficients $(a,b) \in \C$), if
$$
\textnormal{FPE}_f \quad \begin{cases}
\displaystyle \frac{\mathrm{d}}{\mathrm{d}t} \int_{\Rn} \varphi(y) p(t,y) \; \mathrm{d}y = \\
\displaystyle \hspace{10ex}= \int_{\Rn} \left( \Kolmo_t \varphi(y) p(t,y) + f(t,y) \varphi(y) \right) \; \mathrm{d}y , \quad \varphi \in C^{\infty}_c(\Rn,\R) , \\[10pt]
p(0,\cdot) = \bar p(\cdot) .
\end{cases}
$$
\end{definition}

The next theorem gathers important properties of the solution to $\textnormal{FPE}_f$, additionally showing the solution to FPE is absolutely continuous. We refer the reader to Section \ref{sec:FokkerPlanck} for exhaustive presentation and proof of these results.

\begin{theorem} \label{theor:ShortexFPE}
For any $(a,b) \in \C$, $f \in L^2([0,T]\times\Rn,\R)$, and $\bar p \in L^2(\Rn,\R)$, there exists a unique mapping $p \in C(0,T;L^2(\Rn,\R)) \cap H^{0,1}([0,T]\times\Rn,\R)$ which solves FPE$_f$ with coefficients $(a,b) \in \C$. In addition, the following estimate holds:
\begin{align} \label{eq:Shortestimate1}
    \| p (t,\cdot) \|^2_{L^2} + \int^t_0 &\| p(s,\cdot) \|^2_{H^1} \; \mathrm{d}s \le \nonumber \\
    &\le C\big(\secRev{\alpha},\|(\secRev{a - \alpha I},b)\|_{\Cbasis}\big) \left( \| \bar p \|^2_{L^2} + \int^t_0 \| f(s,\cdot) \|^2_{L^2} \; \mathrm{d}s \right) , \quad t \in [0,T] ,
\end{align}
\secRev{with $C\big(\secRev{\alpha},\|(a - \alpha I,b)\|_{\Cbasis}\big) > 0$ continuously depending on $\secRev{\alpha}$ and $\|(a - \alpha I,b)\|_{\Cbasis}$}.

Assume $f = 0$ and $\bar p = p_0$. Then, it holds that
$$
p(t,\cdot) \ge 0 , \quad \int_{\Rn} p(t,y) \; \mathrm{d}y = 1 , \quad t \in [0,T] ,
$$
and therefore, if for every $t \in [0,T]$ we define
$$
\mu_t(A) \triangleq \int_A p(t,y) \; \mathrm{d}y , \quad A \in \Borel ,
$$
then the curve $\mu : [0,T] \to \ProbaSpace(\Rn)$ is narrowly continuous and solves FPE. Finally, if in addition $p_0 \in H^{2 m + 1}(\Rn,\R)$, then the function $p$ satisfies
$$
p \in H^{m+1,2(m + 1)}([0,T]\times\Rn,\R)
$$
and the following Strong Fokker-Planck Equations (with coefficients $(a,b) \in \C$):
$$
\textnormal{SFPE} \quad \begin{cases}
\displaystyle \frac{\partial p}{\partial t}(t,y) = (\Kolmo_t)^* p(t,y) , \quad \textnormal{a.e.} \quad (t,y) \in [0,T]\times\Rn , \\[10pt]
p(0,\cdot) = p_0(\cdot) ,
\end{cases}
$$
where $(\Kolmo_t)^*$ denotes the dual operator of the Kolmogorov generator $\Kolmo_t$.
\end{theorem}

Combining Theorem \ref{theor:ShortexFPE} with the representation formula \eqref{eq:RepFormula} allows us to introduce criteria to establish satisfactory guarantees for our learning approach in the context of observation and regulation of stochastic differential equations (a more exhaustive presentation is provided in Section \ref{sec:FokkerPlanck}). More specifically, let $f \in L^2([0,T]\times\Rn,\R)$ be an observation/regulation integral metric and $(a,b)\in\C$. By denoting $X^{a,b}$ and $p_{a,b}$ respectively the solution to SDE and to SFPE with coefficients $(a,b)\in\C$, we are interested in studying the accuracy with which the following observation/regulation metric is approximated by our learning approach:
\begin{equation} \label{eq:ShortintegralMetric}
    \mathbb{E}_{\mu_0\times\Proba}\left[ \int^T_0 f(t,X^{a,b}_x(t)) \; \mathrm{d}t \right] = \int^T_0 \int_{\Rn} f(t,y) p_{a,b}(t,y) \; \mathrm{d}y \mathrm{d}t .
\end{equation}
To give precise estimates of the approximation error for \eqref{eq:ShortintegralMetric}, the following corollary of Theorem \ref{theor:ShortexFPE} will be crucial (see Section \ref{sec:FokkerPlanck} for a proof):

\begin{corollary} \label{Corol:ShortestimateOperator}
For every $(a_1,b_1) , (a_2,b_2) \in \C$, it holds that
\begin{align} \label{eq:ShortestimateOperator}
    \Bigg| \mathbb{E}_{\mu_0\times\Proba}\Bigg[ \int^T_0 f(t,X^{a_1,b_1}_x(t)) \; \mathrm{d}t \Bigg] - \mathbb{E}_{\mu_0\times\Proba}\Bigg[ \int^T_0 f(t,&X^{a_2,b_2}_x(t)) \; \mathrm{d}t \Bigg] \Bigg| \nonumber \\
    &\le \| f \|_{L^2} \| p_{a_1,b_1} - p_{a_2,b_2} \|_{L^2} .
\end{align}
\end{corollary}

%% file: RKHS.tex

In this section, we list classical results about Sobolev Spaces of functions with scattered zeros and Reproducing Kernel Hilbert Spaces (RKHS), which will be extensively leveraged in the following sections and for which we mainly refer to \cite{aronszajn1950theory,Wendland2004,steinwart2008support}.

Let $r , s \in \mathbb{N}$ such that $r > s / 2$. Every function $u \in H^r(\R^s,\R)$ satisfies
\begin{equation}\label{eq:bessel-sobolev}
C \| u \|_{L^{\infty}} \le \| u \|_{H^r} = \| J_r(\cdot) ({\cal F} u)(\cdot) \|_{L^2} ,
\end{equation}
where $J_r(z) \triangleq (1 + \| z \|^2)^{r/2} / (2 \pi)^{s/2}$, $z \in \R^s$ denotes the Bessel potential and ${\cal F} u$ the Fourier transform of $u$, whereas the constant $C > 0$ depends on $s$ uniquely.

Let $\ell \in \mathbb{N}$, $D \subseteq \R^s$ be a bounded domain, and $u \in H^r(D,\R)$ be a function such that $u|_{\widehat{X}} = 0$, where $\widehat{X}_{\ell} \triangleq \{ x_1 , \dots , x_{\ell} \} \subseteq \R^s$ is a finite set of $\ell$ given points. For the \textit{fill distance}, which is defined by
$$
h_{\widehat{X}_{\ell},D} \triangleq \sup_{x \in D} \ \min_{i=1,\dots,\ell} \ \| x - x_i \| ,
$$
the following inequality holds true, e.g., \cite[Theorem 11.32]{Wendland2004},
\begin{equation} \label{eq:wendland}
\| u \|_{H^{\nu}} \le C h_{\widehat{X}_{\ell},D}^{r-\nu} \| u \|_{H^r} , \quad 0 \le \nu \le r ,
\end{equation}
where the constant $C > 0$ depends on $\nu$, $r$, and $D$ uniquely. Inequality \ref{eq:wendland} may be interpreted as follows: if $u$ is zero on a well distributed set of points over $D$, i.e., $h_{\widehat{X}_{\ell},D}$ is small, and is very regular, i.e., $r$ is large, then any norm $\|\cdot\|_{H^{\nu}}$, $0 \le \nu \le r$, e.g., the $L^2$ norm for \firstRev{$\nu=0$}, is small over the whole domain $D$.

Given a set $D$, a RKHS $\mathcal{H}_D$ is a separable Hilbert space of functions $u : D \to \R$ such that the following \textit{reproducing property} holds true:
\begin{definition}[Reproducing property \cite{aronszajn1950theory}]
For every point $x \in D$ there exists a {\em reproducing function} $k_x \in \mathcal{H}_D$  such that
$$
f(x) = \scal{f}{k_x}_{\mathcal{H}_D}, \quad f \in \mathcal{H}_D.
$$
\end{definition}
We denote $K_D : D \times D \to \R$, the {\em reproducing kernel} associated to $\mathcal{H}_D$, i.e.,
$$
K_D(x_1,x_2) \triangleq \scal{k_{x_1}}{k_{x_2}}_{\mathcal{H}_D}, \quad x_1 , x_2 \in D .
$$
The following crucial result holds true, see, e.g., \cite{steinwart2008support}:
\begin{theorem}
Given a reproducing kernel for a RKHS $\mathcal{H}_D$, by the reproducing property, the reproducing function $k_x \in \mathcal{H}_D$ for any $x \in D$ corresponds to
$$
k_x = K_D(x,\cdot) \in \mathcal{H}_D .
$$
\end{theorem}
We recall that the Sobolev space $H^r(D,\R)$ is a RKHS for $r > s/2$, and $D \subseteq \R^s$ be either $\R^s$ or an open domain with Lipschitz boundary and $s \in \mathbb{N}$. In this case, the associated reproducing kernel has a known closed form, see, e.g., \cite{Wendland2004}.

%% file: LearningDynamics.tex
In this section, we finally introduce and study the problem of learning drift and diffusion coefficients of a stochastic differential equation. For this, from now on we assume the following hypothesis, which naturally stems from our setting, to hold true (see also Definition \ref{def:ShortSDE}):

\begin{Itemize}
    \item[$(A)$] Let $m \in \mathbb{N}$ with $m \ge 1$, $\alpha > 0$, and $p_0 \in H^{2 m + 1}(\Rn,\R)$ \secRev{(these will be considered as hyper-parameters in this work, \thirdRev{see the next remark})}. There exist $(\unknownDiffusion,\unknownDrift) \in \C$ 
    and a (unique) solution $\unknownX$ to SDE with coefficients $(\unknownDiffusion,\unknownDrift) \in \C$.
\end{Itemize}

\thirdRev{
\begin{remark}
    The quantities $m \in \mathbb{N}$ with $m \ge 1$, $\alpha > 0$, and $p_0 \in H^{2 m + 1}(\Rn,\R)$, appearing in Assumption (A), play the role of hyper-parameters in this work. It is common and reasonable to assume prior knowledge of the initial distribution $p_0$ and to have an informed estimate of the degree of smoothness $m$ \cite{Wood2016}. Regarding the parameter $\alpha$, we emphasize that numerous examples of stochastic differential equations frequently employed in various applications have diffusion coefficients $a$ that are uniformly lower-bounded by known $\alpha > 0$. Notable examples include pricing in finance \cite{Baxter1996}, temperature modeling in electricity markets \cite{Prabakaran2020}, and dynamics modeling in planetary landing \cite{Exarchos2019}, among others. Informed estimates of $\alpha$ can be obtained using methods such as doubling trick-type algorithms \cite{Shalev2012}, which involve learning over a finite decreasing sequence of $\alpha$. Exploring more advanced estimation techniques for $\alpha$ and analyzing how corresponding estimates influence the learning process lie beyond the scope of this work and are left for future research.
\end{remark}}

Thanks to the results we gathered in Section \ref{sec:ShortFokkerPlanck}, one readily checks that Assumption $(A)$ yields the following characterization of the mapping $\unknownX$:

\begin{corollary} \label{corol:unknownDensity}
With the notation $\unknownKolmo_t \triangleq \mathcal{L}^{\unknownDiffusion,\unknownDrift}_t$, there exists a unique function $\unknownDensity \in H^{m+1,2(m+1)}([0,T]\times\Rn,\R)$ which is non-negative and such that
$$
\int_A \unknownDensity(t,x) \; \mathrm{d}x = \mu_t(A) \triangleq \int_{\Rn} \int_{\Omega} \mathds{1}_{\{ \unknownX_x(t) \in A \}} \; \mathrm{d}\Proba \; \mu_0(\mathrm{d}x) , \quad \textnormal{for all} \quad A \in \Borel ,
$$
and which satisfies:
$$
\frac{\partial \unknownDensity}{\partial t}(t,y) = \unknownKolmo_t^* \unknownDensity(t,y) , \quad \textnormal{a.e.} \quad (t,y) \in [0,T]\times\Rn .
$$
\end{corollary}

To achieve our goal, we proceed along three successive steps:
\begin{enumerate}
    \item First, we approximate the unknown curve of densities $\unknownDensity$ from given samples $\{ X_{x_j}(t_{\ell},\omega_j) \}_{\ell=1,\dots,M , j=1,\dots,N}$ of the unknown solution $X$ to SDE, through a finite-dimensional RKHS-based model $\approxDensity$. Thanks to this, all the quantities appearing in the learning problems we define in the next steps can be actually (tractably) computed. As is customary in many applications, we assume our data are collected by independently sampling $X$ at pre-defined times.

    \item Second, we approximate the unknown coefficients $(\unknownDiffusion,\unknownDrift) \in \C$ through additional coefficients $(\solDiffusion,\solDrift) \in \C$ which ``best'' match SFPE when evaluated at $\approxDensity$, through an appropriately well-posed infinite-dimensional learning problem. We then combine appropriate estimates with \eqref{eq:Shortestimate1} to show the solution to SFPE with coefficients $(\solDiffusion,\solDrift)$ well-approximates the unknown $\unknownDensity$.

    \item Third, we make the learning problem at the second step ``tractable'' by approximating the coefficients $(\solDiffusion,\solDrift)$ through a finite-dimensional RKHS-based model $(\solAppDiffusion,\solAppDrift)$. We then combine the estimates we obtained at the second step with \eqref{eq:Shortestimate1} to show the solution to SFPE with coefficients $(\solAppDiffusion,\solAppDrift)$ well-approximates the unknown $p$. Finally, this latter property is combined with \eqref{eq:ShortestimateOperator} to show theoretical error bounds in the context of observation and regulation of stochastic differential equations.
\end{enumerate}

\subsection{Approximating solutions to SDE via RKHS-based models} \label{sec:Papprox}

We start by defining our data set. We assume sampling happens at $M \in \mathbb{N}$ fixed times $0 = t_1 < \dots < t_M = T$, which are equally spaced for simplicity, i.e., $t_{\ell} \triangleq T (\ell - 1)/(M - 1)$. Then, we assume    at each time $t_{\ell}$, $\ell=1,\dots,M$, we have access to $N \in \mathbb{N}$ samples $X_{\ell,1} \triangleq X_{x_1}(t_\ell,\omega_1) , \dots , X_{\ell,N} \triangleq X_{x_N}(t_\ell,\omega_N)$ of the solution $X$ to SDE, which have been independently drawn from the same probability $\mu_{t_{\ell}}$, with density $\unknownDensity(t_{\ell},\cdot)$ (see Corollary \ref{corol:unknownDensity}). Now, for every $t \in [0,T]$ the family of probability measures given by
$$
\mu_{t,k} \triangleq \underbrace{\mu_t \otimes \dots \otimes \mu_t}_{k-\textnormal{times}} : \mathcal{B}( \underbrace{\Rn \times \dots \times \Rn}_{k-\textnormal{times}} ) \to [0,1] ,\quad k \in \mathbb{N}
$$
may be extended to a unique probability measure $\mu^{\mathbb{N}}_t : \mathcal{B}^{\mathbb{N}} \to [0,1]$ thanks to Kolmogorov (extension) theorem, and by definition, for every $\ell=1,\dots,M$ and $N \in \mathbb{N}$ the samples $\{ X_{\ell,j} \}_{j=1,\dots,N}$ may be seen as independent random variables in the probability space $((\Rn)^{\mathbb{N}},\mathcal{B}^{\mathbb{N}},\mu^{\mathbb{N}}_{t_{\ell}})$, with equal density $\unknownDensity(t_{\ell},\cdot)$. Finally, by one further application of Kolmogorov theorem, we extend the family of probability measures $\{ \mu^{\mathbb{N}}_{s_1} \otimes \dots \otimes \mu^{\mathbb{N}}_{s_k} \}_{0 \le s_1 \le \dots \le s_k \le T }$ to a unique probability measure $\ProbaX : (\mathcal{B}^{\mathbb{N}})^{[0,T]} \to [0,1]$, so that, with an abuse of notation, $\ProbaX|_{t_1,\dots,t_M} = \mu^{\mathbb{N}}_{t_1} \otimes \dots \otimes \mu^{\mathbb{N}}_{t_M}$ for every $M \in \mathbb{N}$, and for every $M , N \in \mathbb{N}$ the samples $\{ X_{\ell,j} \}_{\ell=1,\dots,M,j=1,\dots,N}$ may be seen as (not necessarily i.i.d.) random variables in the probability space $\big(((\Rn)^{\mathbb{N}})^{[0,T]},(\mathcal{B}^{\mathbb{N}})^{[0,T]},\ProbaX\big)$. Below, the assessment ``with probability at least'' will be meant with respect to $\ProbaX$.


At this step, fix $M, N \in \mathbb{N}$ and define the following model (random) density
\begin{align*}
\approxDensity(t,x) \triangleq \sum_{\ell=1}^M c_{\ell}(t) \widehat{g}_{\ell}(x), \quad \textnormal{where} \quad \widehat{g}_{\ell}(x) \triangleq \frac{1}{N} \sum_{j=1}^N \rho_R(x - X_{\ell,j})
\end{align*}
with $\rho_R(x) \triangleq R^{n/2}\|x\|^{-n/2} B_{n/2}(2\pi R \|x\|)$, $R > 0$ and $B_{n/2}$ is the Bessel $J$ function of order $n/2$, while $c_{\ell}(t) \triangleq e^\top_{\ell} G^{-1} v(t)$ with $\{ e_1 , \dots , e_M \}$ the canonical basis of $\R^M$ and we define $v(t) \triangleq (K_{m+1}(t,t_1),\dots, K_{m+1}(t,t_M))$. The notation $G$ is used for the Gram matrix with elements $G_{j,k} \triangleq K_{m+1}(t_j,t_k)$, where $K_{m+1}$ denotes the Sobolev kernel of smoothness degree $m+1$ (see, e.g., \cite[Page 133]{Wendland2004} for an explicit formula). Finally, for every $u \in H^{1,2}([0,T]\times\Rn,\R)$ we denote
$$
L(u) \triangleq \int^T_0 \bigg( \bigg\| \frac{\partial \unknownDensity}{\partial t}(t,\cdot) - \frac{\partial u}{\partial t}(t,\cdot) \bigg\|^2_{L^2} + \| \unknownDensity(t,\cdot) - u(t,\cdot) \|^2_{H^2} \bigg) \; \mathrm{d}t .
$$
Our result on the approximation of $p$ via $\approxDensity$ is as follows:

\begin{theorem} \label{thm:approxAle}
There exists a constant $C > 0$, which depends on $n$, $m$, and $T$ uniquely, such that the following learning rate for the random model $\approxDensity$ holds with probability at least $1-\delta$, for every $M , N \in \mathbb{N}$ such that $M \ge 2 T$:
\begin{align*}
    L(\approxDensity) &\le C \big( \| \unknownDensity \|^2_{H^{m+1,2}} + \| \unknownDensity \|^2_{H^{1,2(m+1)}} \big) \times \\
    &\hspace{30ex}\times \left( M^{-2m} + R^{-4m} + R^n \firstRev{\log\left( \frac{4 M}{\delta} \right)} N^{-1} \right) .
\end{align*}
In particular, up to overloading $C$, for every $\varepsilon > 0$ small enough we have that
\begin{align*}
    L(\approxDensity) &\le C \big( \| \unknownDensity \|^2_{H^{m+1,2}} + \| \unknownDensity \|^2_{H^{1,2(m+1)}} \big) \left( \firstRev{\log\left( \frac{1}{\delta \varepsilon} \right)^{\frac{1}{2}}} \varepsilon \right)^2 \\
    &\firstRev{\le C \big( \| \unknownDensity \|^2_{H^{m+1,2}} + \| \unknownDensity \|^2_{H^{1,2(m+1)}} \big) \left( \log\left( \frac{M N}{\delta} \right)^{\frac{1}{2}} (M N)^{-\frac{2 m}{n + 2 ( 2 m + 1 )}} \right)^2} ,
\end{align*}
if \firstRev{$M = \varepsilon^{-1/m}/4$, $N = \varepsilon^{-(2 + n/(2m))}$ (as closest integers), and $R = \varepsilon^{-1/(2m)}$}.
\end{theorem}

\begin{proof}
For the sake of clarity, we divide the proof in several steps.

\vspace{5pt}

\noindent \textbf{1) Preliminaries.} With obvious notation, fix $r , s \ge 1$ and define the operator
$$
P_M: H^r([0,T],\R) \to H^r([0,T],\R) , \quad \textnormal{via} \quad (P_M u)(t) = \sum_{\ell=1}^M c_{\ell}(t) u(t_{\ell}) .
$$
For $u \in H^{r,s}([0,T]\times\Rn,\R)$, we denote by $u_x$ the function $u_x(\cdot) = u(\cdot,x)$, $x \in \Rn$, and by $\tilde{u}$ the function $\tilde{u}(t,x) = (P_M u_x)(t)$.
Now, note that the function $v_x := u_x - P_M u_x$ satisfies $v_x(t_\ell) = 0$ for any $\ell=1,\dots, M$. From the results we recalled in Section \ref{sec:RKHS}, we can bound the norm of functions with scattered zeros as follows: for every $0 \le \sigma \le r$ there exists a constant $C_{r,\sigma} > 0$ such that
\begin{equation} \label{eq:ProofAle1}
    \| u_x - P_M u_x \|_{H^{\sigma}} \leq C_{r,\sigma} M^{\sigma - r} \| u_x \|_{H^r},
\end{equation}
for almost every $x \in \Rn$. Below, we will often implicitly overload the constant $C_{r,\sigma}$. Moreover, we recall that, for $\alpha , \beta \ge 0$,
$$
D^\alpha_t D^\beta_x \tilde{u}(t,x) = D^\alpha_t D^\beta_x ( P_M u_x)(t) = D^\alpha_t P_M (D^\beta_x u)(t,x) ,
$$
so that by leveraging \eqref{eq:ProofAle1} with $\sigma = \alpha$, for $0 \leq \alpha \leq r$ and $0 \leq \nu \leq s$, yields
\begin{align} \label{eq:mixed-norm-bound-in-t}
\int_0^T \| &D^{\alpha}_t u(t,\cdot) - D^{\alpha}_t \tilde{u}(t,\cdot) \|^2_{H^\nu} \; \mathrm{d}t = \\
&= \sum_{|\beta| \leq \nu} \int_0^T \int_{\Rn} (D^{\alpha}_t D^{\beta}_x u(t,x) -  D^{\alpha}_t D^{\beta}_x \tilde{u}(t,x))^2 \; \mathrm{d}x \mathrm{d}t \nonumber \\
& = \sum_{|\beta| \leq \nu} \int_{\Rn} \int_0^T (D^{\alpha}_t D^{\beta}_x u(t,x) -  D^{\alpha}_t P_M D^{\beta}_x u(t,x))^2 \; \mathrm{d}x \mathrm{d}t \nonumber \\
& \leq \sum_{|\beta| \leq \nu} \int_{\Rn} \| D^{\beta}_x u(\cdot,x) - P_M D^{\beta}_x u(\cdot,x)\|^2_{H^{\alpha}} \; \mathrm{d}x \nonumber \\
& \leq \sum_{|\beta| \leq \nu} \int_{\Rn} C^2_{r,\alpha} M^{2 (\alpha - r)} \|D^{\beta}_x u(\cdot,x)\|^2_{H^r} \; \mathrm{d}x = C^2_{r,\alpha} M^{2 (\alpha - r)} \| u \|^2_{H^{r,\nu}} . \nonumber
\end{align}
At this step, by decomposing $\unknownDensity - \approxDensity = (\unknownDensity - \widetilde{p}) + (\widetilde{p} - \approxDensity)$ and applying the triangular inequality to $L(\approxDensity)$, we obtain that
\begin{align*}
    &L(\approxDensity)^{1/2} \le L(\widetilde{p})^{1/2} + A^{1/2} , \\
    &A \triangleq \int^T_0 \bigg( \bigg\| \frac{\partial \widetilde{p}}{\partial t}(t,\cdot) - \frac{\partial \approxDensity}{\partial t}(t,\cdot) \bigg\|^2_{L^2} + \| \widetilde{p}(t,\cdot) - \approxDensity(t,\cdot) \|^2_{H^2} \bigg) \; \mathrm{d}t .
\end{align*}
Thanks to \eqref{eq:mixed-norm-bound-in-t} and $\unknownDensity \in H^{m+1,2(m+1)}([0,T]\times\Rn,\R)$, we may bound $L(\widetilde{p})$ by
$$
L(\widetilde{p})^{1/2} \le C_{m+1,1} M^{-m} \| \unknownDensity \|_{H^{m+1,0}} + C_{m+1,0} M^{-(m+1)} \| \unknownDensity \|_{H^{m+1,2}} .
$$
The rest of the proof is devoted to appropriately bounding $A$.

\vspace{5pt}

\noindent \textbf{2) Further decomposing the term $\bm{A}$.} For this, from the definition of both $\widetilde{p}$ and $\approxDensity$, for every $0 \leq \alpha \leq m+1$ and $0 \leq \nu \leq 2(m+1)$ we have that
\begin{align*}
\bigg(\int_0^T \| D^{\alpha}_t \widetilde{p}(t,\cdot) - D^{\alpha}_t \approxDensity(t,\cdot) &\|^2_{H^\nu} \; \mathrm{d}t \bigg)^{1/2} = \\
&= \left( \int_0^T \left\| \sum_{\ell=1}^M D^{\alpha}_t c_{\ell}(t) ( \unknownDensity(t_{\ell},\cdot) - \widehat{g}_{\ell}(\cdot) ) \right\|^2_{H^\nu} \mathrm{d}t \right)^{1/2} \\
& \leq \sum_{\ell=1}^M \| c_{\ell} \|_{H^{\alpha}} \| \unknownDensity(t_{\ell},\cdot) - \widehat{g}_{\ell}(\cdot) \|_{H^\nu} .
\end{align*}
In, particular, by applying this bound to the two terms in $A$, we infer that
$$
A^{1/2} \leq \sum_{\ell=1}^M \Big( \| c_{\ell} \|_{H^1} \| \unknownDensity(t_{\ell},\cdot) - \widehat{g}_{\ell}(\cdot) \|_{L^2} + \| c_{\ell} \|_{L^2} \| \unknownDensity(t_{\ell},\cdot) - \widehat{g}_{\ell}(\cdot) \|_{H^2} \Big) .
$$

\vspace{5pt}

\noindent \textbf{3) Bounding each term $\bm{\| \unknownDensity(t_{\ell},\cdot) - \widehat{g}_{\ell}(\cdot) \|_{H^{\nu}}}$.} Fix $\ell=1,\dots,M$ and $0 \le \nu \le 2(m+1)$. As recalled in Section \ref{sec:RKHS}, the norm of any $u \in H^\nu(\Rn,\R)$, with $\nu \in \mathbb{N}$ writes
$$
\| u \|_{H^{\nu}} = \| J_{\nu}(\cdot) ({\cal F} u)(\cdot) \|_{L^2} ,
$$
where ${\cal F} u$ denotes the Fourier transform of $u$, and $J_{\nu}(z) \triangleq (1 + \| z \|^2)^{\nu/2} / (2 \pi)^{n/2}$, $z \in \Rn$. In particular, since the Fourier transform of $\rho_R$ is $\mathds{1}_{B^{\Rn}_R(0)}$, it holds that
\begin{equation} \label{eq:ProofAle2}
    {\cal F} \widehat{g}_{\ell}(z) = \sum_{j=1}^N \mathds{1}_{B^{\Rn}_R(0)}(z) e^{2\pi i z^\top X_{\ell,j}} = \mathds{1}_{B^{\Rn}_R(0)}(z) {\cal F} \widehat{g}_{\ell}(z) , \quad z \in \Rn ,
\end{equation}
so that, thanks to the fact that $(1 - \mathds{1}_{B^{\Rn}_R(0)}) \mathds{1}_{B^{\Rn}_R(0)} = 0$ and $\mathds{1}^2_{B^{\Rn}_R(0)} = \mathds{1}_{B^{\Rn}_R(0)}$, by denoting the function $\unknownDensity_{\ell}(\cdot) \triangleq \unknownDensity(t_{\ell},\cdot)$ we may compute
\begin{align*}
\| \unknownDensity(t_{\ell},\cdot) - \widehat{g}_{\ell}(\cdot) \|_{H^{\nu}} &= \|J_\nu ({\cal F} p_\ell - {\cal F} \widehat{g}_\ell)\|_{L^2} \\
& = \left\| J_\nu (1 - \mathds{1}_{B^{\Rn}_R(0)}){\cal F} \unknownDensity_{\ell} ~+~ J_{\nu} \mathds{1}_{B^{\Rn}_R(0)} ({\cal F} \unknownDensity_{\ell} - {\cal F} \widehat{g}_{\ell}) \right\|_{L^2} \\
& \leq \left\| J_{\nu} (1 - \mathds{1}_{B^{\Rn}_R(0)}) J_{2(m+1)}^{-1} \right\|_{L^\infty} \|J_{2(m+1)} {\cal F} \unknownDensity_{\ell} \|_{L^2} \\
&\hspace{23ex} + \left\| J_{\nu} \mathds{1}_{B^{\Rn}_R(0)} \right\|_{L^\infty} \left\| \mathds{1}_{B^{\Rn}_R(0)} ({\cal F} p_{\ell} - {\cal F} \widehat{g}_{\ell}) \right\|_{L^2} \\
& \leq R^{\nu - 2(m+1)} \| \unknownDensity_{\ell} \|_{H^{2(m+1)}} + 2^{\nu} R^{\nu} \left\| \mathds{1}_{B^{\Rn}_R(0)} ({\cal F} p_{\ell} - {\cal F} \widehat{g}_{\ell}) \right\|_{L^2} ,
\end{align*}
where we used the fact that $J_{2(m+1)}$, $J_{\nu}$, and $\mathds{1}_{B^{\Rn}_R(0)}$ are radial functions, and thus
\begin{align*}
    &\left\| J_{\nu} (1 - \mathds{1}_{B^{\Rn}_R(0)}) J_{2(m+1)}^{-1} \right\|_{L^\infty} = \sup_{r > R} (1+R^2)^{-(2(m+1)+\nu)/2} \le R^{\nu-2(m+1)} , \\
    &\left\| J_{\nu} \mathds{1}_{B^{\Rn}_R(0)} \right\|_{L^\infty} = \sup_{0< r \le R} (1+R^2)^{\nu/2} \le 2^{\nu} R^{\nu} , \quad \textnormal{as soon as} \quad R \ge 1 .
\end{align*}

\vspace{5pt}

\noindent \textbf{4) Bounding each term $\bm{\left\| \mathds{1}_{B^{\Rn}_R(0)} ({\cal F} \unknownDensity_{\ell} - {\cal F} \widehat{g}_{\ell}) \right\|_{L^2}}$.} For this, for every integers $\ell=1,\dots,M$ and $j=1,\dots,N$ we define a random mapping $\zeta_{\ell,j} : (\Rn)^{\mathbb{N}} \to L^2(\Rn,\mathbb{C})$ through the following expression, which holds for every $x^{\mathbb{N}} \in (\Rn)^{\mathbb{N}}$,
$$
\zeta_{\ell,j}(x^{\mathbb{N}})(z) \triangleq \mathds{1}_{B^{\Rn}_R(0)}(z) \; e^{2 \pi i z^\top X_{\ell,j}(x^{\mathbb{N}})} , \quad z \in \Rn .
$$
Note that, for every fixed $\ell=1,\dots,M$, since they depend deterministically from $X_{\ell,1},\dots,X_{\ell,M}$, the random mappings $\zeta_{\ell,1},\dots,\zeta_{\ell,N}$ are independent random mappings in the probability space $((\Rn)^{\mathbb{N}},\mathcal{B}^{\mathbb{N}},\mu^{\mathbb{N}}_{t_{\ell}})$, and they are additionally equally distributed with respect to $\unknownDensity_{\ell}(\cdot) \triangleq \unknownDensity(t_{\ell},\cdot)$. In particular, for every $j=1,\dots,N$ we may compute
$$
\mathbb{E}_{\mu^{\mathbb{N}}_{t_{\ell}}}[ \zeta_{\ell,j} ](z) = \int_{\Rn} \mathds{1}_{B^{\Rn}_R(0)}(z) \; e^{2 \pi i z^\top x} \unknownDensity(t_{\ell},x) \; \mathrm{d}x = \mathds{1}_{B^{\Rn}_R(0)}(z) ({\cal F} \unknownDensity_{\ell})(z) , \quad z \in \Rn .
$$ 
We may also compute
$$
\underset{x^{\mathbb{N}} \in (\Rn)^{\mathbb{N}}}{\textnormal{ess}\sup} \ \left\| \frac{\zeta_{\ell,j}(x^{\mathbb{N}}) - \mathbb{E}_{\mu^{\mathbb{N}}_{t_{\ell}}}[ \zeta_{\ell,j} ]}{N} \right\|^2_{L^2} \le \frac{4}{N^2} \int_{\Rn} \mathds{1}_{B^{\Rn}_R(0)}(z) \; \mathrm{d}z = \frac{4}{N^2} V_n R^n ,
$$
where $V_n$ is the volume of the $n$-dimensional unit ball. At this step, thanks to Pinelis inequality, for every $\eta > 0$ and $\ell=1,\dots,M$, we may compute
\begin{align*}
    &\mu_X\left( \left\| \mathds{1}_{B^{\Rn}_R(0)} ({\cal F} \unknownDensity_{\ell} - {\cal F} \widehat{g}_{\ell}) \right\|_{L^2} > \frac{3 (V_n R)^{n/2} \log\left( \frac{4}{\eta} \right)}{\sqrt{N}} \right) \le \\
    &\le \mu_X\left( \left\| \mathds{1}_{B^{\Rn}_R(0)} ({\cal F} \unknownDensity_{\ell} - {\cal F} \widehat{g}_{\ell}) \right\|_{L^2} > \frac{2 \sqrt{2} (V_n R)^{n/2} \sqrt{\log\left( \frac{2}{\eta} \right)}}{\sqrt{N}} \right) \\
    &= \mu^{\mathbb{N}}_{t_{\ell}}\left( \left\| \sum_{j=1}^N \frac{\zeta_{\ell,j} - \mathbb{E}_{\mu^{\mathbb{N}}_{t_{\ell}}}[ \zeta_{\ell,j} ]}{N} \right\|_{L^2} > \frac{2 \sqrt{2} (V_n R)^{n/2} \sqrt{\log\left( \frac{2}{\eta} \right)}}{\sqrt{N}} \right) \le \eta .
\end{align*}
Therefore, for every $\delta > 0$, by taking the union bound with $\eta = \delta / M$, with probability at least $1 - \delta$ it finally holds that
$$
\underset{\ell=1,\dots,M}{\max} \ \left\| \mathds{1}_{B^{\Rn}_R(0)} ({\cal F} \unknownDensity_{\ell} - {\cal F} \widehat{g}_{\ell}) \right\|_{L^2} \le \frac{3 (V_n R)^{n/2} \firstRev{\log\left( \frac{4 M}{\delta} \right)^{\frac{1}{2}}}}{\sqrt{N}} .
$$

\vspace{5pt}

\noindent \textbf{5) Bounding each term $\bm{\| c_{\ell} \|_{H^{\nu}}}$.} By construction, each $c_{\ell} \in H^{m+1}([0,T],\R)$ is the function with minimum norm that satisfies $c_{\ell_1}(t_{\ell_2}) = \delta_{\ell_1,\ell_2}$, for $\ell_1 , \ell_2 = 1,\dots,M$. In particular, since the function $z_\ell(t) \triangleq \sinc(M(t-t_\ell)/T)$ is analytic, and thus  $z_{\ell} \in H^{m+1}([0,T],\R)$, and satisfies $z_{\ell_1}(t_{\ell_2}) = \delta_{\ell_1,\ell_2}$, for $\ell_1 , \ell_2 = 1,\dots,M$, it must hold that $\| c_{\ell} \|_{H^{m+1}} \leq \| z_{\ell} \|_{H^{m+1}}$, for every $\ell=1,\dots,M$. Therefore, since
$$
\underset{\ell_1=1,\dots,M}{\max} \ | c_{\ell_2}(t_{\ell_1}) | = 1 , \quad \ell_2=1,\dots,M ,
$$
by applying the bound for Sobolev functions with scattered zeros recalled in \ref{sec:RKHS} with $h = T/M$, we obtain that, for every $0 \le \nu \le m+1$ there exists a constant $C_{m+1,\nu} > 0$ such that
$$
\| c_{\ell} \|_{H^{\nu}} \leq C_{m+1,\nu} \Big( ( T / M )^{-\nu} + ( T / M )^{m+1-\nu} \| z_{\ell} \|_{H^{m+1}} \Big) , \quad \ell=1,\dots,M .
$$
To conclude, we note that the Fourier transform of the extension of each $z_{\ell}$ to $H^{m+1}(\R,\R)$ is the function $(T / M) \mathds{1}_{\{ M / T \}}$. Combining this latter result with the Fourier characterization of the norm $\| \cdot \|_{H^{m+1}}$ yields, for every $\ell=1,\dots,M$,
$$
\displaystyle \| z_{\ell} \|^2_{H^{m+1}} \leq \frac{T}{M} \int_{-\frac{M}{2 T}}^{\frac{M}{2 T}} (1+ \tau^2)^{m+1} \; \mathrm{d}\tau \leq \left( \frac{M}{T} \right)^{2(m+1)},
$$
where the last step follows from the assumption $M \ge 2 T$\footnote{Indeed, under this assumption, $1 \le 3/4 M^2 / T^2$, and thus $\underset{0 \le \tau \le M / (2 T)}{\max} ( 1 + \tau^2) \le M^2/T^2$.}. Therefore, for every $\ell=1,\dots,M$ we finally obtain that (we implicitly overload the constant $C_{m+1,\nu}$)
$$
\| c_{\ell} \|_{H^{\nu}} \leq C_{m+1,\nu} \left( \frac{M}{T} \right)^{\nu} , \quad \nu=0,\dots,m+1 ,
$$

\vspace{5pt}

\noindent \textbf{6) Gathering the previous bounds together and end of the proof.} Recalling from the results in Section \ref{sec:RKHS} that there exists a constant $C > 0$ such that
$$
\| u \|_{L^{\infty}} \le C \| u \|_{H^1} , \quad \textnormal{for every} \quad u \in H^{m+1}([0,T],\R) ,
$$
we may compute, for every $\ell=1,\dots,M$,
$$
\| \unknownDensity(t_{\ell},\cdot) \|_{H^{2(m+1)}(\Rn,\R)} \le \| \unknownDensity \|_{L^\infty([0,T],\R) \otimes H^{2(m+1)}(\Rn,\R)} \le C \| \unknownDensity \|_{H^{1,2(m+1)}} .
$$
Combining this latter inequality with all the previous bounds finally yields
\begin{align*}
    L(\approxDensity)^{1/2} &\le L(\widetilde{p})^{1/2} + A^{1/2} \\
    &\le C_{m+1,1} M^{-m} \| \unknownDensity \|_{H^{m+1,0}} + C_{m+1,0} M^{-(m+1)} \| \unknownDensity \|_{H^{m+1,2}} + A^{1/2} \\
    &\le C \big( \| \unknownDensity \|_{H^{m+1,2}} + \| \unknownDensity \|_{H^{1,2(m+1)}} \big) \times \\
    &\hspace{29ex}\times \left( M^{-m} + R^{-2m} + R^{n/2} \firstRev{\log\left( \frac{4 M}{\delta} \right)^{\frac{1}{2}}} N^{-1/2} \right) ,
\end{align*}
for some appropriate constant $C > 0$, which holds with probability at least $1 - \delta$, and the sought conclusion may be easily inferred.
\end{proof}

The main benefit offered by replacing the unknown density $\unknownDensity$ with the model density $\approxDensity$ consists of the fact that integrals of this latter mapping and of its derivatives can be actually (easily) computed, enabling to correctly instantiate our first learning problem for $(\unknownDiffusion,\unknownDrift) \in \C$, which is our next step.

\subsection{The infinite-dimensional learning problem and its fidelity} \label{sec:LP}

In this section, our goal consists of instantiating and analyzing the learning problem to identify the coefficients $(\unknownDiffusion,\unknownDrift) \in \C$. This problem computes coefficients which ``best'' matches SFPE when evaluated at the model density $\approxDensity$. Below, we seek those coefficients in the whole set $\C$, deferring to a later section the problem of learning $(\unknownDiffusion,\unknownDrift) \in \C$ through finite-dimensional models. In particular, the main result contained herein consists of appropriate learning error estimates which pave the way to obtaining learning error estimates for the problem of learning $(\unknownDiffusion,\unknownDrift) \in \C$ through finite-dimensional models. However, for pedagogical purposes, we present complete learning rates for the infinite-dimensional learning problem as well.

From now on, we fix $0 < \varepsilon , \delta < 1$, and select $M , N \in \mathbb{N}$ and $R > 0$ as claimed in Theorem \ref{thm:approxAle}, so that, with probability at least $1 - \delta$, it holds that:
\firstRev{\begin{align} \label{eq:ErrorDensity}
    L(\approxDensity) &\le C(\unknownDiffusion,\unknownDrift) \left( \firstRev{\log\left( \frac{1}{\delta \varepsilon} \right)^{\frac{1}{2}}} \varepsilon \right)^2 \\
    &\le C(\unknownDiffusion,\unknownDrift) \left( \log\left( \frac{M N}{\delta} \right)^{\frac{1}{2}} (M N)^{-\frac{2 m}{n + 2 ( 2 m + 1 )}} \right)^2 , \nonumber
\end{align}
}
where the constant $C(\unknownDiffusion,\unknownDrift) > 0$ depends on $\unknownDiffusion$ and $\unknownDrift$ uniquely. To compute accurate learning rates, we leverage classical RKHS approximation theory, which is essentially well-posed for mappings which are defined on bounded domains. For this, we will make use of the following additional assumption: 

\begin{Itemize}
    \item[$(B)$] There exists $R_* > 0$ such that the coefficients $(\unknownDiffusion,\unknownDrift) \in \C$ satisfy:
    $$
    \textnormal{supp}\big( \secRev{\unknownDiffusion(t,\cdot) - \alpha I} , \unknownDrift(t,\cdot) \big) \subseteq \overline{B^{\Rn}_{R_*}(0)} , \quad \textnormal{for every} \ t \in [0,T] .
    $$
\end{Itemize}

\begin{remark}
Assumption $(B)$ plays a key role in computing learning rates which, on the one hand, leverage classical RKHS approximation theory, and which, on the other hand, depend explicitly on the parameters defining the learning problem. In particular, the latter property can not be generally obtained by just (smoothly) restricting to balls quantities which are defined on unbounded domains. That said, our result can be shown to hold under less restricting assumptions, such as assuming the coefficients $(\unknownDiffusion,\unknownDrift) \in \C$ decay to zero \secRev{(up to removing $(\alpha I,0)$ of course)} at infinity under specific rates \cite{Wendland2004}, although we opted for Assumption $(B)$ to avoid excessively tedious computations, in turn fostering a smooth exposition. In addition, it is worth mentioning Assumption $(B)$ is often naturally verified in many applications ranging from biology to robotics, where the state space is a bounded domain.
\end{remark}

Thanks to Assumption $(B)$, we may restrict ourselves to the 
subset:
\begin{align*}
    \CR \triangleq \Big\{ (a, \, &b) \in \C : \\
    &\textnormal{supp}\big( \secRev{a(t,\cdot) - \alpha I} , b(t,\cdot) \big) \subseteq \overline{B^{\Rn}_{R_*}(0)} , \ \textnormal{for every} \ t \in [0,T] \Big\} \subseteq \C ,
\end{align*}
We are now ready to define our infinite-dimensional problem for learning the coefficients $(\unknownDiffusion,\unknownDrift) \in \CR$, which writes as follows:

\begin{definition} \label{def:LP}
For every positive real $\lambda > 0$, the (random) infinite-dimensional Learning Problem to learn stochastic differential equations is defined as:
$$
\LearningProblem \quad \underset{(a,b) \in \CR}{\min} L_{\lambda}(a,b) \triangleq \int^T_0 \left\| \frac{\partial \approxDensity}{\partial t}(t,\cdot) - (\Kolmo_t)^* \approxDensity(t,\cdot) \right\|^2_{L^2} \; \mathrm{d}t + \lambda \| (\secRev{a - \alpha I},b) \|^2_{\Cbasis} \ .
$$
\end{definition}

The well-posedness of problem \LearningProblem is proven in the following proposition:

\begin{proposition} \label{prop:wellPosLP}
For every $\lambda > 0$, problem \LearningProblem is well-posed and has a unique solution, which is denoted by $(\solDiffusion,\solDrift) \in \CR$.
\end{proposition}

\begin{proof}
Since the mapping
\begin{equation} \label{eq:stronglyConvexMapping}
    (a,b) \in \Cbasis \mapsto \int^T_0 \left\| \frac{\partial \approxDensity}{\partial t}(t,\cdot) - (\secRev{\mathcal{L}^{a + \alpha I , b}_t})^* \approxDensity(t,\cdot) \right\|^2_{L^2} \; \mathrm{d}t + \lambda \| (a,b) \|^2_{\Cbasis}
\end{equation}
is strictly convex on the closed convex subset
\begin{align*}
    \secRev{\mathcal{S} \triangleq \Big\{ (a,b) \in \Cweak :} \ &\secRev{a(t,x) \succcurlyeq 0 ,} \\
    &\secRev{\textnormal{supp}\big( \secRev{a(t,\cdot)} , b(t,\cdot) \big) \subseteq \overline{B^{\Rn}_{R_*}(0)} , \quad (t,x) \in [0,T] \times \Rn \Big\} \subseteq \Cbasis},
\end{align*}
we just need to prove the existence of a solution to \LearningProblem. For this, if \secRev{$(a_k + \alpha I,b_k)_{k \in \mathbb{N}}$} is any minimizing sequence for \LearningProblem, there must exist some constant $C > 0$ such that $\| (a_k,b_k) \|_{\Cbasis} \le C$ for every $k \in \mathbb{N}$, and therefore, since $\secRev{\mathcal{S}}$ is in particular a closed and convex subset of the Hilbert space $\Cbasis$, up to extracting a subsequence there exists $(\solDiffusion,\solDrift) \in \secRev{\mathcal{S}}$ such that $(a_k,b_k)_{k \in \mathbb{N}}$ converges to $(\solDiffusion,\solDrift)$ for the weak topology of $\Cbasis$. Now, thanks to Theorem \ref{Theo:CoeffProperties} 
the mapping \eqref{eq:stronglyConvexMapping} is 
continuous for the strong topology of $\Cbasis$. Thus, we infer that 
it is in particular weakly lower semi-continuous, 
so that
\secRev{\begin{align*}
    &\int^T_0 \left\| \frac{\partial \approxDensity}{\partial t}(t,\cdot) - (\mathcal{L}^{\solDiffusion + \alpha I,\solDrift}_t)^* \approxDensity(t,\cdot) \right\|^2_{L^2} \; \mathrm{d}t + \lambda \| (\solDiffusion,\solDrift) \|^2_{\Cbasis} \le \\
    &\le \underset{k\to\infty}{\liminf} \ \int^T_0 \left\| \frac{\partial \approxDensity}{\partial t}(t,\cdot) - (\mathcal{L}^{a_k + \alpha I , b_k}_t)^* \approxDensity(t,\cdot) \right\|^2_{L^2} \; \mathrm{d}t + \lambda \| (a_k,b_k) \|^2_{\Cbasis} \\
    &= \underset{(a,b) \in \CR}{\min} \quad \int^T_0 \left\| \frac{\partial \approxDensity}{\partial t}(t,\cdot) - (\Kolmo_t)^* \approxDensity(t,\cdot) \right\|^2_{L^2} \; \mathrm{d}t + \lambda \| (a - \alpha I,b) \|^2_{\Cbasis} ,
\end{align*}}
and the sought conclusion follows.
\end{proof}

We now investigate the fidelity of the learning problem \LearningProblem. More specifically, through the estimates of Section \ref{sec:ShortFokkerPlanck}, under appropriate norms we compute bounds for the approximation error between solutions to SFPE, associated with the drift and diffusion coefficients solutions to \LearningProblem, and the unknown mapping $\unknownDensity$. In turn, these bounds will endow our learning procedure with well-posedness and high fidelity.

We start with the following technical lemma:

\begin{lemma} \label{lemma:estimateSource}
For every $\lambda > 0$, with probability at least $1 - \delta$, it holds that:
\firstRev{\begin{align*}
    L_{\lambda}(\solDiffusion,\solDrift) &\le C(\unknownDiffusion,\unknownDrift) \left( \lambda + \left( \log\left( \frac{1}{\delta \varepsilon} \right)^{\frac{1}{2}} \varepsilon \right)^2 \right) \\
    &\le C(\unknownDiffusion,\unknownDrift) \left( \lambda + \left( \log\left( \frac{M N}{\delta} \right)^{\frac{1}{2}} (M N)^{-\frac{2 m}{n + 2 ( 2 m + 1 )}} \right)^2 \right) ,
\end{align*}
}
where the constant $C(\unknownDiffusion,\unknownDrift) > 0$ depends on $\unknownDiffusion$ and $\unknownDrift$ uniquely \secRev{(in particular, $C(\unknownDiffusion,\unknownDrift)$ implicitly depends on the hyper-parameters $m$, $p_0$, and $\alpha$)}.
\end{lemma}

\begin{proof}
Thanks to Theorem \ref{Theo:CoeffProperties} and Corollary \ref{corol:unknownDensity}, a routine use of H\"older and Young inequalities allows us to compute
\secRev{\begin{align*}
    L_{\lambda}(\solDiffusion,\solDrift) - &\lambda \| (\unknownDiffusion - \alpha I,\unknownDrift) \|^2_{\Cbasis} \le \\
    &\le \int^T_0 \left\| \frac{\partial \approxDensity}{\partial t}(t,\cdot) - \unknownKolmo^*_t \approxDensity(t,\cdot) \right\|^2_{L^2} \; \mathrm{d}t \\
    &\le 2 \int^T_0 \left\| \frac{\partial \approxDensity}{\partial t}(t,\cdot) - \frac{\partial \unknownDensity}{\partial t}(t,\cdot) \right\|^2_{L^2} \; \mathrm{d}t + 2 \int^T_0 \left\| \unknownKolmo^*_t (\approxDensity - \unknownDensity)(t,\cdot) \right\|^2_{L^2} \; \mathrm{d}t \\
    &\le 2 \int^T_0 \left\| \frac{\partial \approxDensity}{\partial t}(t,\cdot) - \frac{\partial \unknownDensity}{\partial t}(t,\cdot) \right\|^2_{L^2} \; \mathrm{d}t + 2 C(\unknownDiffusion,\unknownDrift) \int^T_0 \| \approxDensity(t,\cdot) - \unknownDensity(t,\cdot) \|^2_{H^2} \; \mathrm{d}t \\
    &\le C(\unknownDiffusion,\unknownDrift) L(\approxDensity) \le C(\unknownDiffusion,\unknownDrift) \left( \log\left( \frac{1}{\delta \varepsilon} \right)^{\frac{1}{2}} \varepsilon \right)^2 ,
\end{align*}}
where the (overloaded) constant $C(\unknownDiffusion,\unknownDrift) > 0$, which depends on $\unknownDiffusion$ and $\unknownDrift$ uniquely, comes from the constant in \eqref{eq:ErrorDensity}, and the conclusion follows.
\end{proof}

We are now ready to compute error bounds between solutions to SFPE, associated with the drift and diffusion coefficients solutions to \LearningProblem, and the unknown mappings $\unknownDensity$. It is important to note that, given the nature of our data set, which essentially depends on observations of the law of the state process, the error between the unknown densities and the densities stemming from the learned coefficients is a ``good metric'' with which the convergence of an identification algorithm for stochastic differential equations which leverage observations of the state process may be tested.

We better formalize this metric as follows. For every $(a,b) \in \C$, let $p_{a,b} \in H^{m+1,2(m+1)}([0,T]\times\Rn,\R)$ denote the unique solution to SFPE with coefficients $(a,b) \in \C$. Note that the existence and uniqueness of the regular mapping $p_{a,b} \in H^{m+1,2(m+1)}([0,T]\times\Rn,\R)$ as non-negative solution to SFPE, of unitary mass and with coefficients $(a,b) \in \C$ is immediate consequence of Theorem \ref{theor:ShortexFPE}. We define the following metric to test the accuracy of our method:
$$
\Metric(a,b) \triangleq \| p_{a,b} - \unknownDensity \|^2_{L^2} = \int^T_0 \| p_{a,b}(t,\cdot) - \unknownDensity(t,\cdot) \|^2_{L^2} \; \mathrm{d}t , \quad (a,b) \in \C .
$$
Among other benefits, we will see this metric is also particularly well-suited to estimate the error which is done when computing the observation/regulation metric \eqref{eq:ShortintegralMetric}.

Our main result on the accuracy of our learning method writes as follows:

\begin{theorem} \label{theorem:ultimateEstimate1}
Let the coefficients $(\solDiffusion,\solDrift) \in \CR$ be the unique solution to \LearningProblem with 
$\firstRev{\lambda = \left( \log\left( \frac{1}{\delta \varepsilon} \right)^{\frac{1}{2}} \varepsilon \right)^2}$. With probability at least $1 - \delta$, it holds that:
\firstRev{\begin{align*}
    E(\solDiffusion,\solDrift) &\le C(\unknownDiffusion,\unknownDrift) \left( \log\left( \frac{1}{\delta \varepsilon} \right)^{\frac{1}{2}} \varepsilon \right)^2 \\
    &\le C(\unknownDiffusion,\unknownDrift) \left( \log\left( \frac{M N}{\delta} \right)^{\frac{1}{2}} (M N)^{-\frac{2 m}{n + 2 ( 2 m + 1 )}} \right)^2 ,
\end{align*}
}
where the constant $C(\unknownDiffusion,\unknownDrift) > 0$ depends on $\unknownDiffusion$ and $\unknownDrift$ uniquely \secRev{(in particular, $C(\unknownDiffusion,\unknownDrift)$ implicitly depends on the hyper-parameters $m$, $p_0$, and $\alpha$)}.
\end{theorem}

\begin{proof}
We define
$$
\rho \triangleq \solDensity - \approxDensity \in C(0,T;L^2(\Rn,\R)) \cap H^{0,1}([0,T]\times\Rn,\R) ,
$$
and
$$
f(t,\cdot) \triangleq -\left( \frac{\partial \approxDensity}{\partial t}(t,\cdot) - (\mathcal{L}^{\solDiffusion,\solDrift}_t)^* \approxDensity(t,\cdot) \right) \in L^2([0,T]\times\Rn,\R) .
$$
It is readily seen that
$$
\begin{cases}
\displaystyle \frac{\mathrm{d}}{\mathrm{d}t} \int_{\Rn} \varphi(y) \rho(t,y) \; \mathrm{d}y = \\
\displaystyle \hspace{10ex}= \int_{\Rn} \left( \mathcal{L}^{\solDiffusion,\solDrift}_t \varphi(y) \rho(t,y) + f(t,y) \varphi(y) \right) \; \mathrm{d}y , \quad \varphi \in C^{\infty}_c(\Rn,\R) , \\[10pt]
\rho(0,\cdot) = 0 ,
\end{cases}
$$
and therefore, thanks to Theorem \ref{theor:ShortexFPE}, we may apply the estimate \eqref{eq:Shortestimate1} to $\rho$, which in combination with Lemma \ref{lemma:estimateSource} with the choice 
$\firstRev{\lambda = \left( \log\left( \frac{1}{\delta \varepsilon} \right)^{\frac{1}{2}} \varepsilon \right)^2}$ 
yields
\secRev{$$
\underset{t \in [0,T]}{\sup} \ \| \solDensity(t,\cdot) - \approxDensity(t,\cdot) \|^2_{L^2} \le C\big(\alpha,\|(\solDiffusion - \alpha I,\solDrift)\|_{\Cbasis}\big) C(\unknownDiffusion,\unknownDrift) 
\left( \log\left( \frac{1}{\delta \varepsilon} \right)^{\frac{1}{2}} \varepsilon \right)^2 ,
$$}
where the constant $C(\unknownDiffusion,\unknownDrift) > 0$ depends on $\unknownDiffusion$ and $\unknownDrift$ uniquely, whereas the constant $\secRev{C\big(\alpha,\|(\solDiffusion - \alpha I,\solDrift)\|_{\Cbasis}\big)}$ continuously depends on $\secRev{\alpha}$ and $\secRev{\| (\solDiffusion - \alpha I,\solDrift) \|_{\Cbasis}}$ uniquely. Up to overloading these constants, combined with \eqref{eq:ErrorDensity} this latter inequality readily yields
\secRev{\begin{equation} \label{eq:almostFinalBound}
    \int^T_0 \| \solDensity(t,\cdot) - \unknownDensity(t,\cdot) \|^2_{L^2} \; \mathrm{d}t \le C\big(\alpha,\|(\solDiffusion - \alpha I,\solDrift)\|_{\Cbasis}\big) C(\unknownDiffusion,\unknownDrift) 
    \left( \log\left( \frac{1}{\delta \varepsilon} \right)^{\frac{1}{2}} \varepsilon \right)^2 .
\end{equation}}

At this step, from Lemma \ref{lemma:estimateSource} in particular we obtain that
\secRev{\begin{equation} \label{eq:boundSolCoeff}
\| (\solDiffusion - \alpha I,\solDrift) \|^2_{\Cbasis} \le \frac{C(\unknownDiffusion,\unknownDrift)}{\lambda} 
\left( \lambda + \left( \log\left( \frac{1}{\delta \varepsilon} \right)^{\frac{1}{2}} \varepsilon \right)^2 \right) = 2 C(\unknownDiffusion,\unknownDrift) ,
\end{equation}}
as soon as 
$\firstRev{\lambda = \left( \log\left( \frac{1}{\delta \varepsilon} \right)^{\frac{1}{2}} \varepsilon \right)^2}$. Therefore, up to overloading the constant $C(\unknownDiffusion,\unknownDrift)$, the conclusion follows from combining \eqref{eq:almostFinalBound} with \eqref{eq:boundSolCoeff}.
\end{proof}

\subsection{The finite-dimensional learning problem and its fidelity} \label{sec:Finiteapprox}

The learning problem we introduced in Section \ref{sec:LP} (see Definition \ref{def:LP}) remains difficult to numerically solve. Here, we discuss appropriate finite-dimensional approximations of \LearningProblem and error bounds ensuring the fidelity of this latter approximation, ultimately making our learning approach for stochastic differential equations accurate and tractable.

We start by recalling and adapting the approximation tools we introduced in Section \ref{sec:RKHS} to our framework. 
Let $\secRev{D = [0,T] \times \overline{B^{\Rn}_{R_* + 1}(0)}}$, and consider the RKHS $\mathcal{H}_D = H^{d(m)}(D,\R)$ with associated kernel $K_D$. Without loss of generality, \secRev{we may assume $K_D$ to equal 1 on $[0,T] \times \partial B^{\Rn}_{R_*+1}(0)$ and such that all its derivatives equal zero on $[0,T] \times \partial B^{\Rn}_{R_*+1}(0)$. While not limiting, pursuing learning with slightly larger domains, i.e., $B^{\Rn}_{R_*+1}(0)$ instead of $B^{\Rn}_{R_*}(0)$, is key to enable rates of convergence under positiveness constraints, as we will see shortly.} For any set of $Q \in \mathbb{N}$ points
$$
\widetilde{X}_D \triangleq \Big\{ (t_1,x_1) , \dots , (t_Q, x_Q) \Big\} ,
$$
we consider the following coordinate-wise finite dimensional models to approximate the candidate drift and diffusion coefficient solutions $(\solDiffusion,\solDrift)$:
\begin{align*}
    \begin{split}
        \secRev{a_A}(t,y) &\triangleq \sum_{\ell, \ell'=1}^Q A_{\ell, \ell'} \, K_D((t,y),(t_{\ell},x_{\ell}))\, K_D((t,y),(t_{\ell'},x_{\ell'})) , \quad (t,y) \in D , \\
        \secRev{b_B}(t,x) &\triangleq \sum_{\ell=1}^Q B_{\ell} K_D((t,y),(t_{\ell},x_{\ell})),  \quad (t,y) \in D ,
    \end{split}
\end{align*}
with $A_{\ell,\ell'} \in \R^{n\times n}, B_{\ell} \in \R^n$. Note in particular that, by defining $\Phi_Q: D \to \R^{n \times Qn}$ to be the map $\Phi_Q(t,y) = \big( K_D((t,x),(t_1,x_1)) I_{n\times n} | \dots | K_D((t,y),(t_Q,x_Q)) I_{n\times n} \big)$, then 
\begin{align}\label{eq:ab-with-phi}
    \secRev{a_A}(t,y) = \Phi_Q(t,y) A \Phi_Q(t,y)^{\top}, \quad \secRev{b_B}(t,y) = \Phi_Q(t,y) B,
\end{align}
with $B \triangleq (B_1,\dots,B_Q)^{\top} \in \R^{Qn}$, and $A \in \R^{Qn \times Qn}$ is the $Q \times Q$ block matrix with block entries $A_{\ell,\ell'}$. Therefore, 
we define the finite dimensional convex subset in which $(\solDiffusion,\solDrift)$ is approximated by 
\begin{align*}
    \secRev{\CRh \triangleq 
    \Bigg\{ \bigg( \chi \left( a_A + \frac{\alpha}{2} I \right) + (1 - \chi) \alpha I , \chi b_B \bigg) :} \ &\secRev{A \in \R^{Qn \times Qn} ,} \\
    &\secRev{A \succeq 0 , \ \textnormal{and} \ B \in \R^{Qn} \Bigg\} \subseteq \CRhalf ,}
\end{align*}
\secRev{where $\chi \in C^{\infty}(\R^n,[0,1])$ is any given cut-off function satisfying $\chi|_{\overline{B^{\Rn}_{R_*}(0)}} = 1$ and $\textnormal{supp} \; \chi \subseteq \overline{B^{\Rn}_{R_*+1}(0)}$. Through the subset $\CRh$, we enlarge the space in which we look for finite-dimensional approximating coefficients, due to $a_A + \alpha / 2$ with $A \succeq 0$. Crucial to enable rates of convergence under positiveness constraint as we will see shortly, such strategy is not 
limiting, as it poorly affects computational complexity.}

Before moving to the core of this section, by leveraging the facts we recalled in Section \ref{sec:RKHS}, we provide a crucial approximation results for coefficients $(a,b) \in \CR$. Specifically, by combining the bounds for Sobolev functions with scattered zeros we listed in Section \ref{sec:RKHS} with Theorem \ref{Theo:CoeffProperties}, we obtain the following:

\vspace{-12.5pt}

\firstRev{\begin{theorem}
\label{theo:appOperatorCoeff}
Denote the fill distance between $\widetilde{X}_D$ and $D$ with
$$
h_Q \triangleq \sup_{(t,y) \in D} \ \min_{\ell = 1,\dots,Q} \|(t_{\ell}, x_{\ell}) - (t,y) \| .
$$
\secRev{There exists a constant $C > 0$ such that, for every $(a,b) \in \CR$, 
there exists a tuple $(\proj(a),\proj(b)) \in \CRh$ \secRev{satisfying:}
\begin{align*}
    &\| (a,b) - (\proj(a),\proj(b)) \|_{W^{2,\infty}} \le C (\|\sqrt{a}\|^2_{H^{d(m)}(D)} + \|b\|_{H^{d(m)}})  
    h^{2 m - 1}_Q , \\[5pt]
    &\|(\proj(a) - \alpha I,\proj(b)) \|_{\Cbasis} \le C (\|\sqrt{a}\|^2_{H^{d(m)}(D)} + \|b\|_{H^{d(m)}}).
\end{align*}}
\end{theorem}}

\begin{proof}
\secRev{Recall that, by definition, $a(t,\cdot) = \alpha \, I$ and $b(t,\cdot) = 0$ outside of $\overline{B^{\Rn}_{R_*}(0)}$. First, thanks to the bounds for Sobolev functions with scattered zeros \cref{eq:wendland}, 
we may readily claim the existence of a constant $C > 0$ (implicitly overloaded below) and 
of a map $b_B$ in the form \cref{eq:ab-with-phi} such that
\begin{equation*} 
    \| b - b_B \|_{H^{d(m) - 2 m + 1}(D)} \le C \| b \|_{H^{d(m)}} 
    h^{2 m - 1}_Q .
\end{equation*}
Therefore, on the one hand we may compute
\begin{align*}
    &\| b_B \|_{H^{d(m) - 2 m + 1}\big( [0,T] \times \overline{B^{\Rn}_{R_*+1}(0)} \setminus \overline{B^{\Rn}_{R_*}(0)} \big)} = \\
    &= \| b - b_B \|_{H^{d(m) - 2 m + 1}\big( [0,T] \times \overline{B^{\Rn}_{R_*+1}(0)} \setminus \overline{B^{\Rn}_{R_*}(0)} \big)} \le C \| b \|_{H^{d(m)}} h^{2 m - 1}_Q ,
\end{align*}
and on the other, combining the two previous inequalities yields
\begin{align*}
    &\| b - \chi b_B \|_{H^{d(m) - 2 m + 1}} = \| b - \chi b_B \|_{H^{d(m) - 2 m + 1}(D)} \\
    &\le \| b - b_B \|_{H^{d(m) - 2 m + 1}\big( [0,T] \times \overline{B^{\Rn}_{R_*}(0)} \big)} + \| \chi b_B \|_{H^{d(m) - 2 m + 1}\big( [0,T] \times \overline{B^{\Rn}_{R_*+1}(0)} \setminus \overline{B^{\Rn}_{R_*}(0)} \big)} \\
    &\le C \| b \|_{H^{d(m)}} h^{2 m - 1}_Q .
\end{align*}
In particular, leveraging Theorem \ref{Theo:CoeffProperties} we may thus select $\proj(b) \triangleq \chi b_B$.}


\secRev{We now turn to the construction of 
the approximation $\proj(a)$. This construction is slightly more elaborated and is taken from \cite{rudi2024finding,muzellec2021learning,Rudi2021,marteau2020non}, in particular, adapting the arguments in \cite{muzellec2021learning}. Note that the matrix square root $\sqrt{\cdot}$ has uniformly bounded derivatives of any order on the set $\{ A \in \R^{n\times n} : \ A \succeq \alpha / 2 \, I \}$. By assumption, it in particular follows that $\sqrt{a - \alpha / 2 \, I} \in H^{d(m)}(D,\R^{n\times n})$ (we are implicitly restricting the map $a$ to the domain $D$), and therefore that 
$(\sqrt{a})_{ij} \triangleq e_i \ \sqrt{a - \alpha / 2 \, I} \ e_j \in H^{d(m)}(D,\R)$. Now, we can build our approximation $a_A$ as follows:
$$
a_A(t,x) \triangleq v(t,y)^{\top} v(t,y) , \quad (t,y) \in D ,
$$
where $v: D \to \R^{n\times n}$ is obtained by element-wise applying the bound for Sobolev functions with scattered zeros \cref{eq:wendland} 
to $\sqrt{a - \alpha / 2 \, I}$, which yields
\begin{align} \label{eq:boundWInftyProof-forsqrta}
    &\left\| \left( \sqrt{a - \frac{\alpha}{2} I} \right)_{ij} - v_{ij} \right\|_{H^{d(m) - 2 m + 1}(D)} \le \\
    &\le C \left\| \left( \sqrt{a - \frac{\alpha}{2} I} \right)_{ij} \right\|_{H^{d(m)}(D)} 
    h^{2 m - 1}_Q \le C \| \sqrt{a} \|_{H^{d(m)}(D)} 
    h^{2 m - 1}_Q , \quad i , j \in \{ 1 , \dots , n \} . \nonumber
\end{align}
In particular, since
$$
v(t,y) = \sum^Q_{\ell=1} R_\ell \ K_D((t,y),(t_\ell,x_\ell)) , \quad (t,y) \in D ,
$$
for some $R_1,\dots, R_Q \in \R^{n\times n}$, by denoting $R = (R_1, \dots, R_Q) \in \R^{n\times Qn}$, we have that
\begin{equation} \label{eq:FuncMatrixProof}
    a_A(t,y) = \Phi_Q(t,y) A \Phi_Q(t,y)^{\top} , \quad (t,y) \in D ,
\end{equation}
with $A \triangleq R R^{\top} \in \R^{Qn \times Qn}$ and $A \succeq 0$. At this step, thanks to bound \eqref{eq:boundWInftyProof-forsqrta}, the identity $\sqrt{M}^{\top} \sqrt{M} - v^{\top} v = \sqrt{M}^{\top} (\sqrt{M} - v) + (\sqrt{M} - v)^{\top} \sqrt{M} - (\sqrt{M} - v)^{\top} (\sqrt{M} - v)$, for $M \in \Rnn$, and the fact that $H^{d(m) - 2 m + 1}(D)$ is a Banach algebra, we first infer that
\begin{align*}
    \bigg\| \left( a -\frac{\alpha}{2} I \right) - \ &a_A \bigg\|_{H^{d(m) - 2 m + 1}(D)} \leq \left\| \sqrt{a -\frac{\alpha}{2} I} - v \right\|_{H^{d(m) - 2 m + 1}(D)}^2 \\
    &+ 2 \left\| \sqrt{a -\frac{\alpha}{2} I} - v \right\|_{H^{d(m) - 2 m + 1}(D)} \left\| \sqrt{a -\frac{\alpha}{2} I} \right\|_{H^{d(m) - 2 m + 1}(D)} .
\end{align*}
Then, combining this latter inequality with
$$
\left\| \sqrt{a -\frac{\alpha}{2} I} - v \right\|^2_{H^{d(m) - 2 m + 1}(D)} \leq \sum_{i,j=1}^n \left\| \left( \sqrt{a -\frac{\alpha}{2} I} \right)_{ij} - v_{ij} \right\|^2_{H^{d(m) - 2 m + 1}(D)} ,
$$
identity \eqref{eq:boundWInftyProof-forsqrta}, and Theorem \ref{Theo:CoeffProperties} (readily extended to bounded domains) yields
\begin{align*} 
    &\left\| \left( a -\frac{\alpha}{2} I \right) - a_A \right\|_{W^{2,\infty}(D)} \le \\
    &\le C \left\| \left( a -\frac{\alpha}{2} I \right) - a_A \right\|_{H^{d(m) - 2 m + 1}(D)} \leq  C 
    \| \sqrt{a} \|_{H^{d(m)}(D)}^2 
    h^{2 m - 1}_Q . \nonumber
\end{align*}
Thanks to 
this latter inequality, we may therefore compute
\begin{align*}
    &\left\| \chi \left( a_A + \frac{\alpha}{2} I \right) + (1 - \chi) \alpha I - a \right\|_{W^{2,\infty}} \le \\
    &\le \left\| \chi \left( a_A + \frac{\alpha}{2} I \right) + (1 - \chi) \alpha I - a \right\|_{W^{2,\infty}(D)} \\
    &\hspace{30ex}+ \| \alpha I - a \|_{W^{2,\infty}\big( [0,T] \times \Rn \setminus \overline{B^{\Rn}_{R_*+1}(0)} \big)} \\
    &\le \left\| \chi \left( \left( a_A + \frac{\alpha}{2} I \right) - a \right) \right\|_{W^{2,\infty}(D)} + \| (1 - \chi) (\alpha I - a) \|_{W^{2,\infty}(D)} \\
    &\le C \left\| \left( a - \frac{\alpha}{2} I \right) - a_A \right\|_{W^{2,\infty}(D)} + C \| \alpha I - a \|_{W^{2,\infty}\big( [0,T] \times \Rn \setminus \overline{B^{\Rn}_{R_*}(0)} \big)} \\
    &\le C \| \sqrt{a} \|_{H^{d(m)}(D)}^2 h^{2 m - 1}_Q ,
\end{align*}
so that we may in conclusion select $\proj(a) \triangleq \chi ( a_A + \alpha / 2 \, I ) + (1 - \chi) \alpha I$.}

\secRev{To conclude note that, according to the results in Section \ref{sec:RKHS}, it holds that $b_B = \Pi_Q(b)$ and $v = \Pi_Q( \sqrt{a - \alpha / 2 \, I} )$, where $\Pi_Q$ denotes the corresponding projection operator of each involved RKHS, and thus $a_A = \Pi_Q( \sqrt{a - \alpha / 2 \, I} )^{\top} \Pi_Q( \sqrt{a - \alpha / 2 \, I} )$. This implies that
\begin{align*}
    \|\Pi_Q(b)\|_{H^{d(m)}(D)} \leq \ &\| b \|_{H^{d(m)}(D)} \quad \textnormal{and} \\
    &\| \Pi_Q( \sqrt{a - \alpha / 2 \, I} ) \|_{H^{d(m)}(D)} \leq \| \sqrt{a - \alpha / 2 \, I} \|_{H^{d(m)}(D)} .
\end{align*}
The proof readily follows from these latter inequalities, considering that
$$
\| \proj(a) - \alpha I \|_{H^{d(m)}} \le C \left\| a_A - \frac{\alpha}{2} I \right\|_{H^{d(m)}(D)} ,
$$
and that $H^{d(m)}(D)$ is in particular a Banach algebra.}
\end{proof}

In the rest of the manuscript, we adopt the notation we introduced in Theorem \ref{theo:appOperatorCoeff}. We are now ready to define our \firstRev{finite-dimensional problem} for learning the coefficients $(\unknownDiffusion,\unknownDrift) \in \CR$, which writes as follows:


\begin{definition} \label{def:ALP}
For every positive real $\lambda > 0$, the (random) finite-dimensional Learning Problem to learn stochastic differential equations is defined as:
$$
\LearningProblemFinite \quad \underset{(a,b) \in \CRh}{\min} L_{\lambda}(a,b) \triangleq \int^T_0 \left\| \frac{\partial \approxDensity}{\partial t}(t,\cdot) - (\Kolmo_t)^* \approxDensity(t,\cdot) \right\|^2_{L^2} \; \mathrm{d}t + \lambda \| (\secRev{a - \alpha I},b) \|^2_{\Cbasis} \ .
$$
\end{definition}

Since $\CRh \secRev{- \, \{ (\alpha I , 0) \} \subseteq \Cbasis}$ is 
\secRev{closed and convex}, the well-posedness of \LearningProblemFinite may be proven similarly to the well-posedness of the 
problem \LearningProblem, i.e., by replicating the proof of Proposition \ref{prop:wellPosLP}. We thus report this result below without proof:

\begin{proposition}
For every $\lambda > 0$, problem \LearningProblemFinite is well-posed and has a unique solution, which is denoted by $(\solAppDiffusion,\solAppDrift) \in \CRh$.
\end{proposition}

The next result is a natural extension of Lemma \ref{lemma:estimateSource} to the setting of problem \LearningProblemFinite, and it represents the main result of this section.

\begin{lemma} \label{lemma:estimateAppSource}
For 
$Q \in \mathbb{N}$ with $h_Q \leq 1$,
with probability at least $1 - \delta$, it holds that:
$$
L_{\lambda}(\solAppDiffusion,\solAppDrift) \le C(\unknownDiffusion,\unknownDrift) 
\left( \firstRev{\lambda + \left( \log\left( \frac{1}{\delta \varepsilon} \right)^{\frac{1}{2}} \varepsilon \right)^2 + h^{2(2 m - 1)}_Q} \right) ,
$$
where the constant $C(\unknownDiffusion,\unknownDrift) > 0$ depends on $\unknownDiffusion$ and $\unknownDrift$ uniquely \secRev{(in particular, $C(\unknownDiffusion,\unknownDrift)$ implicitly depends on the hyper-parameters $m$, $p_0$, and $\alpha$)}.
\end{lemma}

\vspace{-15px}

\firstRev{\begin{proof}
Thanks to Lemma \ref{lemma:estimateSource} (actually, its proof), we may compute
\secRev{\begin{align*}
    &L_{\lambda}(\solAppDiffusion,\solAppDrift) - \lambda \| \big( \proj(\unknownDiffusion) - \alpha I , \proj(\unknownDrift) \big) \|^2_{\Cbasis} \le \\
    &\hspace{35ex} \le \int^T_0 \left\| \frac{\partial \approxDensity}{\partial t}(t,\cdot) - ( \projKolmoUnk_t )^* \approxDensity(t,\cdot) \right\|^2_{L^2} \; \mathrm{d}t \\
    &\le 2 \int^T_0 \left\| \frac{\partial \approxDensity}{\partial t}(t,\cdot) - \unknownKolmo^*_t \approxDensity(t,\cdot) \right\|^2_{L^2} \; \mathrm{d}t \\
    &\hspace{35ex} + 2 \int^T_0 \left\| \left( ( \projKolmoUnk_t )^* - \unknownKolmo^*_t \right) \approxDensity(t,\cdot) \right\|^2_{L^2} \; \mathrm{d}t \\
    &\le C(\unknownDiffusion,\unknownDrift) \Bigg( \left( \log\left( \frac{1}{\delta \varepsilon} \right)^{\frac{1}{2}} \varepsilon \right)^2 \\
    &\hspace{25ex} + \| (\unknownDiffusion,\unknownDrift) - \big( \proj(\unknownDiffusion),\proj(\unknownDrift) \big) \|^2_{W^{2,\infty}} \int^T_0 \| \approxDensity(t,\cdot) \|^2_{H^2} \; \mathrm{d}t \Bigg) ,
\end{align*}}
where the constant $C(\unknownDiffusion,\unknownDrift) > 0$, which we will overload below, depends on $\unknownDiffusion$ and $\unknownDrift$ uniquely. Theorem \ref{thm:approxAle} and the choice $0 < \varepsilon , \delta < 1$ readily yield
$$
\int^T_0 \| \approxDensity(t,\cdot) \|^2_{H^2} \; \mathrm{d}t \le C(\unknownDiffusion,\unknownDrift) .
$$
By combining this inequality with Theorem \ref{theo:appOperatorCoeff}, we infer that
$$
\| (\unknownDiffusion,\unknownDrift) - \big( \proj(\unknownDiffusion),\proj(\unknownDrift) \big) \|^2_{W^{2,\infty}} \int^T_0 \| \approxDensity(t,\cdot) \|^2_{H^2} \; \mathrm{d}t \le C(\unknownDiffusion,\unknownDrift) h^{2(2 m - 1)}_Q ,
$$
yielding that
\secRev{\begin{align} \label{eq:proofFirstIneq}
    &L_{\lambda}(\solAppDiffusion,\solAppDrift) - \lambda \| \big( \proj(\unknownDiffusion) - \alpha I , \proj(\unknownDrift) \big) \|^2_{\Cbasis} \le \\
    &\hspace{25ex} \le C(\unknownDiffusion,\unknownDrift) \Bigg( \lambda + \left( \log\left( \frac{1}{\delta \varepsilon} \right)^{\frac{1}{2}} \varepsilon \right)^2 + h^{2( 2 m - 1 )}_Q \Bigg) . \nonumber
\end{align}}
Finally, observe that Theorem \ref{theo:appOperatorCoeff} additionally provides that
$$
\secRev{\| \big( \proj(\unknownDiffusion) - \alpha I , \proj(\unknownDrift) \big) \|^2_{\Cbasis} \le C(\unknownDiffusion,\unknownDrift) ,}
$$
which combined with \eqref{eq:proofFirstIneq} leads to the conclusion.
\end{proof}}

Thanks to Lemma \ref{lemma:estimateAppSource}, Theorem \ref{theorem:ultimateEstimate1} may be straightforwardly extended to the context of the finite-dimensional learning problem \LearningProblemFinite (just by replicating its proof). We thus report this result without proof in the proposition below.

\begin{theorem} \label{theorem:ultimateEstimate2}
Let the coefficients $(\solAppDiffusion,\solAppDrift) \in \CRh$ be the unique solution to \LearningProblemFinite with 
$\firstRev{\lambda = \left( \log\left( \frac{1}{\delta \varepsilon} \right)^{\frac{1}{2}} \varepsilon \right)^2}$ and $Q \in \mathbb{N}$ so that $h_Q = \left( \log\left ( \frac{1}{\delta \varepsilon} \right) \varepsilon \right)^{\frac{2}{2 m - 1}} \leq 1$. 
With probability at least $1 - \delta$, it holds that:
\firstRev{\begin{align*}
    E(\solAppDiffusion,\solAppDrift) &\le C(\unknownDiffusion,\unknownDrift) \left( \log\left( \frac{1}{\delta \varepsilon} \right)^{\frac{1}{2}} \varepsilon \right)^2 \\
    &\le C(\unknownDiffusion,\unknownDrift) \left( \log\left( \frac{M N}{\delta} \right)^{\frac{1}{2}} (M N)^{-\frac{2 m}{n + 2 ( 2 m + 1 )}} \right)^2 ,
\end{align*}}
where the constant $C(\unknownDiffusion,\unknownDrift) > 0$ depends on $\unknownDiffusion$ and $\unknownDrift$ uniquely \secRev{(in particular, $C(\unknownDiffusion,\unknownDrift)$ implicitly depends on the hyper-parameters $m$, $p_0$, and $\alpha$)}.
\end{theorem}

We conclude this section with a result which summarizes our essential contributions, showing how our learning approach may be leveraged for efficient observation/regulation of stochastic differential equations. First, we recall the setting introduced at the end of Section \ref{sec:ShortFokkerPlanck}. Specifically, for every coefficients $(a,b) \in \C$, we denote by $X^{a,b}$ and $p_{a,b}$ respectively the (unique) solutions to SDE and SFPE with coefficients $(a,b) \in \C$. By combining Theorem \ref{theorem:ultimateEstimate2} with Corollary \ref{Corol:ShortestimateOperator} (in particular, compare with \eqref{eq:ShortintegralMetric}; more details about this remark are provided in Section \ref{sec:FokkerPlanck}), we readily obtain the following summarizing result:

\begin{theorem}
Let $m \in \mathbb{N}$, $\alpha > 0$, $R_* > 0$, $p_0 \in H^{2 m + 1}(\Rn,\R)$, and $(\unknownDiffusion,\unknownDrift) \in \CR$ satisfy Assumptions $(A)$ and $(B)$, and denote by $X$ and $p$ respectively the (unique) solutions to SDE and SFPE with coefficients $(\unknownDiffusion,\unknownDrift) \in \CR$. There exists a constant $C(\unknownDiffusion,\unknownDrift) > 0$ which only depends on $\unknownDiffusion$ and $\unknownDrift$, such that by choosing the following learning parameters for the fixed precision parameters $0 < \varepsilon , \delta < 1$:
\begin{Itemize}
    \item 
    \firstRev{$M = \varepsilon^{-1/m}/4$, $N = \varepsilon^{-(2 + n/(2m))}$ (as closest integers), and $R = \varepsilon^{-1/(2m)}$},

    \item $\firstRev{\lambda = \left( \log\left( \frac{1}{\delta \varepsilon} \right)^{\frac{1}{2}} \varepsilon \right)^2}$ and $Q \in \mathbb{N}$ so that $h_Q = \min\left( 1 , \left( \log\left( \frac{1}{\delta \varepsilon} \right)^{\frac{1}{2}} \varepsilon \right)^{\frac{1}{2 m - 1}} \right)$,
\end{Itemize}
with probability at least $1 - \delta$, the following estimate holds:
\firstRev{\begin{align} \label{eq:finalRates}
    \| p - p_{\solAppDiffusion,\solAppDrift} \|_{L^2} &\le C(\unknownDiffusion,\unknownDrift) \log\left( \frac{1}{\delta \varepsilon} \right)^{\frac{1}{2}} \varepsilon \\
    &\le C(\unknownDiffusion,\unknownDrift) \log\left( \frac{M N}{\delta} \right)^{\frac{1}{2}} (M N)^{-\frac{2 m}{n + 2 ( 2 m + 1 )}} , \nonumber
\end{align}}
where $(\solAppDiffusion,\solAppDrift) \in \CRh$ is the unique solution to the finite-dimensional learning problem \LearningProblemFinite, where the learning parameters have been selected as above. Therefore, for every $f \in L^2([0,T]\times\Rn,\R)$ the following observation/regulation estimate holds:
\firstRev{\begin{align*}
    &\Bigg| \mathbb{E}_{\mu_0\times\Proba}\Bigg[ \int^T_0 f(t,X_x(t)) \; \mathrm{d}t \Bigg] - \mathbb{E}_{\mu_0\times\Proba}\Bigg[ \int^T_0 f(t,X^{\solAppDiffusion,\solAppDrift}_x(t)) \; \mathrm{d}t \Bigg] \Bigg| \le \\
    &\le C(\unknownDiffusion,\unknownDrift) \log\left( \frac{1}{\delta \varepsilon} \right)^{\frac{1}{2}} \varepsilon \le C(\unknownDiffusion,\unknownDrift) \log\left( \frac{M N}{\delta} \right)^{\frac{1}{2}} (M N)^{-\frac{2 m}{n + 2 ( 2 m + 1 )}} .
\end{align*}}

\begin{remark}[On the optimality of the learning rates \eqref{eq:finalRates}]
    \firstRev{To the best of our knowledge, the learning rates \eqref{eq:finalRates} have not yet appeared in the literature on non-parametric learning of SDE. Specifically, although in recent works, see, e.g., \cite{Abraham2019,Nickl2020,Aeckerle2022,Marie2023}, learning rates that improve as the regularity of the drift and diffusion coefficients increases are proposed, for the first time our method offers non-asymptotic learning rates when considering estimation of multi-dimensional SDE with non-constant diffusion coefficient, under discrete-time observations of $X$.}

    \firstRev{For the sake of completeness, one may want to compare \eqref{eq:finalRates} with existing learning rates, e.g., the rates offered by \cite{Abraham2019,Nickl2020}, though as we mentioned the estimation methods provided in these works apply to less general settings. If we were to do such comparison, we would infer that, up to a log factor, \eqref{eq:finalRates} would be ``almost optimal'' in that their would be sub-optimal by a factor $\displaystyle (NM)^{\frac{1}{n + 2 (2 m + 1)}}$. However, this factor exponentially converges to $1$ when $m$ tends to infinity. 
    }
\end{remark}
\end{theorem}

%% file: Computational.tex
\newcommand{\tr}{{\operatorname{tr}}}
\renewcommand{\vec}{{\operatorname{vec}}}

 \firstRev{From a numerical viewpoint, we can solve problem \LearningProblemFinite exactly in closed form. Indeed, such a problem amounts to a finite semi-definite program in the variables $A \in \R^{Q n \times Q n}, B \in \R^{Qn \times 1}$, subject to $A \succeq 0$.}

\vspace{5px}

\firstRev{\textbf{Problem reformulation.} 
\secRev{Recall that we naturally extend the mapping $\Phi_Q$ in \eqref{eq:ab-with-phi} to zero outside $D$, and that we arbitrarily fix a smooth cut-off function $\chi$ to define $\CRh$, such as an easily numerically implementable mollifier, which amounts to an additional hyper-parameter.} Thanks to the characterization form \cref{eq:ab-with-phi}, which is valid for every $(a,b) \in \CRh$, one easily shows that
$$
\secRev{(\Kolmo_t)^* \approxDensity(t,y) = \tr(V(t,y) \, A) - U(t,y) B + r(t,y) , \quad (t,y) \in [0,T]\times\R^n ,}
$$
where
$$
\secRev{\displaystyle r(t,y) \triangleq \alpha \Delta \bigg( \bigg( 1 - \frac{\chi}{2} \bigg) \approxDensity \bigg)(t,y) ,}
$$
and $U: [0,T]\times\R^n \to \R^{1 \times Q n}$ and $V: [0,T]\times\R^n \to \R^{Q n \times Q n}$ are defined as
\begin{align*}
&U_{\ell}(t,y) \triangleq \sum_{i=1}^n \frac{\partial}{\partial y_i}\left[ \secRev{\chi(y)} \approxDensity(t,y) e^{\top}_i K_D((t,y),(t_{\ell},x_{\ell})) \right] \in \R^n , \quad \ell , \ell' = 1 , \dots , Q , \\
&V_{\ell,\ell'}(t,x) \triangleq \frac{1}{2} \sum_{i,j=1}^n \frac{\partial^2}{\partial y_j \partial y_i}\big[ \secRev{\chi(y)} \approxDensity(t,y) K_D((t,y),(t_{\ell},x_{\ell})) \times \\
&\hspace{45ex} \times K_D((t,y),(t_{\ell'},x_{\ell'})) e_i e^{\top}_j \big] \in \R^{n \times n} ,
\end{align*}
\secRev{with $A \in \R^{Q n \times Q n}$, $A \succeq 0$, and $B \in \R^{Qn \times 1}$ our variables.} Therefore, by denoting $\displaystyle q \triangleq \frac{\partial \approxDensity}{\partial t}$, the cost in \LearningProblemFinite writes
\begin{align*}
L_\lambda(a,b) &= \int_{[0,T]\times\R^n} \left(q(t,y) - \tr(V(t,y) A) + U(t,y) B - r(t,y) \right)^2 \mathrm{d}t \; \mathrm{d}y \\
& = \int_{[0,T]\times\R^n} (1, B^\top, \vec(A)^\top) W(t,y) (1, B^\top, \vec(A)^\top) \; \mathrm{d}t \; \mathrm{d}y \\
& = (1, B^\top, \vec(A)^\top) H (1, B^\top, \vec(A)^\top)^\top.
\end{align*}
where, by denoting $\widetilde{q} \triangleq q - r$, we define $W: [0,T]\times\R^n \to \R^{1+Qn + Qn^2 \times 1+Qn + Qn^2}$ as
$$
W(t,y) \triangleq \begin{pmatrix} \widetilde{q}^2(t,y) & -\widetilde{q}(t,y) U(t,y) & -\widetilde{q}(t,y) \vec(V(t,y))^\top \\
-\widetilde{q}(t,y) U(t,y) & U(t,y) U(t,y)^\top & U(t,y)^\top \vec(V(t,y))^\top\\
-\widetilde{q}(t,y) \vec(V(t,y)) & \vec(V(t,y)) U(t,y) & \vec(V(t,y))\vec(V(t,y))^\top ,
\end{pmatrix}
$$
whereas $H \in \R^{1+Qn + Q^2n^2 \times 1+Qn + Q^2n^2}$ is defined as
$$
H \triangleq \int_{[0,T]\times\R^n} W(t,y) \; \mathrm{d}t \; \mathrm{d}y .
$$
Summing up, \LearningProblemFinite equals the following semi-definite program
$$
\underset{\begin{cases}
    A \in \R^{Qn \times Qn} , B \in \R^{Qn} \\
    v \triangleq (1, B, \vec(A)) \in \R^{1 + Qn + Q^2n^2} \\
    A \succeq 0
\end{cases}}{\min} v^{\top} H v .
$$
Such semi-definite programs enjoy a particularly simple form. They can be efficiently solved via damped Newton methods, with computational cost $O((Qn)^{3.5})$ \cite{nesterov1994interior}}.

\vspace{5px}

\firstRev{\textbf{Computing the integral $\bm{H}$.} The computation of each entry of the matrix $H$ can be done in closed form, given that they correspond to the integral of products of functions, that are linear combinations of elementary functions. As a matter of example, below we explicitly compute the first element of $H$, i.e., $H_{11}$, the same argument holding for the other elements. We have that
$$
\int_{[0,T]\times\R^n} W_{11}(t,y) \; \mathrm{d}t \; \mathrm{d}y = \int_{[0,T]\times\R^n} ( q(t,y)^2 - 2 q(t,y) r(t,y) + r(t,y)^2 ) \; \mathrm{d}t \; \mathrm{d}y .
$$
Due to their similarity, we proceed with computations only for the integral of the first term in the right-hand side of this equality. For this, we may compute 
\begin{align*}
&\int_{[0,T]\times\R^n} q(t,y)^2 \; \mathrm{d}t \; \mathrm{d}y \\
& = \sum_{\ell,\ell'=1}^M \sum_{j,j'=1}^N \frac{1}{N^2} \int_{[0,T]\times\R^n} \dot{c}_\ell(t) \dot{c}_{\ell'}(t) \rho_R(y-x_{\ell,j}) \rho_R(y-x_{\ell',j'}) \; \mathrm{d}t \; \mathrm{d}y \\
& = \sum_{i,i',\ell,\ell'=1}^M \sum_{j,j'=1}^N \frac{\alpha_{i\ell} \alpha_{i'\ell'}}{N^2} \int_0^T \secRev{\dot{K}_{m+1}}(t-t_i) \secRev{\dot{K}_{m+1}}(t-t_{i'}) \; \mathrm{d}t \ \times \\
&\hspace{45ex} \times \int_{\R^n} \rho_R(y-x_{\ell,j}) \rho_R(y-x_{\ell',j'}) \; \mathrm{d}y .
\end{align*}
Above we used the definition of $\widehat{p}$, and the fact that $c_\ell$ are minimum-norm kernel interpolators \cite{Wendland2004}, i.e., they are given by $c_\ell(t) = \sum_{i=1}^N \alpha_{i,\ell} \secRev{K_{m+1}}(t-t_i)$, where $\alpha_{i,\ell} = (\secRev{G}^{-1})_{i,\ell}$, whereas $\secRev{G} \in \R^{N\times N}$ is given by $\secRev{G}_{i,i'} = \secRev{K_{m+1}}(t_i, t_{i'})$. By denoting
$$
R_{\ell, \ell'} \triangleq \int_0^T \secRev{\dot{K}_{m+1}}(t-t_i) \secRev{\dot{K}_{m+1}}(t-t_{i'}) \; \mathrm{d}t , \quad S_{\ell, j, \ell', j'} \triangleq \int_{\R^n} \rho_R(y-x_{\ell,j})\rho_R(y-x_{\ell',j'}) \; \mathrm{d}y ,
$$
the computations above thus amount to
$$
\int_{[0,T]\times\R^n} q(t,y)^2 \; \mathrm{d}t \; \mathrm{d}y = \sum_{\ell,\ell'=1}^M \sum_{j,j'=1}^N (\secRev{G}^{-1} R \secRev{G}^{-1})_{\ell, \ell'} S_{\ell, j, \ell', j'} .
$$
Assuming that the computational cost of analytically computing such integrals of elementary functions is $O(1)$ per integral, the total computational cost to compute $H_{11}$ is thus $O(M^3 + M^2 N^2)$. Under the same assumption and iterating the same argument for all the elements of $H$, since $U$ is the linear combination of $Q n^2$ elementary functions, whereas $V$ is the linear combination of $Q^2 n^4$ elementary functions, the total computational cost to compute the whole matrix $H$ is thus $O(M^3 + M^2 N^2 + Q^4 n^8)$.}

\vspace{5px}

\firstRev{\textbf{Considerations on the computational complexity.} Since to achieve a desired value for the fill distance $Q = O(h_Q^{-(n+1)})$ number of centers are required, taken, e.g. on a grid of stepsize $h_Q$ \cite{rudi2024finding}, we thus need $Q = O\left(\varepsilon^{-\frac{n + 1}{m - (n-1)/4}}\right)$ to meet the rates of Theorem \ref{theorem:ultimateEstimate2}. In particular, by dimensioning $M$ and $N$ according to Theorem \ref{theorem:ultimateEstimate2} to achieve the desired precision $\varepsilon$, it is sufficient that $Q^{4} \leq M^2 N^2$. Therefore, the proposed algorithm enjoy the following performance:
\begin{align*}
&\textrm{1. Approximation error of $(a^*, b^*)$:} ~~~O(\varepsilon) , \\
&\textrm{2. Total number of required samples:} ~~~O(\varepsilon^{-2 - \frac{n+1/2}{2m}}) , \\
&\textrm{3. Total computational cost:} ~~~O(\varepsilon^{-4 - \frac{2n+1}{2m}} n^8) .
\end{align*}
In particular, under regular enough settings, i.e., $m \geq n+1/2$, we infer that our identification method achieves approximation error $O(\varepsilon)$ by requiring 
$O(\varepsilon^{-2.5})$ samples and a total computational cost of $O(\varepsilon^{-5} n^8)$.}

%% file: Conclusion.tex
In this paper, we propose a Reproducing Kernel Hilbert Space-based learning paradigm for the identification of drift and diffusion coefficients of non-linear stochastic differential equations, which relies upon discrete-time observation of the state. Under assumptions of smoothness for the unknown drift and diffusion coefficients, we provide theoretical estimates of learning rates which become increasingly tighter when both the number of observations of the state and the regularity of the unknown drift and diffusion coefficients grow.

Some possible improvements and perspectives are in order. Since our learning rates essentially apply to the laws of the state process, it would be interesting to understand whether and under what conditions our method may be extended to derive stronger $L^p$ norm-based learning rates. Finally, it would be interesting to investigate extensions of our work for the identification of controlled stochastic differential equations: although these models are crucial for the control of complex systems, e.g., in aerospace and robotics, methods which offer relevant guarantees of accuracy and efficiency of the identification process still require extensive investigation.

%% file: FokkerPlanck.tex
In this section, we provide a more structured exposition of the concepts we previously introduced in Section \ref{sec:ShortFokkerPlanck}. For this, we chronologically retrace in more details every definition and result of Section \ref{sec:ShortFokkerPlanck} step by step.

\subsection{The Fokker-Planck equation}

We recall we fixed $m \in \mathbb{N}$ and a constant $\alpha > 0$, \secRev{such that, since we worked with diffusion coefficients that are never trivial, we replaced $\Cweak$ with}
$$
\secRev{\C \triangleq (\alpha I , 0) + \Big\{ (a,b) \in \Cweak : \ a(t,x) \succcurlyeq 0 , \ (t,x) \in [0,T] \times \Rn \Big\} .} 
$$
Also, we fixed a non-negative density $p_0 \in L^2(\Rn,\R)$ which served as appropriate initial condition, and we denoted by $\mu_0 \in \ProbaSpace(\Rn)$ the associate probability measure. We recall the following notions of stochastic differential equation and its solutions:

\begin{definition} \label{def:SDE}
A measurable mapping $X : \Rn\times\Omega\to\DistSpace$ solves the Stochastic Differential Equation with coefficients $(a,b) \in \C$ if each process $X_x(t,\omega) \triangleq X(x,\omega)(t)$ is $\F$--progressively measurable for every $x \in \Rn$, and
$$
\textnormal{SDE}_x \quad \begin{cases}
\mathrm{d}X_x(t) = b(t,X_x(t)) \; \mathrm{d}t + \secRev{\sqrt{a(t,X_x(t))}} \; \mathrm{d}W_t , \\[5pt]
\Proba\big( X_x(0) = x \big) = 1 ,
\end{cases}
$$
holds in $(\Omega,\G,\F,\Proba)$ for $\mu_0$-almost every $x \in \Rn$. A solution $X$ to SDE is unique if, for every measurable mapping $Y : \Rn\times\Omega\to\DistSpace$ which solves SDE with coefficients $(a,b) \in \C$, it holds that $X(x,\cdot) = Y(x,\cdot)$ a.s., for $\mu_0$-almost every $x \in \Rn$.
\end{definition}

The well-posedness of Definition \ref{def:SDE} is contained in the following theorem, together with other useful properties on solutions to SDE:

\begin{theorem} \label{theo:existenceSDE}
For every $(a,b) \in \C$, there exists a unique measurable mapping  $X : \Rn\times\Omega\to\DistSpace$ which solves SDE with coefficients $(a,b) \in \C$. In addition, the following properties hold true for the mapping $X$:
\begin{enumerate}
    \item The following mapping is measurable:
    $$
    (x,\omega,t) \in \Rn\times\Omega\times[0,T] \mapsto X_x(t,\omega) \in \Rn .
    $$
    
    \item For every $\varphi \in C_b([0,T]\times\Rn,\R)$, the following mapping is continuous:
    $$
    t \in [0,T] \mapsto \int_{\Rn} \int_{\Omega} \varphi(t,X_x(t)) \; \mathrm{d}\Proba \; \mu_0(\mathrm{d}x) \in \R .
    $$
\end{enumerate}
\end{theorem}

\begin{proof}
Thanks to Theorem \ref{Theo:CoeffProperties}, for every $x \in \Rn$ there exists a unique (up to stochastic indistinguishability) $\F$--adapted process $X_x : [0,T]\times\Omega \to \Rn$ with continuous sample paths which solves $\textnormal{SDE}_x$, and which satisfies the following inequality
\begin{equation} \label{eq:Burkholder}
    \secRev{\mathbb{E}\Big[ \Dist( X_{x_1} , X_{x_2} )^2 \Big] \le C\big( \alpha , \| (a - \alpha I,b) \|_{\Cbasis} \big) \| x_1 - x_2 \|^2 ,}
\end{equation}
where the constant \secRev{$C\big( \alpha , \| (a - \alpha I,b) \|_{\Cbasis} \big) > 0$ depends on $(a,b) \in \C$ uniquely} (see, e.g., \cite{LeGall2016})\footnote{Here, we use the fact that the mapping $y \mapsto \sqrt{a(t,y) + \alpha I}$ is Lipschitz, for every $t \in [0,T]$. This is a straightforward consequence of Theorem \ref{Theo:CoeffProperties} and the fact that the mapping $A \in \SymPP \mapsto \sqrt{A}$ is Lipschitz on $\{ A \in \SymPP : \ \| A \| \le \| a \|_{L^{\infty}} + \alpha , \ y^{\top} A y \ge \alpha \| y \|^2 , y \in \Rn \}$. Indeed, the latter set is compact and the mapping $A \in \SymPP \mapsto \sqrt{A}$ is continuously differentiable.}. Therefore, we define the mapping $X : \Rn\times\Omega\to\DistSpace$ by $X(x,\omega)(t) \triangleq X_x(t,\omega)$. From the continuity of the sample paths of each $X_x$, one may show that
$$
X(x,\cdot)^{-1}\left( \overline{B^{\DistSpace}_{\varepsilon}(w_0)} \right) = \underset{t \in [0,T] \cap \mathbb{Q}}{\bigcap} X_x(t)^{-1}\left( \overline{B^{\Rn}_{\varepsilon}(w_0(t))} \right) , \quad x \in \Rn ,
$$
from which, by leveraging a routine monotone class argument, we easily infer the measurability of the process $(X(x,\cdot))_{x \in \Rn}$.

To prove the measurability of $X : \Rn\times\Omega\to\DistSpace$, we rather build a measurable mapping $\tilde X : \Rn\times\Omega\to\DistSpace$ such that the process $(\tilde X(x,\cdot))_{x \in \Rn}$ is a modification of $(X(x,\cdot))_{x \in \Rn}$. In particular, such property would imply for every $x \in \Rn$ the existence of a subset $N_x \in \F_T$ with $\Proba(N_x) = 1$, and such that
\begin{equation} \label{eq:modification}
    \tilde X_x(t,\omega) = X_x(t,\omega) , \quad t \in [0,T] , \quad \omega \in N_x .
\end{equation}
From \eqref{eq:modification}, together with the completeness of $\F$, we would infer each process $\tilde X_x$ is $\F$--adapted and has continuous sample paths (the latter property being trivially true by definition). Moreover, thanks to Theorem \ref{Theo:CoeffProperties}, a routine application of Burkholder-Davis-Gundy inequality, and the fact that each process $X_x$ satisfies $\textnormal{SDE}_x$, one may straightforwardly compute, for every $x \in \Rn$,
\begin{align*}
    &\mathbb{E}\left[ \underset{t \in [0,T]}{\sup} \ \left\| \tilde X_x(t) - x - \int^t_0 b(s,\tilde X_x(s)) \; \mathrm{d}r - \int^t_0 \secRev{\sqrt{a(s,\tilde X_x(s))}} \; \mathrm{d}W_s \right\|^2 \right] \le \\
    &\le C \mathbb{E}\left[ \underset{t \in [0,T]}{\sup} \ \| \tilde X_x(t) - X_x(t) \|^2 \right] + C \mathbb{E}\left[ \left( \int^T_0 \| b(t,\tilde X_x(t)) - b(t,X_x(t)) \| \; \mathrm{d}t \right)^2 \right] \\
    &\hspace{17.5ex}+ C \mathbb{E}\left[ \int^T_0 \left\| \secRev{\sqrt{a(t,\tilde X_x(t))} - \sqrt{a(t,X_x(t))}} \right\|^2 \; \mathrm{d}t \right] \\
    &\le C \mathbb{E}\Big[ \Dist( \tilde X_x , X_x )^2 \Big] ,
\end{align*}
for some (overloaded) constant $C > 0$, and therefore by combining \eqref{eq:Burkholder} with \eqref{eq:modification} yields that each process $\tilde X_x$ satisfies SDE$_x$, $x \in \Rn$. In turn, we showed the existence of a measurable mapping $X : \Rn\times\Omega\to\DistSpace$ which solves SDE with coefficients $(a,b) \in \C$. By leveraging the same argument (and a routine application of Gronw\"all's inequality), one also shows the uniqueness of this mapping as in Definition \ref{def:SDE}.

At this step, we build the aforementioned mapping $\tilde X : \Rn\times\Omega\to\DistSpace$ via Kolmogorov's lemma. More precisely, combining \eqref{eq:Burkholder} with Kolmogorov's lemma yields the existence of a modification $(\tilde X(x,\cdot))_{x \in \Rn}$ of $(X(x,\cdot))_{x \in \Rn}$ such that, for every $\omega \in \Omega$, the mapping $x \in \Rn \mapsto \tilde X_x(\cdot,\omega) \in \DistSpace$ is continuous. Therefore, the mapping $\tilde X : \Rn\times\Omega\to\DistSpace$ is Caratheodory, and thus measurable (see, e.g., \cite[Lemma 8.2.6]{Aubin1990}).

At this step, note that the mapping
$$
\pi : [0,T]\times\DistSpace \to \Rn : (t,w) \mapsto w(t)
$$
satisfies
\begin{align*}
    \| \pi(t_1,w_1) - \pi(t_2,w_2) \| \le \| w_1(t_1) - w_1(t_2) \| + \Dist(w_1,w_2) ,
\end{align*}
for every $t_1 , t_2 \in [0,T]$, $w_1 , w_2 \in \DistSpace$. Hence, $\pi$ is continuous and property 1. follows from $X_x(t,\omega) = \pi\big( t , X(x,\omega) \big)$, for $(x,\omega,t) \in \Rn \times \Omega\times [0,T]$. Moreover, if $\varphi \in C_b([0,T]\times\Rn,\R)$ and $(t_k)_{k \in \mathbb{N}} \subseteq [0,T]$ satisfies $t_k \to t$, from what we just proved:
$$
\varphi\big( t_n , \pi\big( t_n , X(x,\omega) \big) \big) \to \varphi\big( t , \pi\big( t , X(x,\omega) \big) \big) , \quad \textnormal{a.e. in} \quad \Rn\times\Omega ,
$$
and a routine application of the dominated convergence theorem yields property 2.
\end{proof}

At this step, for every $(a,b) \in \C$, we denote the Kolmogorov generator by
$$
\Kolmo_t \varphi(y) \triangleq \frac{1}{2} \sum^n_{i,j=1} \secRev{a}_{ij}(t,y) \frac{\partial^2 \varphi}{\partial y_i \partial y_j}(y) + \sum^n_{i=1} b_i(t,y) \frac{\partial \varphi}{\partial y_i}(y) , \quad \varphi \in C^2(\Rn,\R) .
$$
Fix $(a,b) \in \C$, and assume we are given a measurable mapping  $X : \Rn\times\Omega\to\DistSpace$ which solves SDE with coefficients $(a,b) \in \C$. For $\varphi \in C^{\infty}_c(\Rn,\R)$, a straightforward application of It\^o's formula to SDE$_x$ yields
\begin{equation} \label{eq:Ito}
    \int_{\Omega} \varphi(X_x(t)) \; \mathrm{d}\Proba = \varphi(x) + \int_{\Omega} \int^t_0 \Kolmo_s \varphi(X_x(s)) \; \mathrm{d}s \; \mathrm{d}\Proba , \quad t \in [0,T] ,
\end{equation}
which holds for $\mu_0$-almost every $x \in \Rn$. Thanks to Theorem \ref{theo:existenceSDE}, we may define the curve $\mu : [0,T] \to \ProbaSpace(\Rn)$ of probability measures
\begin{equation*} 
    \mu_t(A) \triangleq \int_{\Rn} \int_{\Omega} \mathds{1}_{\{ X_x(t) \in A \}} \; \mathrm{d}\Proba \; \mu_0(\mathrm{d}x) , \quad A \in \Borel ,
\end{equation*}
and note that $\mu$ is narrowly continuous, i.e., for every $\varphi \in C_b(\Rn,\R)$, the mapping
$$
t \in [0,T] \mapsto \int_{\Rn} \varphi(y) \mu_t(\mathrm{d}y) = \int_{\Rn} \int_{\Omega} \varphi(X_x(t)) \; \mathrm{d}\Proba \; \mu_0(\mathrm{d}x) \in \R
$$
is continuous. By combining this latter property with \eqref{eq:Ito} and Theorem \ref{theo:existenceSDE}
, one readily checks that the curve of probabilities $\mu$ satisfies:

\begin{definition} \label{def:FPE}
A narrowly continuous curve $\mu : [0,T] \to \ProbaSpace(\Rn)$ is said to solve the Fokker-Planck Equation with coefficients $(a,b) \in \C$ if
$$
\textnormal{FPE} \quad \begin{cases}
\displaystyle \frac{\mathrm{d}}{\mathrm{d}t} \int_{\Rn} \varphi(y) \mu_t(\mathrm{d}y) = \int_{\Rn} \Kolmo_t \varphi(y) \mu_t(\mathrm{d}y) , \quad \varphi \in C^{\infty}_c(\Rn,\R) , \\[5pt]
\mu_{t=0} = \mu_0 .
\end{cases}
$$
\end{definition}

Importantly, thanks to the regularity of the coefficients $(a,b) \in \C$, FPE can have one narrowly continuous solution $\mu$ at most, as we state in the following:

\begin{proposition}[Propositions 4.1 and 4.2 in \cite{Figalli2008}] \label{prop:UniquenessFPE}
Given any tuple of coefficients $(a,b) \in \C$, at most one narrowly continuous curve $\mu : [0,T] \to \ProbaSpace(\Rn)$ can solve FPE with coefficients $(a,b) \in \C$.
\end{proposition}


Our previous computations show that solutions to FPE may be obtained from solutions to SDE, and we now establish this process may be inverted. Specifically, motivated by the results in \cite{Stroock1997,Figalli2008}, we prove that any narrowly continuous curve $\mu : [0,T] \to \ProbaSpace(\Rn)$ solution to FPE with coefficients $(a,b) \in \C$ is associated with a unique (in the sense of Definition \ref{def:SDE}) measurable mapping $X : \Rn\times\Omega\to\DistSpace$ which solves SDE with coefficients $(a,b) \in \C$. We gather such result in the following theorem, which is a natural extension of \cite[Theorem 2.6]{Figalli2008} to our setting:

\begin{theorem} \label{theo:existenceMP}
Let a narrowly continuous curve $\mu : [0,T] \to \ProbaSpace(\Rn)$ be solution to FPE with coefficients $(a,b) \in \C$ (which is unique thanks to Proposition \ref{prop:UniquenessFPE}). There exists a unique measurable mapping $X : \Rn\times\Omega\to\DistSpace$ which solves SDE with coefficients $(a,b) \in \C$, and which satisfies the representation formula:
$$
\int_{\Rn} \varphi(y) \mu_t(\mathrm{d}y) = \int_{\Rn} \int_{\Omega} \varphi(X_x(t)) \; \mathrm{d}\Proba \; \mu_0(\mathrm{d}x) , \quad \textnormal{for} \quad t \in [0,T] , \quad \varphi \in C_c(\Rn,\R) .
$$
\end{theorem}

\begin{proof}
Thanks to Theorem \ref{theo:existenceSDE}, we already know there exists a unique measurable mapping  $X : \Rn\times\Omega\to\DistSpace$ which solves SDE with coefficients $(a,b) \in \C$.

To conclude we just need to show the representation formula holds true. For this, we define the curve $\bar \mu : [0,T] \to \ProbaSpace(\Rn)$ of probability measures
$$
\bar \mu_t(A) \triangleq \int_{\Rn} \int_{\Omega} \mathds{1}_{\{ X_x(t) \in A \}} \; \mathrm{d}\Proba \; \mu_0(\mathrm{d}x) , \quad A \in \Borel ,
$$
which is well-defined and narrowly continuous thanks to Theorem \ref{theo:existenceSDE}. In addition, by combining this latter property with \eqref{eq:Ito} and Theorem \ref{theo:existenceSDE}, we see that the curve $\bar \mu : [0,T] \to \ProbaSpace(\Rn)$ solves FPE, and thus Proposition \ref{prop:UniquenessFPE} yields
$$
\int_{\Rn} \varphi(y) \mu_t(\mathrm{d}y) = \int_{\Rn} \varphi(y) \bar \mu_t(\mathrm{d}y) = \int_{\Rn} \int_{\Omega} \varphi(X_x(t)) \; \mathrm{d}\Proba \; \mu_0(\mathrm{d}x) ,
$$
for every $t \in [0,T]$ and $\varphi \in C_c(\Rn,\R)$, and the conclusion follows.
\end{proof}

\subsection{Absolutely continuous solutions to the Fokker-Planck equation} \label{sec:absoluteFPE}

From what we showed, solutions to SDE may be found by solving FPE. In this section, we show the existence of narrowly continuous curves $\mu : [0,T] \to \ProbaSpace(\Rn)$ of type
$$
\mu_t(A) = \int_A p(t,y) \; \mathrm{d}y , \quad A \in \Borel ,
$$
for appropriate densities $p : [0,T]\times\Rn\to\R$, which are solutions to FPE. Note that, if such solutions to FPE exist, then they are unique thanks to Proposition \ref{prop:UniquenessFPE}. Although such existence result is classic (see, e.g., \cite{Figalli2008,Breiten2018}), we retrace its proof in Appendix \ref{sec:AppendixFokkerPlanck} to characterize the constants appearing in some appropriate estimates and regularity properties which have been paramount to derive the results in Section \ref{sec:LearningDynamics}.

For this, let us first recall the broader definition of solution to FPE, which in particular encompasses Definition \ref{def:FPE} as a sub-case (see also Theorem \ref{theor:exFPE} below):

\begin{definition} \label{def:FPENew}
Let $f \in L^2([0,T]\times\Rn,\R)$ and $\bar p \in L^2(\Rn,\R)$. A (regular enough) function $p : [0,T]\times\Rn\to\R$ is said to solve the non-homogeneous Fokker-Planck Equation with coefficients $(a,b) \in \C$, if
$$
\textnormal{FPE}_f \quad \begin{cases}
\displaystyle \frac{\mathrm{d}}{\mathrm{d}t} \int_{\Rn} \varphi(y) p(t,y) \; \mathrm{d}y = \\
\displaystyle \hspace{10ex}= \int_{\Rn} \left( \Kolmo_t \varphi(y) p(t,y) + f(t,y) \varphi(y) \right) \; \mathrm{d}y , \quad \varphi \in C^{\infty}_c(\Rn,\R) , \\[10pt]
p(0,\cdot) = \bar p(\cdot) .
\end{cases}
$$
\end{definition}

We gather results on the existence, uniqueness, and energy-type estimates for solutions to FPE$_f$, and therefore for solutions to FPE, in the following:

\begin{theorem} \label{theor:exFPE}
For every $f \in L^2([0,T]\times\Rn,\R)$ and every $\bar p \in L^2(\Rn,\R)$, there exists a unique mapping $p \in C(0,T;L^2(\Rn,\R)) \cap L^2(0,T;H^1(\Rn,\R))$ which solves FPE$_f$ with coefficients $(a,b) \in \C$, and with $\displaystyle \frac{\partial p}{\partial t} \in L^2(0,T;H^{-1}(\Rn,\R))$. In addition, the following first parabolic estimate holds:
\begin{align} \label{eq:estimate1}
    \| p (t,\cdot) \|^2_{L^2} + \int^t_0 &\| p(s,\cdot) \|^2_{H^1} \; \mathrm{d}s \le \nonumber \\
    &\le C\big(\secRev{\alpha},\|(\secRev{a - \alpha I},b)\|_{\Cbasis}\big) \left( \| \bar p \|^2_{L^2} + \int^t_0 \| f(s,\cdot) \|^2_{L^2} \; \mathrm{d}s \right) , \quad t \in [0,T] ,
\end{align}
\secRev{with $C\big(\secRev{\alpha},\|(a - \alpha I,b)\|_{\Cbasis}\big) > 0$ continuously depending on $\secRev{\alpha}$ and $\|(a - \alpha I,b)\|_{\Cbasis}$}. Finally, if $f = 0$ and $\bar p$ is a non-negative density in $L^2(\Rn,\R)$, then
$$
p(t,\cdot) \ge 0 , \quad \int_{\Rn} p(t,y) \; \mathrm{d}y = 1 , \quad t \in [0,T] ,
$$
and therefore, if for every $t \in [0,T]$ we define
$$
\mu_t(A) \triangleq \int_A p(t,y) \; \mathrm{d}y , \quad A \in \Borel ,
$$
then the curve $\mu : [0,T] \to \ProbaSpace(\Rn)$ is narrowly continuous and solves FPE.
\end{theorem}

\begin{remark}
The regularity of the mapping $f$ may be weakened (see, e.g., \cite{Lions1968,Breiten2018}), although the requirement $f \in L^2([0,T]\times\Rn,\R)$ already fits our purpose. 
\end{remark}

The proof of Theorem \ref{theor:exFPE} is reported in Appendix \ref{sec:AppendixFokkerPlanck} and is based on the classical Lions scheme (see, e.g., \cite{Lions1968,Chipot2000,Evans2010}). Among straightforward benefits, we recall Theorem \ref{theor:exFPE} enables introducing rigorous criteria to establish satisfactory guarantees for our learning approach in many circumstances. Let us better introduce this concept below. As a matter of example, when dealing with stochastic differential equations in applications such as observation and regulation, one must often manipulate metrics
\begin{equation} \label{eq:integralMetric}
    \mathbb{E}_{\mu_0\times\Proba}\left[ \int^T_0 f(t,X_x(t)) \; \mathrm{d}t \right] = \int^T_0 \left( \int_{\R^{2n}} f(t,y) \; \Proba_{X_x(t)}(\mathrm{d}y) \; \mu_0(\mathrm{d}x) \right) \; \mathrm{d}t ,
\end{equation}
where the measurable mapping $X : \Rn\times\Omega \to \DistSpace$ solves SDE for some coefficients $(a,b) \in \C$, and the mapping $f : [0,T]\times\Rn\to\R$ is regular enough. In the case the coefficients $(a,b) \in \C$ of SDE are to be learned, one does not have perfect knowledge of \eqref{eq:integralMetric}, and therefore the error between \eqref{eq:integralMetric} and its counterpart in which $(X_x)_{x\in\Rn}$ is replaced with the solution to SDE stemming from rather learned coefficients must be estimated. This gap is filled with the following:

\begin{corollary} \label{Corol:useInAppli}
For any coefficients $(a,b)\in\C$, let $\XCoeff : \Rn\times\Omega \to \DistSpace$ denote the unique measurable mapping which solves SDE with coefficients $(a,b) \in \C$ (see Theorem \ref{theo:existenceMP}). There exists an operator
\begin{align*}
    \OperatorSDE : \ &\C \to L^{\infty}(0,T;L^2(\Rn,\R)^*) \cong L^{\infty}(0,T;L^2(\Rn,\R)) \\
    &(a,b) \mapsto \OperatorSDECoeff_{\cdot}(\cdot)
\end{align*}
which, for every $t \in [0,T]$ and every $\varphi \in C_c(\Rn,\R)$, satisfies
$$
\OperatorSDECoeff_t(\varphi) = \int_{\R^{2n}} \varphi(y) \; \Proba_{\XCoeff_x(t)}(\mathrm{d}y) \; \mu_0(\mathrm{d}x) .
$$
In addition, if $\densityCoeff \in C(0,T;L^2(\Rn,\R))$ denotes the unique solution to FDE$_0$ with coefficients $(a,b)\in\C$ and $\bar p = p_0$ (which uniquely exists thanks to Theorem \ref{theor:exFPE}), for every $f \in L^2([0,T]\times\Rn,\R)$ the mapping
$$
t \in [0,T] \mapsto \OperatorSDECoeff_t\big( f(t,\cdot) \big) \in \R
$$
is measurable and in $L^1([0,T],\R)$, and it satisfies, for every $(a_1,b_1) , (a_2,b_2) \in \C$,
\begin{equation} \label{eq:estimateOperator}
    \int^T_0 \big| \OperatorSDE^{a_1,b_1}_t\big( f(t,\cdot) \big) - \OperatorSDE^{a_2,b_2}_t\big( f(t,\cdot) \big) \big| \; \mathrm{d}t \le \| f \|_{L^2} \| p_{a_1,b_1} - p_{a_2,b_2} \|_{L^2} .
\end{equation}
\end{corollary}

\begin{remark}
Thanks to the previous Corollary, we may provide \eqref{eq:integralMetric} with a rigorous meaning by defining, for every $f \in L^2([0,T]\times\Rn,\R)$,
$$
\mathbb{E}_{\mu_0\times\Proba}\left[ \int^T_0 f(t,\XCoeff_x(t)) \; \mathrm{d}t \right] \triangleq \int^T_0 \OperatorSDECoeff_t\big( f(t,\cdot) \big) \; \mathrm{d}t .
$$
As we showed in Section \ref{sec:LearningDynamics}, the estimate \eqref{eq:estimateOperator} enables the control of any estimation error occurring during the computation of the metric \eqref{eq:integralMetric}. Note that more refined estimates than \eqref{eq:estimateOperator} may be easily obtained by requiring more regularity on the function $f$ (e.g., see the proof of Corollary \ref{Corol:useInAppli} below).
\end{remark}

\begin{proof}
Fix any $(a,b)\in\C$. For $t \in [0,T]$ and $\varphi \in L^2(\Rn,\R)$, we define
$$
\OperatorSDECoeff_t(\varphi) = \underset{k\to\infty}{\lim} \ \int_{\R^{2n}} \varphi_k(y) \; \Proba_{\XCoeff_x(t)}(\mathrm{d}y) \; \mu_0(\mathrm{d}x) ,
$$
where $(\varphi_k)_{k\in\mathbb{N}} \subseteq C_c(\Rn,\R)$ is any sequence which converges to $\varphi$ for the strong topology of $L^2(\Rn,\R)$. This definition is well-posed thanks to Theorem \ref{theo:existenceMP}: indeed, the limit above uniquely exists, given that if $(\varphi^1_k)_{k\in\mathbb{N}} , (\varphi^2_k)_{k\in\mathbb{N}} \subseteq C_c(\Rn,\R)$ are two sequences which converge to $\varphi$ for the strong topology of $L^2(\Rn,\R)$, we may compute
\begin{align*}
    \bigg| \int_{\R^{2n}} \varphi^1_k(y) \; \Proba_{\XCoeff_x(t)}(\mathrm{d}y) \; \mu_0(\mathrm{d}x) - &\int_{\R^{2n}} \varphi^2_k(y) \; \Proba_{\XCoeff_x(t)}(\mathrm{d}y) \; \mu_0(\mathrm{d}x) \bigg| \le \\
    &\le \| \densityCoeff(t,\cdot) \|_{L^2} \| \varphi^1_k - \varphi^2_k \|_{L^2} \to 0 , \quad k\to\infty .
\end{align*}
Similarly, one may easily prove that $\OperatorSDECoeff_t(\cdot) \in L^2(\Rn,\R)^* \cong L^2(\Rn,\R)$ for every $t \in [0,T]$. In particular, note that thanks to \eqref{eq:estimate1}, for $t \in [0,T]$ and $\varphi \in L^2(\Rn,\R)$,
\begin{equation} \label{eq:boundOperatorSDE}
    |\OperatorSDECoeff_t(\varphi)| \le C(a,b) \| p_0 \|_{L^2} \| \varphi \|_{L^2} ,
\end{equation}
for some appropriate constant $C(a,b) > 0$.

Finally, let $g \in L^2([0,T]\times\Rn,\R)$ and let $(g_k)_{k\in\mathbb{N}} \subseteq C_c([0,T]\times\Rn,\R)$ converge to $g$ for the strong topology of $L^2([0,T]\times\Rn,\R)$. In particular, up to some subsequence, we infer that each sequence $\big( g_k(t,\cdot) \big)_{k\in\mathbb{N}} \subseteq C_c([0,T]\times\Rn,\R)$ converges to $g(t,\cdot) \in L^2(\Rn,\R)$ for the strong topology of $L^2(\Rn,\R)$, for almost every $t \in [0,T]$. Hence, the definition of $\OperatorSDECoeff_t(\cdot)$ and Theorem \ref{theo:existenceMP} yield
\begin{equation} \label{eq:limitOperatorSDE}
    \OperatorSDECoeff_t\big( g(t,\cdot) \big) = \underset{k\to\infty}{\lim} \ \int_{\Rn} g_k(t,y) \densityCoeff(t,y) \; \mathrm{d}y , \quad \textnormal{for almost every $t \in [0,T]$} .
\end{equation}
On the one hand, combined with Theorem \ref{theor:exFPE} and Pettis' theorem, \eqref{eq:limitOperatorSDE} yields the Bochner-measurability of the mapping $t \in [0,T] \mapsto \OperatorSDECoeff_t(\cdot) \in L^2(\Rn,\R)^*$, and therefore from \eqref{eq:boundOperatorSDE} we deduce that $\OperatorSDECoeff_{\cdot}(\cdot) \in L^{\infty}(0,T;L^2(\Rn,\R)^*)$. On the other hand, \eqref{eq:limitOperatorSDE} yields the measurability of the mapping $t \in [0,T] \mapsto \OperatorSDECoeff_t\big( f(t,\cdot) \big) \in \R$, which together with \eqref{eq:boundOperatorSDE} makes the mapping $t \in [0,T] \mapsto \OperatorSDECoeff_t\big( f(t,\cdot) \big) \in \R$ measurable and in $L^1([0,T],\R)$. The conclusion of the proof follows from the fact that \eqref{eq:estimateOperator} is a straightforward consequence of \eqref{eq:limitOperatorSDE} and Fatou's lemma.
\end{proof}

\subsection{Additional regularity of solutions to the Fokker-Planck equation}

In Section \ref{sec:LearningDynamics}, we dealt with solutions $p : [0,T]\times\Rn\to\R$ of FPE$_0$ which enjoy higher regularity properties, and specifically for which $p \in H^{m+1,2(m+1)}([0,T]\times\Rn,\R)$, and for which FPE$_0$ holds pointwise. Below, we are going to show such additional attributes stem from the properties of coefficients $(a,b) \in \C$ listed in Theorem \ref{Theo:CoeffProperties}.

First, for every $(a,b) \in \C$ and almost every $t \in [0,T]$ we introduce the notation:
$$
(\Kolmo_t)^* u(t,y) \triangleq \frac{1}{2} \sum^n_{i,j=1} \frac{\partial^2 \big( a_{ij} u \big)}{\partial y_j \partial y_i}(t,y) - \sum^n_{i=1} \frac{\partial \big( b_i u \big)}{\partial y_i}(t,y) , \quad u \in L^2(0,T;H^2(\Rn,\R)) ,
$$
which denotes the dual operator of $\Kolmo_t$. This operator is well-defined with image in $L^2([0,T]\times\Rn,\R)$. Indeed, thanks to Theorem \ref{Theo:CoeffProperties}, one readily shows that
\begin{align*}
    &b_i u \in L^2(0,T;H^1(\Rn,\R)) , \quad \textnormal{with} \quad \frac{\partial \big( b_i u \big)}{\partial y_r} = b_i \frac{\partial u}{\partial y_r} + u \frac{\partial b_i}{\partial y_r} \in L^2([0,T]\times\Rn,\R) , \\
    &a_{ij} u \in L^2(0,T;H^2(\Rn,\R)) , \quad \textnormal{with} \quad \frac{\partial \big( a_{ij} u \big)}{\partial y_r} = a_{ij} \frac{\partial u}{\partial y_r} + u \frac{\partial a_{ij}}{\partial y_r} \in L^2([0,T]\times\Rn,\R) , \\
    &\hspace{30.5ex}\textnormal{and} \quad \frac{\partial^2 \big( a_{ij} u \big)}{\partial y_r \partial y_s} = a_{ij} \frac{\partial^2 u}{\partial y_r \partial y_s} + \frac{\partial a_{ij}}{\partial y_s} \frac{\partial u}{\partial y_r} \\
    &\hspace{37.5ex}+ \frac{\partial a_{ij}}{\partial y_r} \frac{\partial u}{\partial y_s} + u \frac{\partial^2 a_{ij}}{\partial y_r \partial y_s} \in L^2([0,T]\times\Rn,\R) ,
\end{align*}
for $u \in L^2(0,T;H^2(\Rn,\R))$, $i,j,r,s=1,\dots,n$, thus the well-posedness of $(\Kolmo_t)^*$.

Thanks to Theorem \ref{Theo:CoeffProperties}, the following higher regularity result holds:

\begin{theorem} \label{Theo:RegularityFPE}
Under the setting and notation of Theorem \ref{theor:exFPE}, if furthermore $\bar p \in H^{2 \expo + 1}(\Rn,\R)$, then the unique solution $p : [0,T]\times\Rn\to\R$ to FPE$_0$ with coefficients $(a,b) \in \C$ is additionally such that
$$
p \in H^{m+1}(0,T;H^{2(m + 1)}(\Rn,\R)) .
$$
In particular, the mapping $p$ additionally satisfies the following Strong Fokker-Planck Equations with coefficients $(a,b) \in \C$:
$$
\textnormal{SFPE} \quad \begin{cases}
\displaystyle \frac{\partial p}{\partial t}(t,y) = (\Kolmo_t)^* p(t,y) , \quad \textnormal{a.e.} \quad (t,y) \in [0,T]\times\Rn , \\[10pt]
p(0,\cdot) = \bar p(\cdot) .
\end{cases}
$$
\end{theorem}

\begin{remark}
This result is classically proved under either time-independent coefficients or bounded domains (e.g., \cite{Evans2010,Lions2012}). Although in classical reference which consider time-dependent coefficients such as \cite{Lions2012} it is explicitly mentioned therein that their results may be extended to unbounded domains, we found such extension non-straightforward. We therefore decided to revisit and extend the regularity results for parabolic equations which are contained in \cite[Section 6.3]{Evans2010} to time-dependent coefficients and unbounded domains, and to report the proof in Appendix \ref{sec:AppendixFokkerPlanck}.
\end{remark}

%% file: AppendixFokkerPlanck.tex
\subsection{Proofs of Section \ref{sec:Notation}}

\begin{proof}[Proof of Theorem \ref{Theo:CoeffProperties}]
The first and second statements in Theorem \ref{Theo:CoeffProperties} stem from Morrey theorem as soon as we show the existence of a linear and bounded operator:
$$
E : H^1([0,T]\times\Rn,\R) \to H^1(\R^{n+1},\R) ,
$$
for which there exists a constant $C > 0$ such that:
\begin{itemize}
    \item $E u|_{[0,T]\times\Rn} = u$, for every $u \in H^1([0,T]\times\Rn,\R)$,
    \item $\| E u \|_{L^2} \le C \| u \|_{L^2}$ and $\| E u \|_{H^1} \le C \| u \|_{H^1}$, for every $u \in H^1([0,T]\times\Rn,\R)$,
\end{itemize}
for instance, see \cite{Brezis2011}. For this, we may follow the proof of \cite[Theorem 8.6]{Brezis2011}. More specifically, let $\eta \in C^{\infty}(\R,[0,1])$ be such that
$$
\eta(t) = \left\{\begin{array}{cc}
1 , & \displaystyle t < \frac{T}{4} , \\
0 , & \displaystyle t > \frac{3}{4}T ,
\end{array}\right.
$$
and for every $u : [0,T]\times\Rn\to\R$ define the mappings
$$
u^-(t,y) \triangleq \left\{\begin{array}{cc}
u(t,y) , & 0 < t \le T , \\
0 , & t \le 0
\end{array}\right. \quad \textnormal{and} \quad u^+(t,y) \triangleq \left\{\begin{array}{cc}
u(t,y) , & 0 \le t < T , \\
0 , & t \ge T .
\end{array}\right.
$$
If $u \in H^1([0,T]\times\Rn,\R)$, one sees that $\eta u^+ \in H^1([0,\infty)\times\Rn,\R)$ with
$$
\frac{\partial \eta u^+}{\partial t} = \eta \left( \frac{\partial u}{\partial t}\right)^+ + u^+ \frac{\partial \eta}{\partial t} \quad \textnormal{and} \quad \frac{\partial \eta u^+}{\partial y_i} = \eta \left( \frac{\partial u}{\partial y_i} \right)^+ , \quad i=1,\dots,n ,
$$
and that $(1 - \eta) u^- \in H^1((-\infty,T]\times\Rn,\R)$ with, for every $i=1,\dots,n$,
$$
\frac{\partial (1 - \eta) u^-}{\partial t} = (1 - \eta) \left( \frac{\partial u}{\partial t}\right)^- + u^- \frac{\partial \eta}{\partial t} \quad \textnormal{and} \quad \frac{\partial (1 - \eta) u^-}{\partial y_i} = (1 - \eta) \left( \frac{\partial u}{\partial y_i} \right)^-  .
$$
At this step, since $u = \eta u + (1 - \eta) u$, we first extend $\eta u^+$ by reflection at $t = 0$ through a mapping $v^1 \in H^1(\R^{n+1},\R)$ which thus satisfies
$$
\| v^1 \|_{L^2} \le C \| u \|_{L^2} \quad \textnormal{and} \quad \| v^1 \|_{H^1} \le C \| u \|_{H^1} ,
$$
and then we extend $(1 - \eta) u^-$ by reflection at $t = T$ through another mapping $v^2 \in H^1(\R^{n+1},\R)$ which thus satisfies
$$
\| v^2 \|_{L^2} \le C \| u \|_{L^2} \quad \textnormal{and} \quad \| v^2 \|_{H^1} \le C \| u \|_{H^1} ,
$$
where the constant $C > 0$ only depends on $\eta$. The conclusion easily follows if we define the linear and bounded operator $E$ as $E u \triangleq v^1 + v^2$.
\end{proof}

\subsection{Proofs of Section \ref{sec:FokkerPlanck}}

\begin{proof}[Proof of Theorem \ref{theor:exFPE}]
The proof is standard and based on the classical Lions scheme (see, e.g., \cite{Lions1968,Chipot2000,Evans2010}), therefore we mainly focus on deriving the constant 
$\secRev{C\big(\alpha,\|(a - \alpha I,b)\|_{\Cbasis}\big)}$ in \eqref{eq:estimate1} and the properties of $p $ when $f = 0$ and $\bar p$ is a non-negative density in $L^2(\Rn,\R)$. For this, consider the Gelfand triple
$$
V \triangleq H^1(\Rn,\R) = H^1_0(\Rn,\R) \hookrightarrow H \triangleq L^2(\Rn,\R) \hookrightarrow V^* = H^{-1}(\Rn,\R)
$$
and define the $t$-measurable bilinear form
\begin{align*}
    \Bil : \ &[0,T] \times V \times V \to \R \\
    &(t;u,v) \mapsto \sum^n_{i,j=1} \int_{\Rn} \frac{1}{2} a_{ij}(t,y) \frac{\partial u}{\partial y_j}(y) \frac{\partial v}{\partial y_i}(y) \; \mathrm{d}y \\
    &\hspace{20ex}+ \sum^n_{i=1} \int_{\Rn} \left( \sum^n_{j=1} \frac{1}{2} \frac{\partial a_{ij}}{\partial y_j}(t,y) - b_i(t,y) \right) \frac{\partial v}{\partial y_i}(y) u(y) \; \mathrm{d}y .
\end{align*}
Thanks to Theorem \ref{Theo:CoeffProperties}, one may prove that the form $\Bil$ is continuous and semi-coercive (see also the computations for the estimate \eqref{eq:estimate1} below). Therefore, thanks to a straightforward modification of the proof of \cite[Theorem 11.7]{Chipot2000}, there exists a unique solution $p \in C(0,T;H) \cap L^2(0,T;V)$ to the variational problem
\begin{equation} \label{eq:VarProblem}
    \begin{cases}
    \displaystyle \frac{\mathrm{d}}{\mathrm{d}t} ( p(t,\cdot) , \varphi )_H + \Bil(t;p(t,\cdot),\varphi) = ( f(t,\cdot) , \varphi )_H , \quad \varphi \in V , \\[10pt]
    p(0,\cdot) = \bar p(\cdot) ,
    \end{cases}
\end{equation}
which in addition satisfies $\displaystyle \frac{\partial p}{\partial t} \in L^2(0,T;V^*)$. At this step, thanks to Theorem \ref{Theo:CoeffProperties}
$$
a_{ij}(t,\cdot) p(t,\cdot) \in H^1(\Rn,\R) , \quad \textnormal{with} \quad \frac{\partial (a_{ij} p)}{\partial y_j}(t,\cdot) = p(t,\cdot) \frac{\partial a_{ij}}{\partial y_j}(t,\cdot) + a_{ij}(t,\cdot) \frac{\partial p}{\partial y_j}(t,\cdot) ,
$$
for every $i,j=1,\dots,n$ and almost every $t \in [0,T]$. Therefore, for every $\varphi \in C^{\infty}_c(\Rn;\R)$, $i,j=1,\dots,n$, and almost every $t \in [0,T]$, it holds that
\begin{align*}
    \int_{\Rn} a_{ij}(t,y) p(t,y) &\frac{\partial^2 \varphi}{\partial y_j \partial y_i}(y) \; \mathrm{d}y = \\
    &= -\int_{\Rn} \left( a_{ij}(t,y) \frac{\partial p}{\partial y_j}(t,y) + p(t,y) \frac{\partial a_{ij}}{\partial y_j}(t,y) \right) \frac{\partial \varphi}{\partial y_i}(y) \; \mathrm{d}y ,
\end{align*}
showing that $p$ solves FPE$_f$ with coefficients $(a,b) \in \C$, $\alpha > 0$.

Next, since $p \in L^2(0,T;V)$ and $\displaystyle \frac{\partial p}{\partial t} \in L^2(0,T;V^*)$, from \eqref{eq:VarProblem} we may compute
\begin{align*}
    \frac{\mathrm{d}}{\mathrm{d}t} \frac{\| p(t,\cdot) \|^2_{H}}{2} = \left\langle \frac{\partial p}{\partial t}(t,\cdot) , p(t,\cdot) \right\rangle_{V^*} = ( f(t,\cdot) , p(t,\cdot) )_H - \Bil(t;p(t,\cdot),p(t,\cdot)) ,
\end{align*}
and therefore Young's inequality yields
\begin{align*}
    \frac{\mathrm{d}}{\mathrm{d}t} \frac{\| p(t,\cdot) \|^2_{H}}{2} &\le \| f(t,\cdot) \|_{H} \| p(t,\cdot) \|_{H} - \frac{\alpha}{2} \| \nabla p(t,\cdot) \|^2_{H} \\
    &\hspace{35ex} + \bar C(a,b) \sum^n_{i=1} \int_{\Rn} \left| \frac{\partial p(t,y)}{\partial y_i} p(t,y) \right| \; \mathrm{d}y \\
    &\le \frac{\| f(t,\cdot) \|^2_{H}}{2} + \frac{\| p(t,\cdot) \|^2_{H}}{2} - \frac{\alpha}{2} \| \nabla p(t,\cdot) \|^2_{H} \\
    &\hspace{28ex} + n M \bar C(a,b)^2 \frac{\| p(t,\cdot) \|^2_{L^2}}{2} + \frac{\| \nabla p(t,\cdot) \|^2_{H}}{2 M} \\
    &\le \big(1 + n M \bar C(a,b)^2 \big) \frac{\| p(t,\cdot) \|^2_{H}}{2} + \frac{\| f(t,\cdot) \|^2_{H}}{2} - \frac{\alpha}{4} \| \nabla p(t,\cdot) \|^2_{H} ,
\end{align*}
where $M > 2 / \alpha$ is some large enough constant which stems from applying Young's inequality, whereas $\bar C(a,b) > 0$ is a constant which stems from Theorem \ref{Theo:CoeffProperties} and continuously depends on the $L^{\infty}$ norms of $b$ and of the derivatives of $a$ uniquely, and a routine application of Gronw\"all's inequality together with Theorem \ref{Theo:CoeffProperties} yield \eqref{eq:estimate1}.

Finally, assume $f = 0$ and that $\bar p$ is a non-negative density in $L^2(\Rn,\R)$. Thanks to the fact that $u^+ , u^- \in H^1(\Rn,\R)$ for every $u \in H^1(\Rn,\R)$, with $\nabla (u^+) = \nabla u \mathds{1}_{\{ u > 0 \}}$ and $\nabla (u^-) = \nabla u \mathds{1}_{\{ u < 0 \}}$, and that $u^- = (-u)^+$, leveraging the notation we introduced previously and applying \cite[Lemma 11.2]{Chipot2000} with $u = -p$ yield
\begin{align*}
    \frac{\mathrm{d}}{\mathrm{d}t} &\frac{\| p(t,\cdot)^- \|^2_{H}}{2} = -\left\langle \frac{\partial p}{\partial t}(t,\cdot) , p(t,\cdot)^- \right\rangle_{V^*} \\
    &= -\sum^n_{i,j=1} \int_{\Rn} \frac{1}{2} a_{ij}(t,y) \frac{\partial (p^+ - p^-)}{\partial y_j}(t,y) \frac{\partial (-p^-)}{\partial y_i}(t,y) \; \mathrm{d}y \\
    &\hspace{15ex}+ \sum^n_{i=1} \int_{\Rn} \left( \sum^n_{j=1} \frac{1}{2} \frac{\partial a_{ij}}{\partial y_j}(t,y) - b_i(t,y) \right) \frac{\partial (p^+ - p^-)}{\partial y_i}(t,y) p(t,y)^- \; \mathrm{d}y \\
    &\le -\frac{\alpha}{2} \| \nabla (p^-)(t,\cdot) \|^2_{H} + \bar C(a,b) \sum^n_{i=1} \int_{\Rn} \left| \frac{\partial p^-}{\partial y_i}(t,y) p(t,y)^- \right| \; \mathrm{d}y \\
    &\le -\frac{\alpha}{2} \| \nabla (p^-)(t,\cdot) \|^2_{H} + n M \bar C(a,b)^2 \frac{\| p(t,\cdot)^- \|^2_{H}}{2} + \frac{\| \nabla (p^-)(t,\cdot) \|^2_{H}}{2 M} \\
    &\le n M \bar C(a,b)^2 \frac{\| p(t,\cdot)^- \|^2_{H}}{2} ,
\end{align*}
and thanks to a routine application of Gronw\"all’s inequality, the fact that $\bar p(x) \ge 0$ almost everywhere yields $p(t,x) \ge 0$ for every $t \in [0,T]$ and almost every $x \in \Rn$.

At this step, for every $k \in \mathbb{N}$ choose a cut-off function $\varphi_k \in C^{\infty}_c(\Rn,[0,1])$ such that $\varphi_k(x) = 1$ for every $x \in \overline{B^{\Rn}_k(0)}$, $\textnormal{supp}(\varphi_k) \subseteq B^{\Rn}_{2k}(0)$, and whose first and second derivatives are uniformly bounded. For every $t \in [0,T]$, from the definition of FPE$_f$ and the monotone and dominated convergence theorems (here, we leverage both Theorem \ref{Theo:CoeffProperties} and the fact that $p(t,\cdot) \in L^2(\Rn,\R)$, $t \in [0,T]$) we may compute
\begin{align*}
    &\int_{\Rn} p(t,y) \mathrm{d}y = \underset{k\to\infty}{\lim} \ \int_{\Rn} p(t,y) \varphi_k(y) \mathrm{d}y \\
    &= \underset{k\to\infty}{\lim} \ \bigg( \int_{\Rn} \bar p(y) \varphi_k(y) \; \mathrm{d}y + \sum^n_{i,j=1} \int_{\Rn} \frac{1}{2} a_{ij}(t,y) \frac{\partial^2 \varphi_k}{\partial y_i \partial y_j}(y) p(t,y) \; \mathrm{d}y \\
    &\hspace{45ex} + \sum^n_{i=1} \int_{\Rn} b_i(t,y) \frac{\partial \varphi_k}{\partial y_i}(y) p(t,y) \; \mathrm{d}y \bigg) \\
    &= \int_{\Rn} \bar p(y) \; \mathrm{d}y = 1 .
\end{align*}

To conclude, it is clear that we just need to prove that the curve $\mu : [0,T] \to \ProbaSpace(\Rn)$ defined in the statement of the theorem is narrowly continuous. For this, let $t \in [0,T]$ and $(t_k)_{k \in \mathbb{N}} \subseteq [0,T]$ such that $t_k \to t$ for $k \to \infty$. Note that $p \in C(0,T;H)$ implies that $\| p(t_n,\cdot) - p(t,\cdot) \|_{L^2} \to 0$ for $k \to \infty$. In particular, for any function $\varphi \in C_c(\Rn,\R)$, as soon as $k \to \infty$ we infer that
$$
\left| \int_{\Rn} \varphi(y) p(t_n,y) \; \mathrm{d}y - \int_{\Rn} \varphi(y) p(t,y) \; \mathrm{d}y \right| \le \| \varphi \|_{L^2} \| p(t_n,\cdot) - p(t,\cdot) \|_{L^2} \to 0 .
$$
We conclude from the fact that the narrow and weak* topologies coincide in $\ProbaSpace(\Rn)$.
\end{proof}

\begin{proof}[Proof of Theorem \ref{Theo:RegularityFPE}]
Given that the proof is substantially long, for the sake of clarity we divide it in several step. Below, we adopt the notation we introduced and used in the proof of Theorem \ref{theor:exFPE}.

\vspace{5pt}

\noindent \textbf{1) A second parabolic estimate.} In this section, we provide computations by rather considering SFPE$_f$ with $0\neq f \in L^2([0,T]\times\Rn,\R)$. Let $(v_i)_{i \in \mathbb{N}}$ be a countable basis of $H^{2(m + 1)}(\Rn,\R)$ (and in turn of $H^{\ell}(\Rn,\R)$, for $\ell=0,\dots,2m+1$), such that $(v_i)_{i \in \mathbb{N}}$ is orthonormal in $H^1(\Rn,\R)$. Therefore, there is $(\bar x_i)_{i \in \mathbb{N}} \subseteq \R$ such that
\begin{equation} \label{eq:convInitialCond}
    \underset{j\to\infty}{\lim} \ \sum^j_{i=1} \bar x_i v_i = \bar p , \quad \textnormal{in} \ H^1(\Rn,\R) .
\end{equation}
Moreover, by revisiting the proof of \cite[Theorem 11.7]{Chipot2000}, one easily see that for every $k \in \mathbb{N}$ there exists $x^k \in AC([0,T],\R^k)$ such that the function
$$
p_k(t,\cdot) \triangleq \sum^k_{i=1} x^k_i(t) v_i \in C(0,T;L^2(\Rn,\R)) \cap L^2(0,T;H^1(\Rn,\R))
$$
is the unique solution to the variational problem
\begin{equation} \label{eq:VarProblemDis}
    \begin{cases}
    \displaystyle \frac{\mathrm{d}}{\mathrm{d}t} ( p_k(t,\cdot) , v_i )_{L^2} + \Bil(t;p_k(t,\cdot),v_i) = (f(t,\cdot),v_i)_H , \quad i=1,\dots,k , \\[10pt]
    x^k_i(0)= \bar x_i , \quad i=1,\dots,k ,
    \end{cases}
\end{equation}
which in addition satisfies (see also \cite[Page 197]{Chipot2000})
\begin{equation} \label{eq:CondConv}
    \begin{cases}
    \displaystyle \frac{\partial p_k}{\partial t}(t,\cdot) = \sum^k_{i=1} \frac{\mathrm{d}}{\mathrm{d}t} x^k_i(t) v_i \in L^2(0,T;H^1(\Rn,\R)) , \\
    \displaystyle \frac{\partial p_k}{\partial y_j}(t,\cdot) = \sum^k_{i=1} x^k_i(t) \frac{\partial v_i}{\partial y_j} \in AC(0,T;L^2(\Rn,\R)) , \quad j=1,\dots,n , \\
    p_k \to p \quad \textnormal{strongly in $L^2(0,T;H^1(\Rn,\R))$} .
    \end{cases}
\end{equation}
Moreover, from the proof of Theorem \ref{theor:exFPE}, it is clear that each $p_k$ satisfies \eqref{eq:estimate1}.

At this step, by multiplying each equation in \eqref{eq:VarProblemDis} by $\displaystyle \frac{\mathrm{d}}{\mathrm{d}t} x^k_i(t)$ and summing those over $i=1,\dots,k$, thanks to \eqref{eq:CondConv} for almost every $t \in [0,T]$ we may compute
\begin{align} \label{eq:DesiredBound}
    \bigg\| \frac{\partial p_k}{\partial t}&(t,\cdot) \bigg\|^2_{L^2} = \sum^k_{i=1} \left( \frac{\mathrm{d}}{\mathrm{d}t} x^k_i(t) \right) \frac{\mathrm{d}}{\mathrm{d}s}\bigg|_{s=t} ( p_k(s,\cdot) , v_i)_{L^2} \nonumber \\
    &= \sum^k_{i=1} \left( \frac{\mathrm{d}}{\mathrm{d}t} x^k_i(t) \right) \Big( (f(t,\cdot),v_i)_{L^2} - \Bil(t;p_k(t,\cdot),v_i) \Big) \nonumber \\
    &\le \| f(t,\cdot) \|_{L^2} \left\| \frac{\partial p_k}{\partial t}(t,\cdot) \right\|_{L^2} - \sum^n_{i,j=1} \int_{\Rn} \frac{1}{2} a_{ij}(t,y) \frac{\partial p_k}{\partial y_j}(t,y) \frac{\mathrm{d}}{\mathrm{d}t}\left( \frac{\partial p_k}{\partial y_i} \right)(t,y) \; \mathrm{d}y \\
    &\quad - \sum^n_{i=1} \int_{\Rn} \left( \sum^n_{j=1} \frac{1}{2} \frac{\partial a_{ij}}{\partial y_j}(t,y) - b_i(t,y) \right) \frac{\partial}{\partial y_i}\left( \frac{\partial p_k}{\partial t} \right)(t,y) \ p_k(t,y) \; \mathrm{d}y . \nonumber
\end{align}
We are going to bound the last two terms in \eqref{eq:DesiredBound}. Specifically, for the first to the last term, thanks to Theorem \ref{Theo:CoeffProperties} and \eqref{eq:CondConv} for every $t \in [0,T]$ we obtain that
\begin{align*}
    &-\int^t_0 \int_{\Rn} \sum^n_{i,j=1} \frac{1}{2} a_{ij}(s,y) \frac{\partial p_k}{\partial y_j}(s,y) \frac{\mathrm{d}}{\mathrm{d}s}\left( \frac{\partial p_k}{\partial y_i} \right)(s,y) \; \mathrm{d}y \; \mathrm{d}s = \\
    &= -\int^t_0 \int_{\Rn} \sum^n_{i,j=1} \frac{1}{2} \bigg( \frac{\mathrm{d}}{\mathrm{d}s}\left( a_{ij} \frac{\partial p_k}{\partial y_j} \frac{\partial p_k}{\partial y_i} \right)(s,y) - \dot{a}_{ij}(s,y) \frac{\partial p_k}{\partial y_j}(s,y) \frac{\partial p_k}{\partial y_i}(s,y) \\
    &\hspace{37.5ex} - a_{ij}(s,y) \frac{\mathrm{d}}{\mathrm{d}s}\left( \frac{\partial p_k}{\partial y_j} \right)(s,y) \ \frac{\partial p_k}{\partial y_i}(s,y) \bigg) \; \mathrm{d}y \; \mathrm{d}s ,
\end{align*}
and since $a_{ji} = a_{ij}$ for every $i,j=1,\dots,n$, we may compute
\begin{align*}
    &-\int^t_0 \int_{\Rn} \sum^n_{i,j=1} a_{ij}(s,y) \frac{\partial p_k}{\partial y_j}(s,y) \frac{\mathrm{d}}{\mathrm{d}s}\left( \frac{\partial p_k}{\partial y_i} \right)(s,y) \; \mathrm{d}y \; \mathrm{d}s \le \\
    &\le -\sum^n_{i,j=1} \frac{1}{2} \int_{\Rn} \left( a_{ij}(t,y) \frac{\partial p_k}{\partial y_j}(t,y) \frac{\partial p_k}{\partial y_i}(t,y) - a_{ij}(0,y) \frac{\partial p_k}{\partial y_j}(0,y) \frac{\partial p_k}{\partial y_i}(0,y) \right) \; \mathrm{d}y \; \mathrm{d}s \\
    &\hspace{30ex} + \frac{\| \dot{a} \|_{L^{\infty}}}{2} \int^t_0 \int_{\Rn} \sum^n_{i,j=1} \left| \frac{\partial p_k}{\partial y_j}(s,y) \right| \left| \frac{\partial p_k}{\partial y_i}(s,y) \right| \; \mathrm{d}y \; \mathrm{d}s .
\end{align*}
Therefore, the convergence \eqref{eq:convInitialCond} and \eqref{eq:estimate1} finally yield
\begin{align} \label{eq:IntermediateIneq}
    -\int^t_0 \int_{\Rn} &\sum^n_{i,j=1} a_{ij}(s,y) \frac{\partial p_k}{\partial y_j}(s,y) \frac{\mathrm{d}}{\mathrm{d}s}\left( \frac{\partial p_k}{\partial y_i}(s,y) \right) \; \mathrm{d}y \; \mathrm{d}s \le \nonumber \\
    &\le -\frac{\alpha}{2} \| \nabla p_k(t,\cdot) \|^2_{L^2} + C\big( \alpha , \|(\secRev{a - \alpha I},b)\|_{\Cbasis} \big ) \left( \| \bar p \|^2_{H^1} + \int^t_0 \| f(s,\cdot) \|^2_{L^2} \; \mathrm{d}s \right) ,
\end{align}
where $C\big( \alpha , \|(\secRev{a - \alpha I},b)\|_{\Cbasis} \big ) > 0$ continuously depends on $\alpha$ and $\|(\secRev{a - \alpha I},b)\|_{\Cbasis}$ uniquely. Below, we will implicitly overload the constant $C\big( \alpha , \|(\secRev{a - \alpha I},b)\|_{\Cbasis} \big )$.

We now focus on the last term in \eqref{eq:DesiredBound}. For this, first thanks to Theorem \ref{Theo:CoeffProperties} one obtains that, for every $i=1,\dots,n$, $\varphi \in C^{\infty}_c(\Rn,\R)$, and almost every $t \in [0,T]$,
\begin{align*}
    -\int_{\Rn} \bigg( \sum^n_{j=1} \frac{1}{2} \frac{\partial a_{ij}}{\partial y_j}(t,y) - b_i(t,y) &\bigg) \frac{\partial \varphi}{\partial y_i}(y) p_k(t,y) \; \mathrm{d}y = \\
    &= \int_{\Rn} \frac{\partial}{\partial y_i} \left( \left( \sum^n_{j=1} \frac{1}{2} \frac{\partial a_{ij}}{\partial y_j} - b_i \right) p_k \right)(t,y) \ \varphi(y) \; \mathrm{d}y .
\end{align*}
Since from \eqref{eq:CondConv} we infer the existence of a sequence $(\varphi_{\ell})_{\ell \in \mathbb{N}} \subseteq C^{\infty}_c(\Rn,\R)$ such that $\displaystyle \left\| \frac{\partial p_k}{\partial t}(t,\cdot) - \varphi_{\ell} \right\|_{H^1} \to 0$, for $\ell\to\infty$, again Theorem \ref{Theo:CoeffProperties} and \eqref{eq:estimate1} yield
\begin{align*}
    &-\int^t_0 \sum^n_{i=1} \int_{\Rn} \left( \sum^n_{j=1} \frac{1}{2} \frac{\partial a_{ij}}{\partial y_j}(s,y) - b_i(s,y) \right) \frac{\partial}{\partial y_i}\left( \frac{\partial p_k}{\partial s} \right)(s,y) \ p_k(s,y) \; \mathrm{d}y \; \mathrm{d}s = \\
    &= \int^t_0 \sum^n_{i=1} \int_{\Rn} \left( \sum^n_{j=1} \frac{1}{2} \frac{\partial^2 a_{ij}}{\partial y_i \partial y_j}(s,y) - \frac{\partial b_i}{\partial y_i}(s,y) \right) \frac{\partial p_k}{\partial s}(s,y) p_k(s,y) \; \mathrm{d}y \; \mathrm{d}s \\
    &\quad+ \int^t_0 \sum^n_{i=1} \int_{\Rn} \left( \sum^n_{j=1} \frac{1}{2} \frac{\partial a_{ij}}{\partial y_j}(s,y) - b_i(s,y) \right) \frac{\partial p_k}{\partial y_i}(s,y) \frac{\partial p_k}{\partial s}(s,y) \; \mathrm{d}y \; \mathrm{d}s \\
    &\le C\big( \alpha , \|(\secRev{a - \alpha I},b)\|_{\Cbasis} \big )  \int^t_0 \int_{\Rn} \left( |p_k(s,y)| + \sum^n_{i=1} \left| \frac{\partial p_k}{\partial y_i}(s,y) \right| \right) \left| \frac{\partial p_k}{\partial s}(s,y) \right| \; \mathrm{d}y \; \mathrm{d}s \\
    &\le \frac{1}{2} \int^t_0 \left\| \frac{\partial p_k}{\partial s}(s,\cdot) \right\|^2_{L^2} \; \mathrm{d}s + C\big( \alpha , \|(\secRev{a - \alpha I},b)\|_{\Cbasis} \big ) \left( \| \bar p \|^2_{H^1} + \int^t_0 \| f(s,\cdot) \|^2_{L^2} \; \mathrm{d}s \right) ,
\end{align*}
for every $t \in [0,T]$. By summing up this latter inequality with \eqref{eq:IntermediateIneq} and \eqref{eq:DesiredBound}, via a routine Granw\"all's inequality argument we infer that, for $k \in \mathbb{N}$ and $t \in [0,T]$, 
\begin{align} \label{eq:estimate2}
    \int^t_0 \left\| \frac{\partial p_k}{\partial s}(s,\cdot) \right\|^2_{L^2} \; \mathrm{d}s \le C\big( \alpha , \|(\secRev{a - \alpha I},b)\|_{\Cbasis} \big ) \left( \| \bar p \|^2_{H^1} + \int^t_0 \| f(s,\cdot) \|^2_{L^2} \; \mathrm{d}s \right) .
\end{align}

\vspace{5pt}

\noindent \textbf{2) First-order-in-time and second-order-in-space regularity.} In this section, we provide computations by rather considering SFPE$_f$ with $0\neq f \in L^2([0,T]\times\Rn,\R)$. We first show the inclusion $\displaystyle \frac{\partial p}{\partial t} \in L^2(0,T;L^2(\Rn,\R)) \cong L^2([0,T]\times\Rn,\R)$, this latter identification being true since $L^2(\Rn,\R)$ is a separable Hilbert space. For this, we show the existence of $\xi \in L^2(0,T;L^2(\Rn,\R))$ such that the following holds in $L^2(\Rn,\R)$:
\begin{equation} \label{eq:WeakTimeDer}
    \int^T_0 \xi(t,\cdot) \psi(t) \; \mathrm{d}t = -\int^T_0 p(t,\cdot) \frac{\mathrm{d}\psi}{\mathrm{d}t}(t) \; \mathrm{d}t , \quad \textnormal{for all} \ \psi \in C^{\infty}_c([0,T],\R) .
\end{equation}
For this, we first note that, thanks to \eqref{eq:estimate2} there exists $\xi \in L^2(0,T;L^2(\Rn,\R))$ such that $\displaystyle\left( \frac{\partial p_k}{\partial t} \right)_{k \in \mathbb{N}} \subseteq L^2(0,T;L^2(\Rn,\R))$ weakly converges to $\xi$, up to a subsequence. Now, for any $\psi \in C^{\infty}_c([0,T],\R)$ and any $\varphi \in L^2(\Rn,\R)$, thanks to the last (convergence) property in \eqref{eq:CondConv} we may compute
\begin{align*}
    \bigg( \int^T_0 &\xi(t,\cdot) \psi(t) \; \mathrm{d}t , \varphi \bigg)_{L^2} = \int^T_0 \big( \xi(t,\cdot) , \psi(t) \varphi \big)_{L^2} \; \mathrm{d}t = \\
    &= \underset{k\to\infty}{\lim} \ \int^T_0 \left( \frac{\partial p_k}{\partial t}(t,\cdot) , \psi(t) \varphi \right)_{L^2} \; \mathrm{d}t = \underset{k\to\infty}{\lim} \ \left( \int^T_0 \frac{\partial p_k}{\partial t}(t,\cdot) \psi(t) \; \mathrm{d}t , \varphi \right)_{L^2} \\
    &= -\underset{k\to\infty}{\lim} \ \left( \int^T_0 p_k(t,\cdot) \frac{\mathrm{d}\psi}{\mathrm{d}t}(t) \; \mathrm{d}t , \varphi \right)_{L^2} = -\underset{k\to\infty}{\lim} \ \int^T_0 \left( p_k(t,\cdot) , \frac{\mathrm{d}\psi}{\mathrm{d}t}(t) \varphi \right)_{L^2} \; \mathrm{d}t \\
    &= -\int^T_0 \left( p(t,\cdot) , \frac{\mathrm{d}\psi}{\mathrm{d}t}(t) \varphi \right)_{L^2} \; \mathrm{d}t = \left( -\int^T_0 p(t,\cdot) \frac{\mathrm{d}\psi}{\mathrm{d}t}(t) \; \mathrm{d}t , \varphi \right)_{L^2} ,
\end{align*}
and the equivalence \eqref{eq:WeakTimeDer} readily follows.

Next, we prove that $p \in L^2(0,T;H^2(\Rn,\R))$. For this, we first note that, from Theorem \ref{Theo:CoeffProperties} and the variational problem \eqref{eq:VarProblem}, thanks to our previous computations the following variational equality holds true for almost every $t \in [0,T]$:
\begin{align}
    \label{eq:Elliptic1}
    &\int_{\Rn} \sum^n_{i,j=1} \frac{1}{2} a_{ij}(t,y) \frac{\partial p}{\partial y_j}(t,y) \frac{\partial \varphi}{\partial y_i}(y) \; \mathrm{d}y \hspace{20ex} \textnormal{for} \quad \varphi \in H^1(\Rn,\R) \\
    &+ \sum^n_{i=1} \int_{\Rn} \left( \sum^n_{j=1} \frac{1}{2} \frac{\partial a_{ij}}{\partial y_j}(t,y) - b_i(t,y) \right) \frac{\partial \varphi}{\partial y_i}(y) p(t,y) \; \mathrm{d}y = \int_{\Rn} g(t,y) \varphi(y) \; \mathrm{d}y , \nonumber
\end{align}
where
\begin{equation} \label{eq:Elliptic2}
    g \triangleq f - \frac{\partial p}{\partial t} \in L^2([0,T]\times\Rn,\R) .
\end{equation}
Therefore, Theorem \ref{Theo:CoeffProperties} and the classical elliptic regularity theory (see, e.g., \cite[Section 6.3]{Evans2010}) imply, for almost every $t \in [0,T]$, both that $p(t,\cdot) \in H^2(\Rn,\R)$ and that
\begin{equation} \label{eq:EllipticBound}
    \| p(t,\cdot) \|^2_{H^2} \le C\big( \alpha , \|(\secRev{a - \alpha I},b)\|_{\Cbasis} \big ) \Big( \| p(t,\cdot) \|^2_{L^2} + \| g(t,\cdot) \|^2_{L^2} \Big) .
\end{equation}
At this step, from our choice for the countable basis $(v_i)_{i \in \mathbb{N}} \subseteq H^{2(m + 1)}(\Rn,\R)$, for almost every $t \in [0,T]$ and every $i \in \mathbb{N}$, in particular there exists $z_i(t) \in \R$ such that
\begin{equation} \label{eq:convInH2}
    p(t,\cdot) = \underset{k\to\infty}{\lim} \ \sum^k_{i=1} z_i(t) v_i(\cdot) , \quad \textnormal{in \ $H^2(\Rn,\R)$} ,
\end{equation}
for almost every $t \in [0,T]$. But, up to extracting a subsequence, \eqref{eq:CondConv} yields
$$
p(t,\cdot) = \underset{k\to\infty}{\lim} \ \sum^k_{i=1} x^k_i(t) v_i(\cdot) , \quad \textnormal{in \ $H^1(\Rn,\R)$} ,
$$
for almost every $t \in [0,T]$, and from the orthonormality of $(v_i)_{i \in \mathbb{N}}$ in $H^1(\Rn,\R)$ we obtain that $z_i(t) = x^k_i(t)$, for every $k \in \mathbb{N}$, $i=1,\dots,k$ and almost every $t \in [0,T]$. Hence, \eqref{eq:convInH2} yields that the mapping $p : [0,T] \to H^2(\Rn,\R)$ is strongly Bochner measurable. In addition, combining the elliptic estimate \eqref{eq:EllipticBound} together with the estimates \eqref{eq:estimate1} and \eqref{eq:estimate2} finally provides that $p \in L^2(0,T;H^2(\Rn,\R))$ with
\begin{align} \label{eq:estimate3}
    \int^t_0 \| p(s,\cdot) \|^2_{H^2} \; &\mathrm{d}s \le \nonumber \\
    &\le C\big( \alpha , \|(\secRev{a - \alpha I},b)\|_{\Cbasis} \big ) \left( \| \bar p \|^2_{H^1} + \int^t_0 \| f(s,\cdot) \|^2_{L^2} \; \mathrm{d}s \right) , \quad t \in [0,T] .
\end{align}

Before moving on with additional regularity properties, we note that, thanks to Theorem \ref{Theo:CoeffProperties} and the fact that $p \in L^2(0,T;H^2(\Rn,\R))$, integrating by parts the variational problem \eqref{eq:VarProblem} with $f = 0$ readily yields
$$
\int^T_0 \int_{\Rn} \left( \frac{\partial p}{\partial t}(t,y) - (\Kolmo_t)^* p(t,y) \right) \varphi(y) \; \mathrm{d}y \; \mathrm{d}t = 0 , \quad \varphi \in C^{\infty}_c(\Rn,\R) ,
$$
and SFPE$_0$ follows from the previous regularity properties and a density argument.

\vspace{5pt}

\noindent \textbf{3) Second-order-in-time and fourth-order-in-space regularity.} We now turn to the original setting SFPE$_0$, i.e., SFPE$_f$ with $f = 0$. The first step to further improve the regularity of $p$ consists of formally differentiating \eqref{eq:VarProblem} with respect to time and studying solutions $q : [0,T]\times\Rn\to\R$ to the new variational problem
\begin{equation} \label{eq:VarProblemReg}
    \begin{cases}
    \displaystyle \frac{\mathrm{d}}{\mathrm{d}t} ( q(t,\cdot) , \varphi )_{L^2} + \Bil(t;q(t,\cdot),\varphi) + \\[5pt]
    \hspace{25ex} + \BilDot(t;p(t,\cdot),\varphi) = 0 , \quad \varphi \in H^1(\Rn,\R) , \\
    q(0,\cdot) = (\Kolmo_0)^* \bar p(\cdot) .
    \end{cases}
\end{equation}
Problem \eqref{eq:VarProblemReg} is well-posed. Indeed, since $p \in L^2(0,T;H^2(\Rn,\R))$, thanks to Theorem \ref{Theo:CoeffProperties} one easily shows that
$$
\BilDot(t;p(t,\cdot),\varphi) = -\int_{\Rn} (\KolmoDot_t)^* p(t,y) \varphi(y) \; \mathrm{d}y ,
$$
for every $\varphi \in H^1(\Rn,\R)$ and almost every $t \in [0,T]$. Therefore, since $(\KolmoDot_t)^* p \in L^2([0,T]\times\Rn,\R)$ and $(\Kolmo_0)^* \bar p \in H^1(\Rn,\R)$, problem \eqref{eq:VarProblemReg} fits the setting of Theorem \ref{theor:exFPE}, thus from our previous computations there exists a unique solution $q \in L^2(0,T;H^2(\Rn,\R))$ to \eqref{eq:VarProblemReg}, which additionally satisfies $\displaystyle \frac{\partial q}{\partial t} \in L^2([0,T]\times\Rn,\R)$.

At this step, we define the function
\begin{align*}
    w : \ &[0,T]\times\Rn\to\R \\
    &(t,y) \mapsto \bar p(y) + \left( \int^t_0 q(s,\cdot) \; \mathrm{d}s \right)(y) = \bar p(y) + \int^t_0 q(s,y) \; \mathrm{d}s ,
\end{align*}
where the Lebesgue integral in the last equality stems from the properties of the Bochner integral for functions in $L^2([0,T]\times\Rn,\R)$. Clearly, $w \in C^0(0,T;H^2(\Rn,\R))$, and Fubini theorem yields, for every $\varphi_1 \in L^2(\Rn,\R)$ and $\psi \in C^{\infty}_c([0,T],\R)$,
$$
\int_{\Rn} \left( \int^T_0 \left(\int^t_0 q(s,y) \; \mathrm{d}s \right) \frac{\mathrm{d} \psi}{\mathrm{d}t}(t) \; \mathrm{d}t + \int^T_0 q(t,y) \psi(t) \; \mathrm{d}t \right) \varphi_1(y) \; \mathrm{d}y = 0 ,
$$
and, for almost every $t \in [0,T]$ and every $\varphi_2 \in C^{\infty}_c(\Rn,\R)$,
$$
\int_{\Rn} w(t,y) \frac{\partial \varphi_2}{\partial y_j}(y) \; \mathrm{d}y = -\int_{\Rn} \left( \frac{\partial \bar p}{\partial y_j}(y) + \int^t_0 \frac{\partial q}{\partial y_j}(s,y) \; \mathrm{d}s \right) \varphi_2(y) \; \mathrm{d}y , \quad j=1,\dots,n .
$$
In particular, from such properties we readily deduce that, for every $j=1,\dots,n$,
\begin{equation} \label{eq:CondDer}
    \frac{\partial w}{\partial t} = q \in L^2(0,T;H^2(\Rn,\R)) , \quad \textnormal{with} \quad \frac{\partial}{\partial y_j}\left( \frac{\partial w}{\partial t} \right) = \frac{\partial}{\partial t}\left( \frac{\partial w}{\partial y_j} \right) .
\end{equation}
Thanks to Theorem \ref{Theo:CoeffProperties} and the properties listed in \eqref{eq:CondDer}, 
for every $\varphi \in C^{\infty}_c(\Rn,\R)$ and almost every $t \in [0,T]$, we may compute
\begin{align*}
    &\frac{\mathrm{d}}{\mathrm{d}t} ( w(t,\cdot) , \varphi )_{L^2} = ( q(t,\cdot) , \varphi )_{L^2} \\
    &= \big( (\Kolmo_0)^* \bar p , \varphi \big)_{L^2} - \int^t_0 \Bil(s;q(s,\cdot),\varphi) \; \mathrm{d}s - \int^t_0 \BilDot(s;p(s,\cdot) \pm w(s,\cdot),\varphi) \; \mathrm{d}s \\
    &= \int_{\Rn} (\Kolmo_0)^* \bar p(y) \varphi(y) \; \mathrm{d}y - \\
    &\quad -\sum^n_{i,j=1} \int_{\Rn} \frac{1}{2} \left( \int^t_0 \left( a_{ij}(s,y) \frac{\partial}{\partial s}\left( \frac{\partial w}{\partial y_j} \right)(s,y) + \dot{a}_{ij}(s,y) \frac{\partial w}{\partial y_j}(s,y) \right) \mathrm{d}s \right) \frac{\partial \varphi}{\partial y_i}(y) \; \mathrm{d}y \\
    &\quad - \sum^n_{j=1} \int_{\Rn} \Bigg( \int^t_0 \Bigg( \Bigg( \sum^n_{j=1} \frac{\partial \dot{a}_{ij}}{\partial y_j}(s,y) - \dot{b}_i(s,y) \Bigg) w(s,y) \\
    &\hspace{28.75ex}+ \Bigg( \sum^n_{j=1} \frac{\partial a_{ij}}{\partial y_j}(s,y) - b_i(s,y) \Bigg) \frac{\partial w}{\partial s}(s,y) \Bigg) \mathrm{d}s \Bigg) \frac{\partial \varphi}{\partial y_i}(y) \; \mathrm{d}y \\
    &\quad - \int^t_0 \BilDot(s;(p - w)(s,\cdot),\varphi) \; \mathrm{d}s .
\end{align*}
Therefore, from \cite[Theorem 11.5]{Chipot2000} and the fact that
$$
\int_{\Rn} (\Kolmo_0)^* \bar p(y) \varphi(y) \; \mathrm{d}y = -\Bil(0;\bar p,\varphi) ,
$$
we conclude that $w$ solves the following integro-differential variational problem
\begin{equation} \label{eq:VarProblemIntegro}
    \begin{cases}
    \displaystyle \frac{\mathrm{d}}{\mathrm{d}t} ( w(t,\cdot) , \varphi )_{L^2} + \Bil(t;w(t,\cdot),\varphi) + \\[5pt]
    \displaystyle \hspace{10ex} + \int^t_0 \BilDot(s;(p-w)(s,\cdot),\varphi) \; \mathrm{d}s = 0 , \quad \varphi \in H^1(\Rn,\R) , \\[10pt]
    w(0,\cdot) = \bar p(\cdot) .
    \end{cases}
\end{equation}
Clearly, $p = w$ is solution to \eqref{eq:VarProblemIntegro}. In particular, if so we would obtain that $\displaystyle \frac{\partial p}{\partial t} \in L^2(0,T;H^2(\Rn,\R))$ and $\displaystyle \frac{\partial^2 p}{\partial t^2} \in L^2([0,T]\times\Rn,\R)$, and thanks to the choice of the basis $(v_i)_{i \in \mathbb{N}} \subseteq H^{2(m + 1
)}(\Rn,\R)$ and to a higher order elliptic regularity argument applied to \eqref{eq:Elliptic1}--\eqref{eq:Elliptic2} with $f=0$, we would also obtain that $p \in L^2(0,T;H^4(\Rn,\R))$, that is the sought second-order-in-time and fourth-order-in-space regularity.

Hence, to conclude we need to prove that the integro-differential variational problem \eqref{eq:VarProblemIntegro} may have one solution at most. For this, let $z_1 , z_2 \in L^2(0,T;H^2(\Rn,\R))$, such that $\displaystyle \frac{\partial z_1}{\partial t} , \frac{\partial z_2}{\partial t} \in L^2([0,T]\times\Rn,\R)$, satisfy \eqref{eq:VarProblemIntegro}. By defining the mapping
\begin{align*}
    h : \ &[0,T] \to L^2(\Rn,\R) \\
    &t \mapsto \int^t_0 (\KolmoDot_s)^* (z_1 - z_2)(s,\cdot) \; \mathrm{d}s ,
\end{align*}
it is readily checked that $h \in L^2([0,T]\times\Rn,\R)$, in particular satisfying
\begin{align} \label{eq:IntegroBound}
    &\int^t_0 \| h(s,\cdot) \|^2_{L^2} \; \mathrm{d}s \le \nonumber \\
    &\hspace{5ex}\le C\big( \alpha , \|(\secRev{a - \alpha I},b)\|_{\Cbasis} \big ) \int^t_0 \left( \int^s_0 \| (z_1 - z_2)(r,\cdot) \|^2_{H^2} \; \mathrm{d}r \right) \; \mathrm{d}s , \quad t \in [0,T] ,
\end{align}
and that
\begin{equation*}
    \begin{cases}
    \displaystyle \frac{\mathrm{d}}{\mathrm{d}t} ( (z_1 - z_2)(t,\cdot) , \varphi )_{L^2} + \Bil(t;(z_1 - z_2)(t,\cdot),\varphi) = \\[5pt]
    \hspace{45ex}= ( h(t,\cdot) , \varphi )_{L^2} , \quad \varphi \in H^1(\Rn,\R) , \\
    (z_1 - z_2)(0,\cdot) = 0 .
    \end{cases}
\end{equation*}
Therefore, we are in the setting of Theorem \ref{theor:exFPE} and of our previous computations. In particular, we may combine \eqref{eq:estimate3} with \eqref{eq:IntegroBound} to obtain that, for every $t \in [0,T]$,
\begin{align*}
    \int^t_0 \| (z_1 - z_2)(s,\cdot) \|^2_{H^2} \; &\mathrm{d}s \le C\big( \alpha , \|(\secRev{a - \alpha I},b)\|_{\Cbasis} \big ) \int^t_0 \| h(s,\cdot) \|^2_{L^2} \; \mathrm{d}s \\
    &\le C\big( \alpha , \|(\secRev{a - \alpha I},b)\|_{\Cbasis} \big ) \int^t_0 \left( \int^s_0 \| (z_1 - z_2)(r,\cdot) \|^2_{H^2} \; \mathrm{d}r \right) \; \mathrm{d}s ,
\end{align*}
and a routine application of Gronw\"all's inequality allows us to conclude that $z_1 = z_2$.

\vspace{5pt}

\noindent \textbf{4) Conclusion.} Given Theorem \ref{Theo:CoeffProperties}, the regularity of the initial condition $\bar p \in H^{2 m + 1}(\Rn,\R)$, and the choice of the basis $(v_i)_{i \in \mathbb{N}} \subseteq H^{2(m + 1)}(\Rn,\R)$, one may easily iterate the previous computations by induction to define successive time derivatives of $p$ via iteratively considering formal differentiation-in-time of the variational problem \eqref{eq:VarProblemReg}. In particular, it is clear how to extend properties \eqref{eq:CondDer} and define updated integro-differential variational problems \eqref{eq:VarProblemIntegro} by induction, which, thanks to iteratively higher order elliptic regularity arguments applied to \eqref{eq:Elliptic1}--\eqref{eq:Elliptic2} and the elliptic estimate \eqref{eq:estimate3}, provide that $\displaystyle \frac{\partial^{\ell} p}{\partial t^{\ell}} \in L^2(0,T;H^{2(m+1-\ell)}(\Rn,\R))$, for $\ell=1,\dots,m+1$, and $p \in L^2(0,T;H^{2(m + 1)}(\Rn,\R))$. The conclusion follows.
\end{proof}